\documentclass[11pt]{article}

\usepackage[letterpaper,margin=1in]{geometry}
\usepackage{times}
\usepackage{microtype}
\usepackage{natbib}
\usepackage[hypertexnames=false]{hyperref}
\usepackage{url}

\usepackage[T1]{fontenc}
\usepackage{booktabs}
\usepackage{amsfonts}
\usepackage{nicefrac}
\usepackage{xcolor}
\usepackage{amsmath,amssymb}
\usepackage{amsthm}
\usepackage{algorithm}
\usepackage{algpseudocode}
\usepackage{graphicx}
\usepackage{placeins}
\usepackage{subcaption}
\usepackage{multirow}
\usepackage{array}
\usepackage{tikz}
\usepackage{pgfplots}
\usepackage{xspace}
\usepackage{enumitem}

\usepackage{amsmath,amsfonts,bm}

\def\eqref#1{equation~\ref{#1}}

\def\1{\bm{1}}

\DeclareMathAlphabet{\mathsfit}{\encodingdefault}{\sfdefault}{m}{sl}
\SetMathAlphabet{\mathsfit}{bold}{\encodingdefault}{\sfdefault}{bx}{n}

\hypersetup{
  colorlinks=true,
  linkcolor=blue,
  citecolor=blue,
  urlcolor=blue,
  pdftitle={PACE-FNO: Physics-Aligned Canonical Equivariance for Fourier Neural Operators},
  pdfauthor={Jiaxiao Xu; Changhong Mou; Yeyu Zhang; Fengxiang He}
}

\usetikzlibrary{arrows.meta, positioning, calc, backgrounds, shadows.blur}
\pgfplotsset{compat=1.18}
\graphicspath{{images/}}

\definecolor{cblue}{RGB}{138,182,207}
\definecolor{corange}{RGB}{234,166,92}
\definecolor{cgreen}{RGB}{223,223,174}
\definecolor{cyellow}{RGB}{236,214,132}
\definecolor{cred}{RGB}{232,108,64}
\definecolor{cmidio}{RGB}{181,214,191}

\newcommand{\pmstd}{\,\ensuremath{(\pm)}\,}
\newcommand{\best}[1]{\textbf{#1}}

\algrenewcommand\algorithmicrequire{\textbf{Input:}}
\algrenewcommand\algorithmicensure{\textbf{Output:}}
\newtheorem{theorem}{Theorem}[section]
\newtheorem{lemma}[theorem]{Lemma}
\newtheorem{proposition}[theorem]{Proposition}
\newtheorem{corollary}[theorem]{Corollary}
\theoremstyle{definition}
\newtheorem{remark}{Remark}[section]

\newtheorem{definition}{Definition}[section]

\title{PACE-FNO: Physics-Aligned Canonical Equivariance for Fourier Neural Operators}

\author{Jiaxiao Xu$^{1}$ \quad Changhong Mou$^{2}$ \quad Yeyu Zhang$^{3}$ \quad Fengxiang He$^{4}$\\[0.4ex]
  $^{1,3}$ Shanghai University of Finance and Economics\\
  $^{2}$ Utah State University\\
  $^{4}$ University of Edinburgh}
\date{}

\begin{document}

\maketitle

\begin{abstract}
Neural operators are often tested on states that differ physically from training data. A distinct failure occurs when the physical dynamics are unchanged but the observed coordinate frame differs from training. PACE-FNO addresses this case by estimating the frame, predicting after pulling the field to a canonical representative, and restoring the requested terminal frame. The default inference path uses one forward prediction; optional test-time adaptation (TTA) updates only the low-dimensional coordinate. On translated and Galilean-shifted Burgers and shallow-water systems, PACE-FNO lowers out-of-distribution (OOD) relative error by up to $12\times$ relative to a data-augmented Fourier Neural Operator (FNO+Aug). A matched-estimator control supports attributing the gain to prediction in the canonical frame rather than added estimator capacity. Our analysis separates the canonical approximation error from the two alignment residuals. We also test regional WeatherBench2 time-series forecasting with the fifth-generation ECMWF reanalysis (ERA5) under natural time OOD, where PACE-FNO and FNO achieve nearly identical performance. Experiments with approximate rotation, other backbones, irregular domains, rollouts, and image data delineate the conditions under which the mechanism is effective.
\end{abstract}

\section{Introduction}
\label{sec:introduction}

Neural operators learn solution maps between function spaces, enabling fast predictions across varying inputs and parameters in applications such as fluid dynamics and weather forecasting \citep{karniadakis2021physics,li2020fourier,kovachki2023neural,pathak2022fourcastnet}. A vortex translated across a periodic domain still follows the same equation, yet an operator trained around its original location may treat the translated field as OOD. Galilean-shifted shallow-water states and rotated fields have the same property: the dynamics are familiar, while the observed coordinate frame is new.

Data augmentation exposes a model to more transformed examples. Equivariant architectures build selected transformations into the predictor \citep{wang2020incorporating,helwig2023group}. Both approaches improve coverage of selected transformations, but neither explicitly separates frame estimation from field evolution. When the action is known at test time, asking one predictor to infer the frame and evolve the field couples two tasks that fail differently under a novel coordinate frame.

PACE-FNO separates frame estimation from field prediction. It estimates a low-dimensional group coordinate, pulls the observation back to a canonical representative, applies an FNO, and restores the terminal frame. The terminal transformation is prescribed by the known dynamics rather than learned as an independent map. Consequently, the input and terminal transformations may differ. We train the estimator and operator jointly. During inference, one-shot prediction is the default; TTA refines the coordinate while keeping the learned weights fixed.

\paragraph{Contributions.}
We formulate PACE-FNO as a wrapper around a Fourier Neural Operator (FNO). The wrapper estimates a low-dimensional transformation coordinate, applies the inverse action to map the observed input to a canonical representative, predicts in that canonical frame, and applies the dynamics-prescribed terminal action to the output. Periodic translations are implemented spectrally. Under the exact input--terminal relation, compact canonical support, continuity, and Lipschitz conditions stated in Appendix~\ref{thm:equiv_approx}, we derive a conditional field-space error decomposition. The bound separates canonical approximation error, residual input alignment error, and residual terminal alignment error. For approximate numerical actions, it includes an additional symmetry residual. The result does not guarantee recovery of a unique physical coordinate or provide a distribution-free out-of-distribution guarantee.

We evaluate the wrapper on exact-action partial differential equation benchmarks with periodic translation and Galilean shifts. PACE-FNO reduces out-of-distribution relative error by up to $12\times$ relative to a data-augmented Fourier Neural Operator, while a matched-estimator control retains the estimator and its supervision but removes the pullback and pushforward operations. We further evaluate approximate rotation, alternative predictor backbones, an irregular-domain Poisson problem, a 300-step Kuramoto--Sivashinsky rollout, image transformations, and WeatherBench2 time-series forecasting. These experiments measure portability of the prescribed alignment mechanism and also identify settings in which the supplied geometric structure does not provide a clear advantage.

Figure~\ref{fig:overview} illustrates this decomposition: an observed field is mapped to a canonical representative, evolved there, and returned to the requested terminal frame.

\begin{figure*}[!ht]
    \centering
    \includegraphics[width=0.96\textwidth]{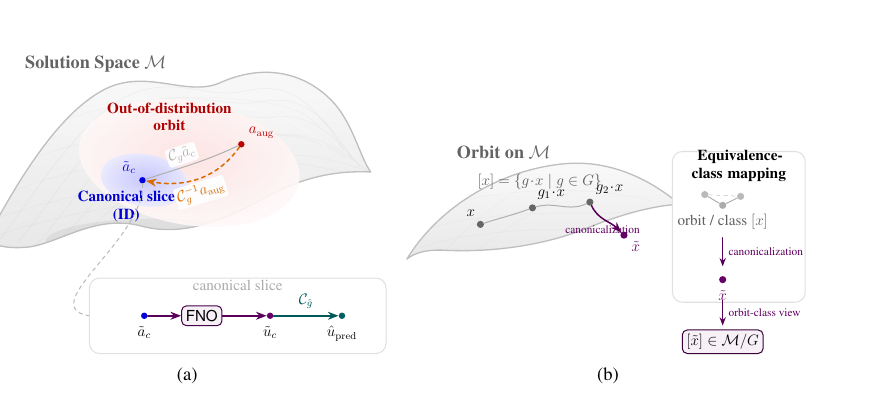}
    \caption{Geometric view of PACE-FNO. (a) The input is pulled from its observed orbit to a canonical slice, evolved there, and pushed forward to the target frame. (b) Points connected by the group action share an orbit $[x]=\{g\cdot x\mid g\in G\}$; canonicalization chooses one representative $\tilde x$ from each equivalence class.}
    \label{fig:overview}
\end{figure*}

\section{Related work}
\label{sec:related-work}

\subsection{Neural operators and geometric distribution shift}

DeepONet, Fourier Neural Operator (FNO), and later operator models learn solution maps directly from discretized functions or fields \citep{lu2019deeponet,li2020fourier,kovachki2023neural,wu2024transolver,wang2025cvit}. These methods address how the solution map is represented and evaluated on a discretized domain. They do not, by themselves, specify how an input should be transformed when the observed coordinate frame differs from the training frame. This distinction is relevant for periodic fields: a translated field can satisfy the same governing equation while appearing outside the canonical support used to train the predictor.

\subsection{Equivariance and data augmentation}

Data augmentation exposes a predictor to transformed training examples. Group-equivariant networks instead encode selected transformations in the model structure \citep{cohen2016group,finzi2020generalizing,helwig2023group,wang2020incorporating}, while Lie-group augmentation samples transformations during training \citep{brandstetter2022lie}. Both approaches can improve performance on transformations represented by the training protocol. They do not require an explicit estimate of the frame of each test input. PACE-FNO separates this estimation step from field prediction and uses the prescribed action to map the input toward a canonical representative before the predictor is applied.

\subsection{Canonicalization and Lie-group adaptation}

Canonicalization methods construct a representative for an equivalence class induced by a transformation group. Frame averaging and learned canonicalization provide architecture-agnostic symmetry mechanisms, although a continuous global representative need not exist \citep{puny2022frame,kaba2023equivariance,ma2024canonicalization,dym2024equivariant}. LieLAC is the closest prior work in spirit because it uses Lie-algebra canonicalization for neural operators under Lie-group transformations and evaluates PDE solvers \citep{shumaylov2024lie}. PACE-FNO differs in jointly learning an amortized coordinate estimator with the operator, using one-shot canonicalization as the default inference path, and allowing the known dynamics to induce different input and terminal transformations.

\subsection{Positioning of PACE-FNO}

PACE-FNO is also related to geometry-aware and intrinsically equivariant operators. Geo-FNO learns a deformation from an irregular physical domain to a uniform latent space before applying Fourier layers \citep{li2022geofno}. Clifford Neural Layers encode geometric structure directly in multivector-valued neural layers and Fourier transforms \citep{brandstetter2023clifford}. The irregular-domain and backbone-transfer experiments in this paper do not reproduce these baselines under matched protocols. They instead evaluate whether the prescribed alignment wrapper can be transferred across predictors and domains. 

Figure~\ref{fig:pace-fno-pipeline} illustrates the resulting architecture: PACE-FNO wraps a standard FNO with learned coordinate estimation and prescribed input/output actions.

\section{PACE-FNO method}
\label{sec:method}

\begin{figure}[htbp]
    \centering
    \includegraphics[width=0.94\textwidth]{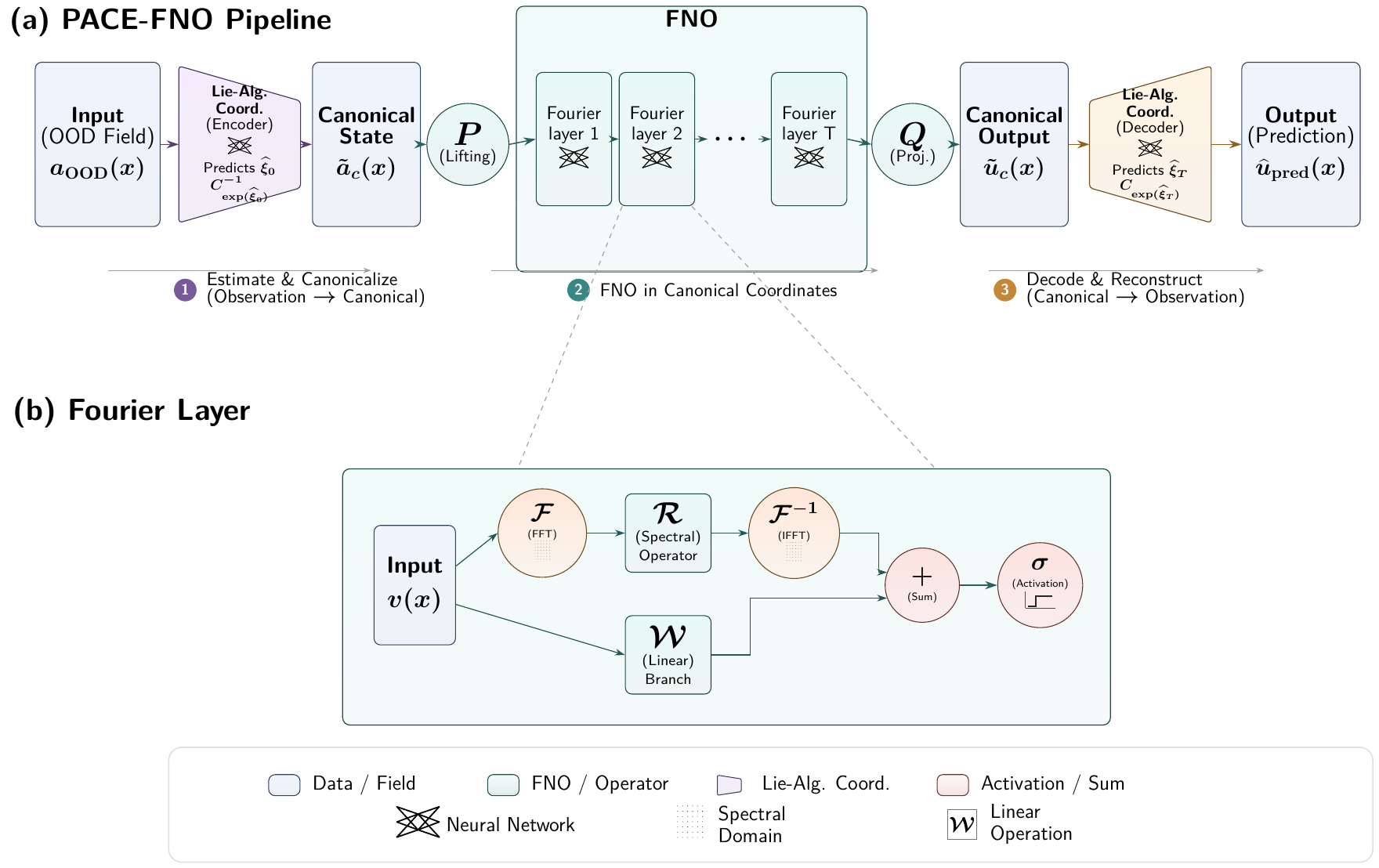}
    \caption{PACE-FNO architecture. (a) The Lie-algebra encoder estimates the input coordinate, the inverse action pulls the observation into a canonical frame, the FNO predicts there, and the terminal action restores the requested output frame. (b) A representative FNO layer retains its standard spectral and linear branches; PACE-FNO changes the surrounding coordinate path rather than the predictor architecture.}
    \label{fig:pace-fno-pipeline}
\end{figure}

\subsection{Problem formulation}

Let $\mathcal{G}^\dagger:\mathcal A\to\mathcal U$ be a physical solution operator and let $\mathcal G_\theta$ be its learned approximation, trained by
\[
\|\mathcal G^\dagger-\mathcal G_\theta\|^2_{L^2_\mu(\mathcal A;\mathcal U)}
=\mathbb E_{a\sim\mu}\|\mathcal G^\dagger(a)-\mathcal G_\theta(a)\|_{\mathcal U}^2.
\]
We take canonical inputs from $\mu_c$ and construct geometric OOD inputs as $a_{\mathrm{OOD}}=\mathcal C_g(a_c)$ for $a_c\sim\mu_c$ and $g\sim P_G$. In the exact-action experiments, this changes the frame while preserving the viscosity, forcing, solver, and boundary conditions. We write the action on fields as $[\mathcal C_g a](x)=a(g^{-1}\!\cdot x)$ and parameterize a local group element by $g=\exp(\xi)$, $\xi\in\mathfrak g=T_eG$.

When spatial and kinematic transformations are both present, we use the decomposition
\[
\mathfrak{g}=\mathfrak{g}_{\mathcal S}\oplus\mathfrak{g}_{\mathcal K},
\]
where $\mathfrak{g}_{\mathcal S}$ contains spatial frame variables and $\mathfrak{g}_{\mathcal K}$ contains problem-specific kinematic variables. This decomposition determines which coordinates are estimated, which are computed from the observed field, and which are refined by TTA.

\subsection{Kinematic reduction and coordinate parameterization}

The estimator acts on the geometric part of the frame after removing
problem-specific kinematic components that can be obtained from the observed
field. For a velocity field $\boldsymbol v$, we use its spatial mean as the
background velocity,
\[
\boldsymbol v_{\mathrm{bg}}(a)=\frac{1}{|D|}\int_D\boldsymbol v(x)\,\mathrm{d}x,
\qquad
\mathcal S_{\boldsymbol v}(a)=a-\boldsymbol v_{\mathrm{bg}}(a).
\]
The estimator receives $\mathcal S_{\boldsymbol v}(a)$ and predicts the
geometric coordinate. This separates learned spatial alignment from the
kinematic quantity used to determine the terminal transport.

Let $k=\dim(\mathfrak g_{\mathcal S})$. The coordinate network produces
$\mathcal M_\phi(a)\in\mathbb R^{2k}$. For the $j$th generator basis element
$B_j$, the two corresponding outputs are mapped to the bounded coefficient
\[
\tau_j(z)=\tau_{j,\max}
\frac{z_{2j}}{\sqrt{z_{2j-1}^2+z_{2j}^2+\varepsilon_{\mathrm{ret}}^2}},
\qquad \varepsilon_{\mathrm{ret}}=10^{-6},
\]
and the spatial estimate is
$\hat\xi_{\mathcal S}=\sum_j\tau_j(\mathcal M_\phi(\mathcal S_{\boldsymbol v}(a)))B_j$.
The full estimate combines this learned coordinate with the analytic
kinematic component. The retraction enforces the support used by the training
and OOD protocols; it does not provide a coordinate identifiability guarantee.

\subsection{Canonicalization and terminal transport}

Given $a_{\mathrm{OOD}}$, the estimator $\mathcal E_\phi$ predicts an input coordinate $\hat\xi_0$. The inverse action maps the observation toward the canonical training frame,
\[
\tilde a=\mathcal C^{-1}_{\exp(\hat\xi_0)}(a_{\mathrm{OOD}}).
\]
The predictor operates on $\tilde a$ rather than directly on the observed field. Problem-specific kinematic components, when needed, are computed from the observed field. The terminal coordinate is $\hat\xi_T=\Gamma_B(\hat\xi_0;T)$, giving
\begin{equation*}
\hat u_{\mathrm{pred}}
=\mathcal C_{\exp(\hat\xi_T)}\circ\mathcal G_\theta\circ
\mathcal C^{-1}_{\exp(\hat\xi_0)}(a_{\mathrm{OOD}}).
\end{equation*}

The input and terminal transformations may differ because the terminal coordinate is prescribed by the known dynamics rather than learned as an independent map. Periodic translations use the Fourier shift theorem. The approximate rotation experiment uses a three-shear construction. These choices keep the numerical action explicit and avoid interpolation in the periodic translation operation.

\subsection{Transformation-coordinate parameterization}

The estimator output is mapped to the task-specific support before it is used to construct a group element. Rotational variables use the two-channel parameterization above, which avoids a singular representation at the origin. The resulting coordinate is a transformation parameter, not a direct prediction of the target field. The task-specific generator sources and support bounds are listed in Appendix~\ref{app:method_details}.

In the Burgers and shallow-water settings, the observed background velocity determines the terminal transport over the prediction horizon. The spatial coordinate is estimated from the input, while the corresponding terminal displacement is obtained from the known evolution relation. This construction allows PACE-FNO to use different input and terminal actions while keeping the predictor in the canonical frame.

\subsection{Training objective}

The estimator and operator are trained jointly. The total loss is
\begin{equation*}
\mathcal L_{\mathrm{total}}
=\mathcal L_{\mathrm{PDE}}+\lambda_1\mathcal L_{\mathrm{equiv}}
+\lambda_2\mathcal L_{\mathrm{anchor}}.
\end{equation*}
The PDE term supervises the predicted field. The equivariance term compares predictions under the prescribed action, and the anchor term constrains the estimator on canonical inputs. The complete loss definitions and task-specific weights are given in Appendix~\ref{app:method_details}. The matched control used in the experiments retains the estimator, its supervision, and the same parameter budget, but removes the input pullback and terminal pushforward during prediction.

For a clean pair $(x,y)$ and an independently perturbed input
$x^+=\mathcal C_{\exp(\Delta_0)}x$, the three terms have distinct roles. The
PDE term compares the restored prediction with $y$ in field space. The
equivariance term matches the relative estimator response to $\Delta_0$,
using wrapped coordinate differences for commuting translations and a local
group difference for the rotation--translation block. The anchor term
penalizes the spatial coordinate predicted for a canonical input. This
separation keeps the physical prediction loss on the clean pair while using
the perturbed input only for the geometric consistency constraint.

\begin{algorithm}[t]
\caption{PACE-FNO training and inference}
\label{alg:pace-fno-main}
\small
\begin{algorithmic}[1]
\Require Training pairs $(x,y)$, predictor $\mathcal G_\theta$, estimator
$\mathcal E_\phi$, group action $\mathcal C$, and loss weights.
\Ensure Trained $(\theta,\phi)$ and one-shot or TTA prediction.
\For{each training step}
    \State Sample $(x,y)$ and a bounded training generator $\Delta_0$.
    \State Estimate $\hat\xi_0=\mathcal E_\phi(x)$ and compute
    $\hat\xi_T=\Gamma_B(\hat\xi_0;T)$.
    \State Form $x_c=\mathcal C_{\exp(\hat\xi_0)}^{-1}x$, predict
    $\hat u_c=\mathcal G_\theta(x_c)$, and restore
    $\hat y=\mathcal C_{\exp(\hat\xi_T)}\hat u_c$.
    \State Evaluate the PDE loss on $(\hat y,y)$, form
    $x^+=\mathcal C_{\exp(\Delta_0)}x$, and estimate
    $\hat\xi_0^+=\mathcal E_\phi(x^+)$.
    \State Evaluate the equivariance and anchor losses, then update
    $(\theta,\phi)$ with the weighted sum of all terms.
\EndFor
\State \textbf{One-shot:} apply the estimator, inverse input action, predictor,
and terminal action once.
\State \textbf{Optional TTA:} freeze both networks, refine only the spatial
coordinate using $\mathcal J_{\mathrm{TTA}}$, and recompute $\hat\xi_T$.
\end{algorithmic}
\end{algorithm}

\subsection{One-shot inference and test-time adaptation}

One-shot inference uses a single estimator pass, one inverse input action, one predictor evaluation, and one terminal action. The learned weights remain fixed. Optional test-time adaptation (TTA) refines only the spatial coordinate by minimizing canonical self-consistency,
\begin{equation*}
\mathcal J_{\mathrm{TTA}}(\hat\xi_{\mathcal S})
=\left\|\mathcal P_{\mathcal S}\circ\mathcal E_\phi\!\left(
\mathcal C^{-1}_{\exp(\hat\xi_{\mathcal S})}(a_{\mathrm{OOD}})\right)\right\|_{\mathfrak g}^{2},
\end{equation*}
where $\mathcal P_{\mathcal S}$ selects spatial generator components. TTA updates a low-dimensional coordinate while keeping both the estimator and predictor weights fixed. After each coordinate update, the terminal coordinate is recomputed from the known dynamics. The batch objective and optimization details are given in Appendix~\ref{app:method_details}.

The resulting procedure is therefore: estimate the input coordinate, compute
the analytic kinematic component when the task provides one, apply the inverse
input action, evaluate the predictor once, and restore the terminal frame.
Optional TTA repeats only the coordinate refinement step and reuses the same
predictor weights. The task-specific actions, loss coefficients, and timing
conventions are given in Appendix~\ref{app:implementation}; the complete
three-seed results are in Appendix~\ref{app:empirical}.

\section{Theory}
\label{sec:theory}

\subsection{Conditional field-space error decomposition}

Under an exact input--terminal relation, the prediction error can be separated into approximation error on the canonical support, residual input alignment error, and residual terminal alignment error. The output-action Lipschitz constant scales the canonical approximation term. For an approximate physical action, the bound includes an additional symmetry residual.

\begin{theorem}[Error decomposition, abbreviated]
\label{thm:main-error-summary}
Assume compact canonical support, a continuous exact input--terminal action, and the Lipschitz conditions stated in Appendix~\ref{proof}. If the chosen FNO has canonical approximation error at most $\epsilon$, then
\[
\sup_{a\in\mathcal A_{\mathrm{OOD}}}
\|\mathcal C_{\hat g_T}\mathcal G_\theta\mathcal C_{\hat g_0}^{-1}(a)-\mathcal G^\dagger(a)\|_{\mathcal U}
\le K_C\epsilon+K_CL_\theta\epsilon_{\mathrm{geo},0}+\epsilon_{\mathrm{geo},T}.
\]
Here $\epsilon_{\mathrm{geo},0}$ and $\epsilon_{\mathrm{geo},T}$ are field-space residuals from input and terminal alignment. For an isometric output action $K_C=1$; an approximate physical action contributes an additional symmetry residual. The complete statement and proof are in Appendix~\ref{thm:equiv_approx}.
\end{theorem}

The bound is stated in field space. It does not imply that the estimator recovers a unique generating coordinate, because symmetry-equivalent coordinates can produce the same canonical field. It also does not provide a distribution-free guarantee for arbitrary distribution shifts. The empirical target is therefore the reduction of the two field-space alignment residuals while preserving approximation on canonical support.

\subsection{Local stability of test-time adaptation}

The TTA objective has a local stability property under the conditions stated below.

\begin{proposition}[Local stability of TTA]
\label{prop:tta-stability-main}
Let $\mathcal J_{\mathrm{TTA}}$ be differentiable near its zero set
$\Xi^*=\{\xi:\mathcal J_{\mathrm{TTA}}(\xi)=0\}$. If its gradient is $L$-Lipschitz there and, for some $\alpha>0$,
$\mathcal J_{\mathrm{TTA}}(\xi)\geq\alpha\,\mathrm{dist}(\xi,\Xi^*)^2$, and if the initial sublevel set is compactly contained in that neighborhood, then gradient flow
$\dot\xi=-\nabla\mathcal J_{\mathrm{TTA}}(\xi)$ satisfies
\begin{equation*}
\frac{\mathrm d}{\mathrm dt}\mathcal J_{\mathrm{TTA}}(\xi(t))
=-\|\nabla\mathcal J_{\mathrm{TTA}}(\xi(t))\|_{\mathfrak g}^{2}\leq 0,
\qquad
\mathrm{dist}(\xi(t),\Xi^*)
\leq \alpha^{-1/2}\mathcal J_{\mathrm{TTA}}(\xi(t))^{1/2}.
\end{equation*}
The complete hypotheses and proof are given in Appendix~\ref{thm:tta_local_stability_qg}.
\end{proposition}

This result is local. It does not claim global convergence for arbitrary OOD inputs or recovery of a unique physical frame. Symmetry-related minimizers may form a manifold.

\subsection{Scope of the theoretical statements}

The exact input--terminal relation is used for the periodic translation and Galilean experiments. The arbitrary-angle Navier--Stokes rotation is treated as an approximate numerical action and therefore carries a separate symmetry residual. The theory does not assume that the transformation support used for OOD evaluation is itself a subgroup.

\section{Experiments}
\label{sec:experiments}

\subsection{Experimental protocol}

We evaluate periodic 1-D and 2-D Burgers, shallow-water (SWE), and Navier--Stokes (NS) systems under transformation ranges excluded from training. We report in-distribution (ID) and out-of-distribution (OOD) results. Fourier Neural Operator (FNO), data-augmented FNO (FNO+Aug), and group-equivariant Fourier neural operator (GFNO) follow the reported protocols \citep{li2020fourier,takamoto2022pdebench,li2021markov,wang2020incorporating,helwig2023group}. Each learned result averages seeds $\{42,1234,3407\}$ with identical data splits per method.

The suite covers exact-action PDE shifts, an approximate rotation diagnostic,
backbone and domain transfer, mechanism controls, vision shifts, and
WeatherBench2 time OOD. The protocols and measurement conventions are in
Appendix~\ref{app:implementation}; complete tables and diagnostic figures are
in Appendix~\ref{app:empirical}.

\subsection{Exact-action PDE benchmarks}

The periodic tests use cross-resolution OOD evaluation ($1024\!\to\!2048$ in 1-D Burgers and $64^2\!\to\!128^2$ otherwise); full equations, transformation ranges, hyperparameters, runtimes, and uncertainty tables are in Appendices~\ref{app:implementation} and~\ref{app:empirical}. Table~\ref{tab:main-results} collects the primary evidence. Its first three rows use exact periodic translation and Galilean actions. The NS rotation--translation row uses an approximate arbitrary-angle numerical rotation, so it is a diagnostic evaluation rather than an exact-equivariance guarantee.

Unless stated otherwise, the primary metric is relative field error,
\[
\mathrm{RelErr}(\hat u,u)=\frac{\|\hat u-u\|_{L^2}}{\|u\|_{L^2}},
\]
which measures the error after returning the prediction to the requested output frame.

On the exact-action benchmarks, one-shot PACE-FNO reduces OOD error from
$0.3222$ to $0.0267$ on 1-D Burgers, from $0.2035$ to $0.0755$ on 2-D
Burgers, and from $0.5136$ to $0.1069$ on SWE relative to FNO+Aug. TTA
provides a smaller Burgers improvement at higher inference cost and leaves
the SWE value unchanged at the reported precision. The NS change from
$0.2840\pm0.2001$ to $0.2766\pm0.0353$ is small relative to the baseline
spread, so we treat it as a diagnostic rather than a headline improvement.

On the same 1-D Burgers split, a dedicated LieLAC+FNO run gives $0.0342\pm0.0119$ OOD error and $8.73$~s inference, versus $0.0267\pm0.0014$ and $0.53$~s for one-shot PACE-FNO (Table~\ref{tab:burgers_1d_main}). The comparison contrasts amortized one-shot inference with per-instance iterative optimization; timing conventions are documented in the table caption.

The OOD reduction is accompanied by higher ID error on Burgers: $0.0623$
versus $0.0293$ on 1-D Burgers and $0.0393$ versus $0.0261$ on 2-D Burgers.
On SWE and NS, one-shot PACE-FNO has lower ID error than FNO+Aug. The
complete ID/OOD tables are in Appendix~\ref{app:empirical}.

\begin{table*}[t]
\caption{Primary geometric-OOD results (mean $\pm$ standard deviation over three seeds). One-shot PACE-FNO is the default operating point. The NS rotation and translation row includes an approximate numerical rotation; all other rows use the stated exact continuum actions.}
\label{tab:main-results}
\centering
\small
\setlength{\tabcolsep}{6pt}
\resizebox{0.98\textwidth}{!}{%
\begin{tabular}{lcccc}
\toprule
Problem and action & FNO & FNO+Aug & PACE-FNO (one-shot) & PACE-FNO + TTA \\
\midrule
1-D Burgers, translation + Galilean & $1.0352\pm0.0001$ & $0.3222\pm0.0077$ & $0.0267\pm0.0014$ & $\mathbf{0.0259\pm0.0007}$ \\
2-D Burgers, translation + Galilean & $0.6137\pm0.0362$ & $0.2035\pm0.0423$ & $0.0755\pm0.0139$ & $\mathbf{0.0640\pm0.0012}$ \\
2-D SWE, Galilean & $0.7260\pm0.0003$ & $0.5136\pm0.0010$ & $\mathbf{0.1069\pm0.0168}$ & $0.1071\pm0.0168$ \\
NS, translation + approximate rotation & $0.6422\pm0.0302$ & $0.2840\pm0.2001$ & $0.2766\pm0.0353$ & $\mathbf{0.2593\pm0.0522}$ \\
\bottomrule
\end{tabular}
}
\end{table*}

\begin{figure}[htbp]
    \centering
    \begin{subfigure}[t]{0.52\textwidth}
        \centering
        \includegraphics[width=\linewidth]{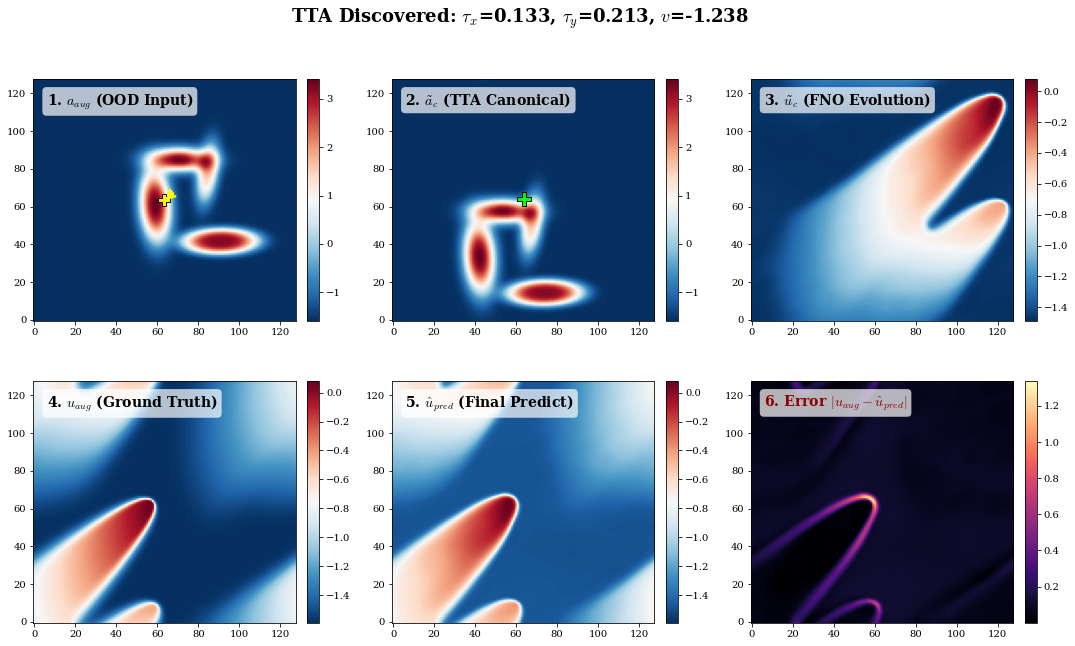}
        \caption{2-D Burgers under translation and Galilean drift.}
    \end{subfigure}\hfill%
    \begin{subfigure}[t]{0.46\textwidth}
        \centering
        \includegraphics[width=\linewidth]{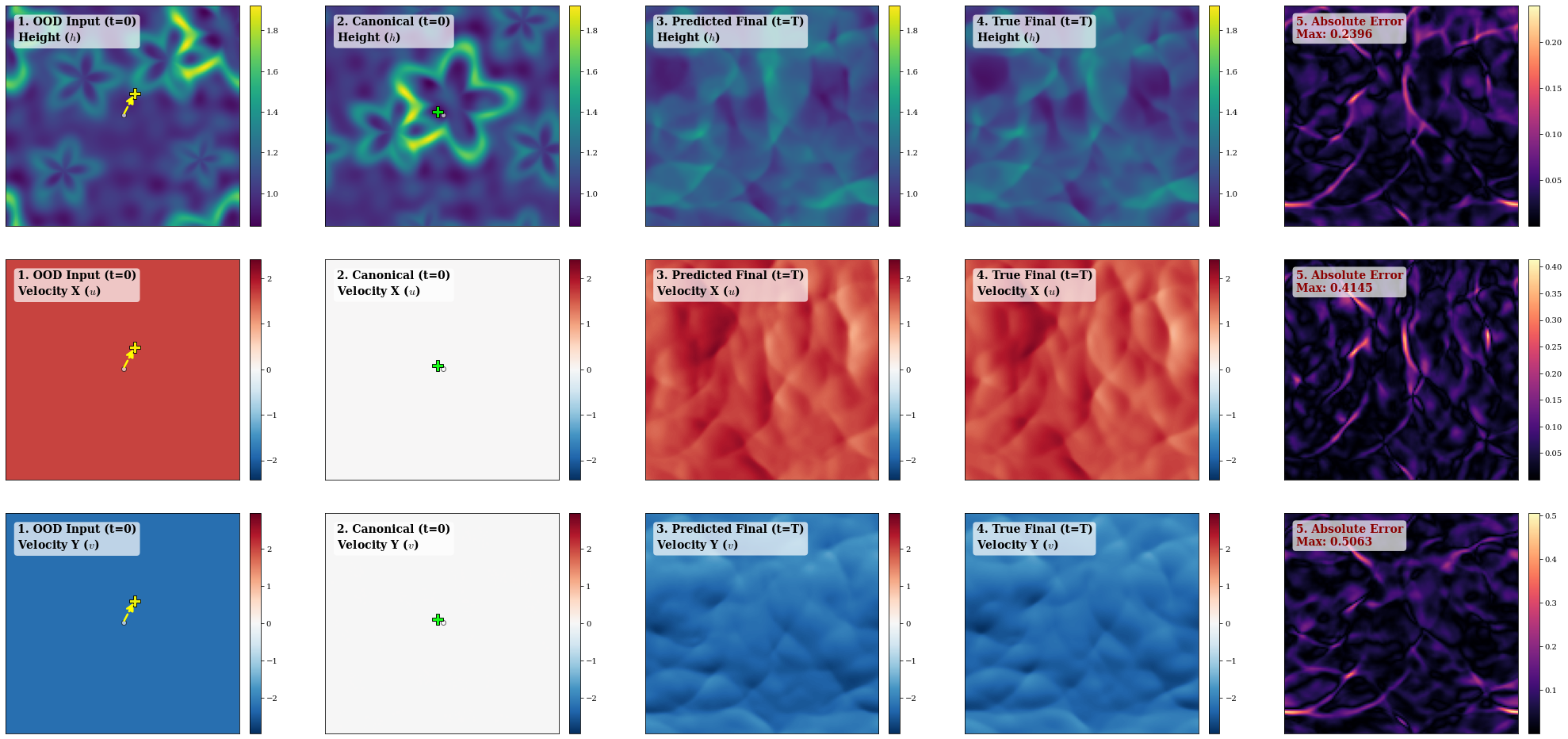}
        \caption{SWE under Galilean perturbation.}
    \end{subfigure}
    \caption{Representative exact-action PDE predictions. Full method comparisons and generation protocols are retained in Appendices~\ref{app:implementation} and~\ref{app:empirical}.}
    \label{fig:core-pde}
\end{figure}

\subsection{Controls and adaptation diagnostics}
\label{sec:mechanism-scope}

\paragraph{The matched control isolates the mechanism.} It uses the same FNO, the same $297{,}156$-parameter estimator, and the same frame supervision as PACE-FNO, but removes the pullback and pushforward. On standard 2-D Burgers OOD it attains $0.66021\pm0.03644$ relative error, compared with $0.05863\pm0.00287$ for PACE-FNO. The comparison tests whether the pullback and pushforward, rather than added estimator capacity, create the canonical support used by the predictor.

Figure~\ref{fig:tta-landscape} illustrates the low-dimensional TTA objective when one-shot alignment leaves a residual frame error. The objective has symmetry-induced low-energy basins. Its non-convexity also explains why the refinement has a latency cost and can be sensitive to initialization.

Noise and sparse observations test the same mechanism without changing the predictor. With $\alpha=0.10$ observation noise, PACE-FNO reports $0.05867\pm0.00291$, while the matched auxiliary-estimator control reports $0.66016\pm0.03650$. At $50\%$ randomly retained observations, their errors are $0.2397\pm0.0031$ and $0.6622\pm0.0413$. Structured missingness raises PACE-FNO to $0.3244\pm0.0016$, still below FNO+Aug at $0.3560\pm0.0119$. These checks do not claim that canonicalization resolves arbitrary sensing loss.

Figures~\ref{fig:tta-landscape} and \ref{fig:tta-sensitivity-swe-main} in the appendix show the symmetry-induced TTA basins and the nearly unchanged SWE response over the tested learning rates and step budgets.

\subsection{Backbone and domain transfer}

We retain the alignment procedure and replace only the PDE predictor. On 2-D Burgers, PACE-MambaGalerkin reduces OOD error from $0.32505\pm0.11714$ to $0.04246\pm0.00048$; on SWE, PACE-Transolver reduces it from $0.78036\pm0.06026$ to $0.33906\pm0.00456$. These replacements use the Mamba Neural Operator, Transolver, and Continuous Vision Transformer (CViT) as the underlying predictor families \citep{cheng2026mamba,wu2024transolver,wang2025cvit}; CViT improves on average, although its seed variance prevents a precise ranking. The Mamba comparison includes a data-augmented baseline, whereas CViT+Aug and Transolver+Aug were not run; those two rows therefore measure portability relative to the listed base models, not a fully augmentation-matched wrapper effect.

The appendix also reports runtimes for these one-shot transfer settings. PACE-MambaGalerkin uses $22.876$ ms per OOD sample and adds $11.4\%$ training time relative to Mamba. PACE-Transolver uses $18.061$ ms per sample, versus $20.064$ ms for Transolver, while increasing training time by $53.4\%$. The corresponding parameter increases are $6.12\%$ and $2.65\%$, respectively (Table~\ref{tab:backbone_transfer}).

We also evaluate the wrapper on an irregular unit-ball Poisson task, where PACE-Transolver reduces $90^\circ$ rotation OOD error from $0.3000\pm0.0470$ to $0.1815\pm0.0178$. In a 300-step closed-loop Kuramoto--Sivashinsky rollout, PACE-FNO reaches $0.35919\pm0.07196$ versus $0.46978\pm0.02451$ for FNO+Aug. Table~\ref{tab:scope-checks-main} collects these results, which support portability of the prescribed wrapper rather than discovery of an unknown symmetry.

\begin{table}[!ht]
\centering
\caption{Additional scope checks under their fixed protocols (three seeds). These rows test portability and protocol-specific limitations, not the primary exact-action claim. Resource cells report Reference/PACE-FNO in that order; peak memory is allocated CUDA memory in MiB from the same RTX~4080 runs.}
\label{tab:scope-checks-main}
\small
\setlength{\tabcolsep}{3.5pt}
\resizebox{\columnwidth}{!}{%
\begin{tabular}{llccccc}
\toprule
Setting & Reference predictor & Reference OOD & PACE-FNO OOD &
\shortstack{Total Params\\(Ref./PACE)} &
\shortstack{Peak Train MiB\\(Ref./PACE)} &
\shortstack{Peak Infer MiB\\(Ref./PACE)} \\
\midrule
2-D Burgers & MambaGalerkin+Aug & $0.32505\pm0.11714$ & $\mathbf{0.04246\pm0.00048}$ &
$4{,}854{,}177/5{,}151{,}333$ & $3108.2/3124.8$ & $3270.4/3221.7$ \\
2-D SWE & Transolver & $0.78036\pm0.06026$ & $\mathbf{0.33906\pm0.00456}$ &
$11{,}229{,}251/11{,}526{,}983$ & $967.4\pm44.4/1034.8\pm46.3$ & $1489.6/1501.4$ \\
3-D Poisson-ball ($90^\circ$) & Transolver & $0.3000\pm0.0470$ & $\mathbf{0.1815\pm0.0178}$ &
$2{,}808{,}129/3{,}068{,}676$ & $644.8/701.8$ & $94.6/96.7$ \\
1-D KS (300 steps) & FNO+Aug & $0.46978\pm0.02451$ & $\mathbf{0.35919\pm0.07196}$ &
$21{,}334{,}273/21{,}359{,}363$ & $667.4\pm215.3/532.3\pm94.1$ & $134.53/135.12$ \\
\bottomrule
\end{tabular}
}%
\end{table}

For the Navier--Stokes rotation diagnostic, the action itself is approximate. The corresponding residual is
\[
\epsilon_{\mathrm{sym}}=
\sup_{a\in\mathcal A_{\mathrm{OOD}}}
\left\|\mathcal G^\dagger(a)-
\mathcal C_{g_T(a)}\mathcal G^\dagger(a_c(a))\right\|_{\mathcal U},
\]
which adds to the exact-action bound. The residual keeps numerical action error separate from estimator and predictor error.

Figure~\ref{fig:ns-scope} visualizes this approximate-action diagnostic; the residual is a property of the numerical transformation, not an unreported change to the PDE task.

\begin{figure}[!ht]
    \centering
    \includegraphics[width=0.86\linewidth]{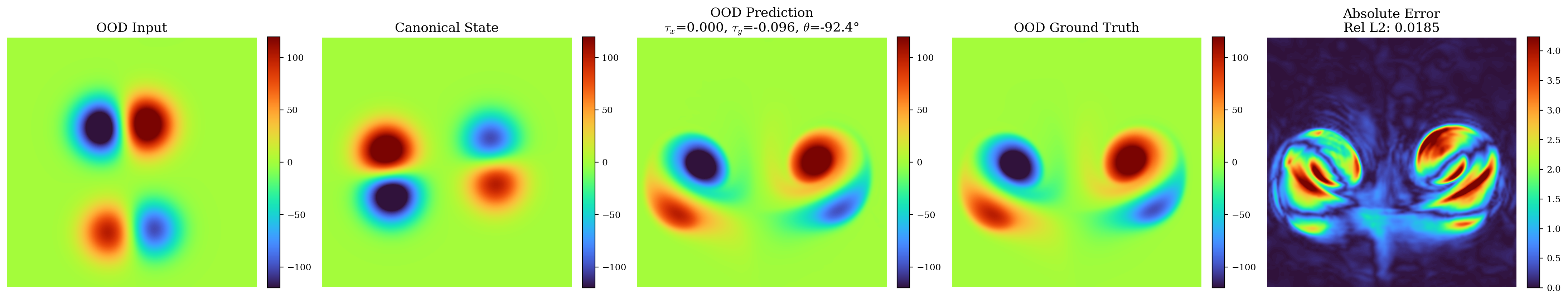}
\caption{Navier-Stokes rotation and translation diagnostic. Translation is exact; arbitrary-angle rotation is the stated approximate three-shear numerical action and therefore carries a symmetry residual.}
    \label{fig:ns-scope}
\end{figure}

\subsection{Image and WeatherBench2 evaluations}

\subsubsection{Image transformations}

Camelyon17 uses a six-dimensional H\&E stain/color transformation. affNIST
and homNIST use affine and projective spatial warps. The figures are fixed
diagnostics; quantitative claims use the complete test sets.

\begin{figure}[htbp]
    \centering
    \begin{subfigure}[t]{0.47\textwidth}
        \centering
        \includegraphics[width=\linewidth]{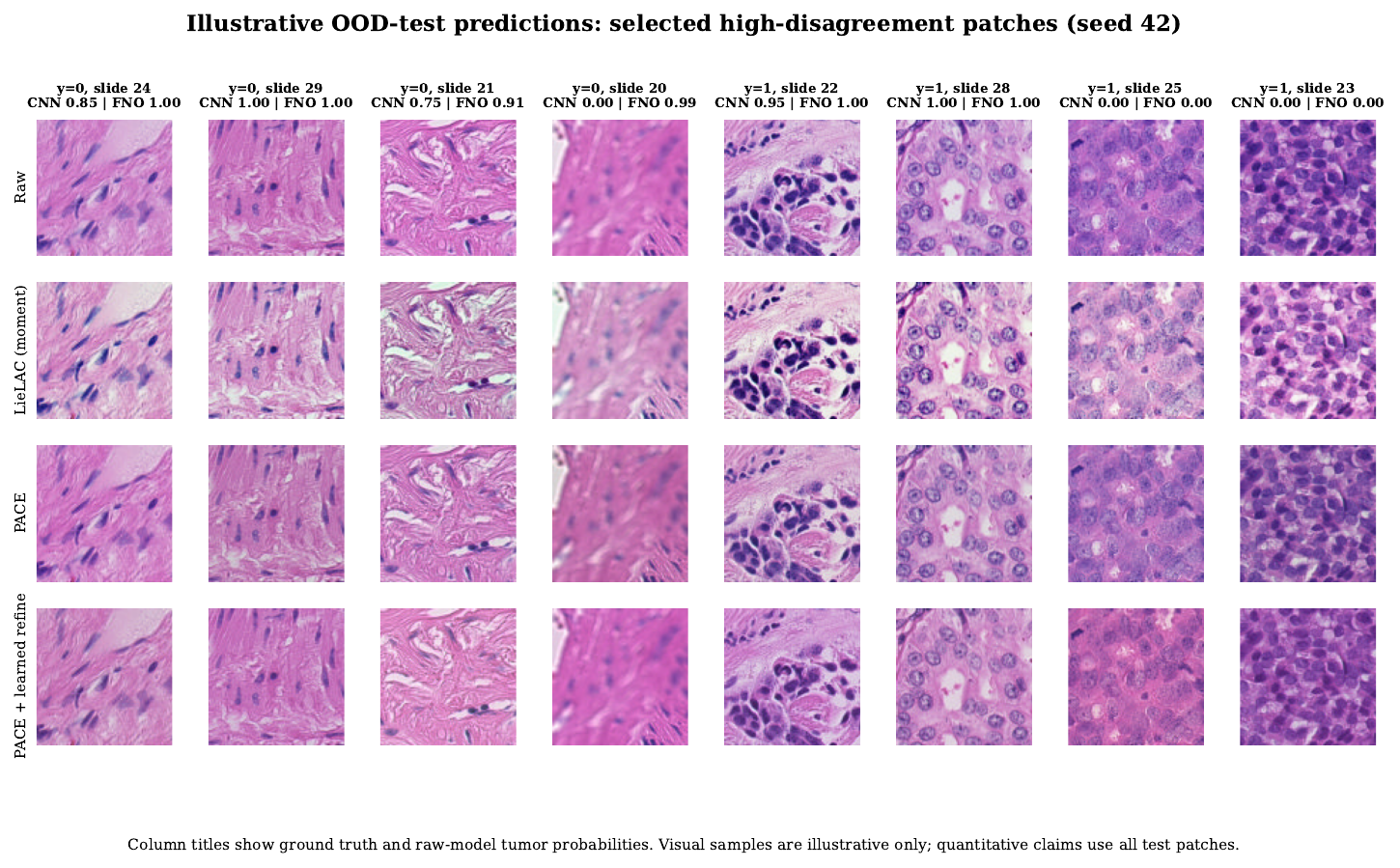}
        \caption{Camelyon17 hospital-2 patches with high method disagreement.}
        \label{fig:camelyon-predictions}
    \end{subfigure}\hfill
    \begin{subfigure}[t]{0.47\textwidth}
        \centering
        \includegraphics[width=\linewidth]{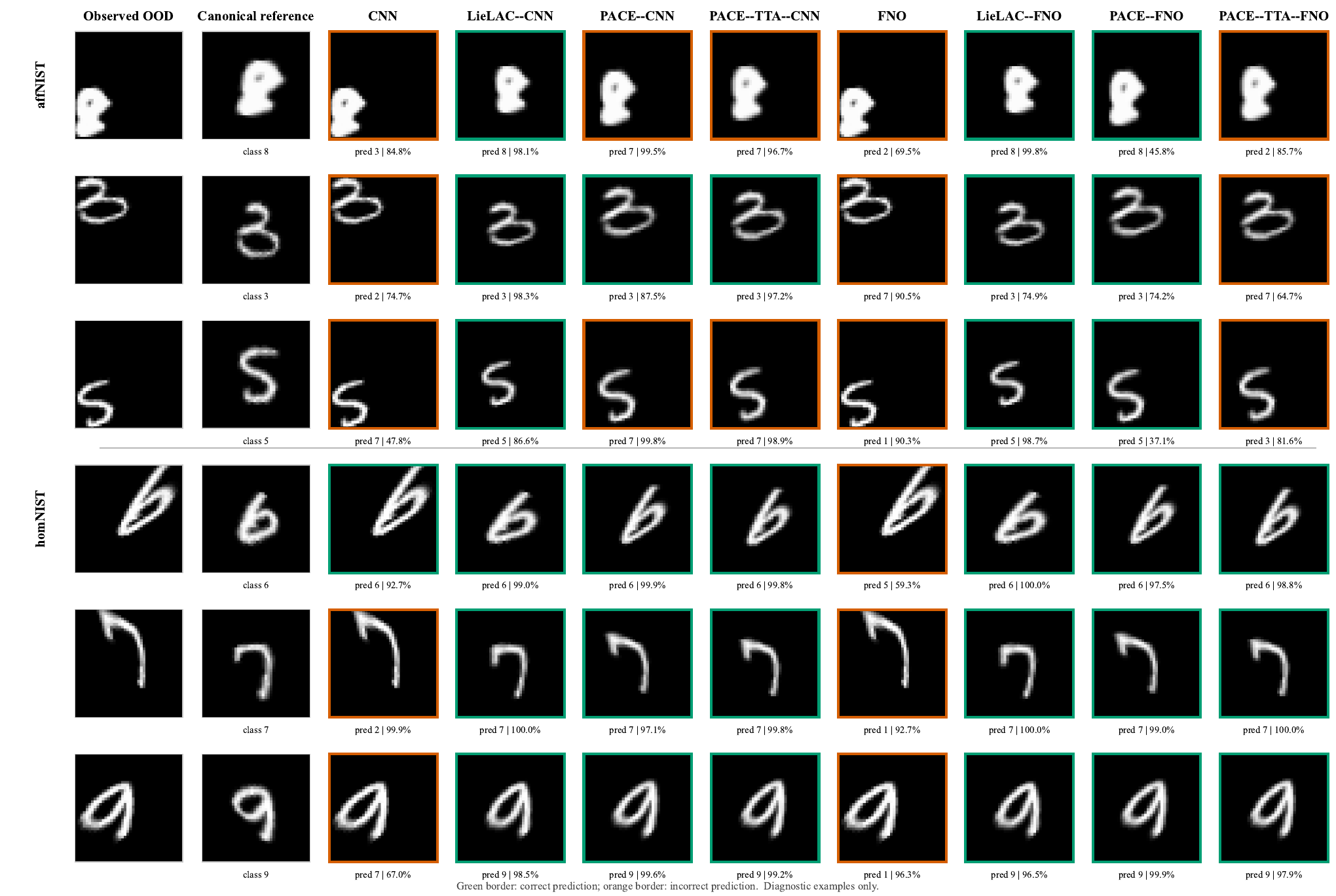}
        \caption{Selected affNIST and homNIST examples.}
        \label{fig:digits-predictions}
    \end{subfigure}
    \caption{Qualitative vision diagnostics. Quantitative claims use the complete test sets.}
\end{figure}

On Camelyon17 hospital~2, learned TTA reaches $83.31\%$ OOD accuracy,
$12.36$ points above raw FNO and $2.68$ points above learned LieLAC, at
$0.597$ rather than $21.854$ ms per patch. The seed deviations preclude a
significance claim for these point differences. One-shot PACE-FNO reaches
$89.953\%$ on affNIST and $97.180\%$ on homNIST, while LieLAC reaches
$95.503\%$ and $98.433\%$ at much higher latency. Complete results are in
Tables~\ref{tab:camelyon-main} and~\ref{tab:digits-main}.

\subsubsection{WeatherBench2 boundary case}

WeatherBench2 tests natural time OOD rather than a prescribed geometric
action. The audit contains 353 ID-test, 1,059 time-OOD, and 109 pre-defined
extreme-OOD windows; Figure~\ref{fig:weather-audit} summarizes these roles.

At direct $+72$-h lead, FNO and PACE-FNO differ by less than $0.003$ on
the aggregate metrics. PACE-FNO processes $600.8$ samples/s with $1.59$ GiB
peak GPU memory, versus $669.0$ samples/s and $1.39$ GiB for FNO; SFNO
reaches $919.1$ samples/s with $0.62$ GiB. The full measurements are in
Tables~\ref{tab:weatherbench2_main} and~\ref{tab:weatherbench2_efficiency}.
These results do not support a universal claim for temporal distribution shift.

The fixed extreme-OOD case shows similar large-scale T2m and Z500 patterns,
with local residual differences (Figure~\ref{fig:weather-fields}).

\begin{figure}[htbp]
    \centering
    \begin{subfigure}[t]{0.31\textwidth}
        \centering
        \includegraphics[width=\linewidth]{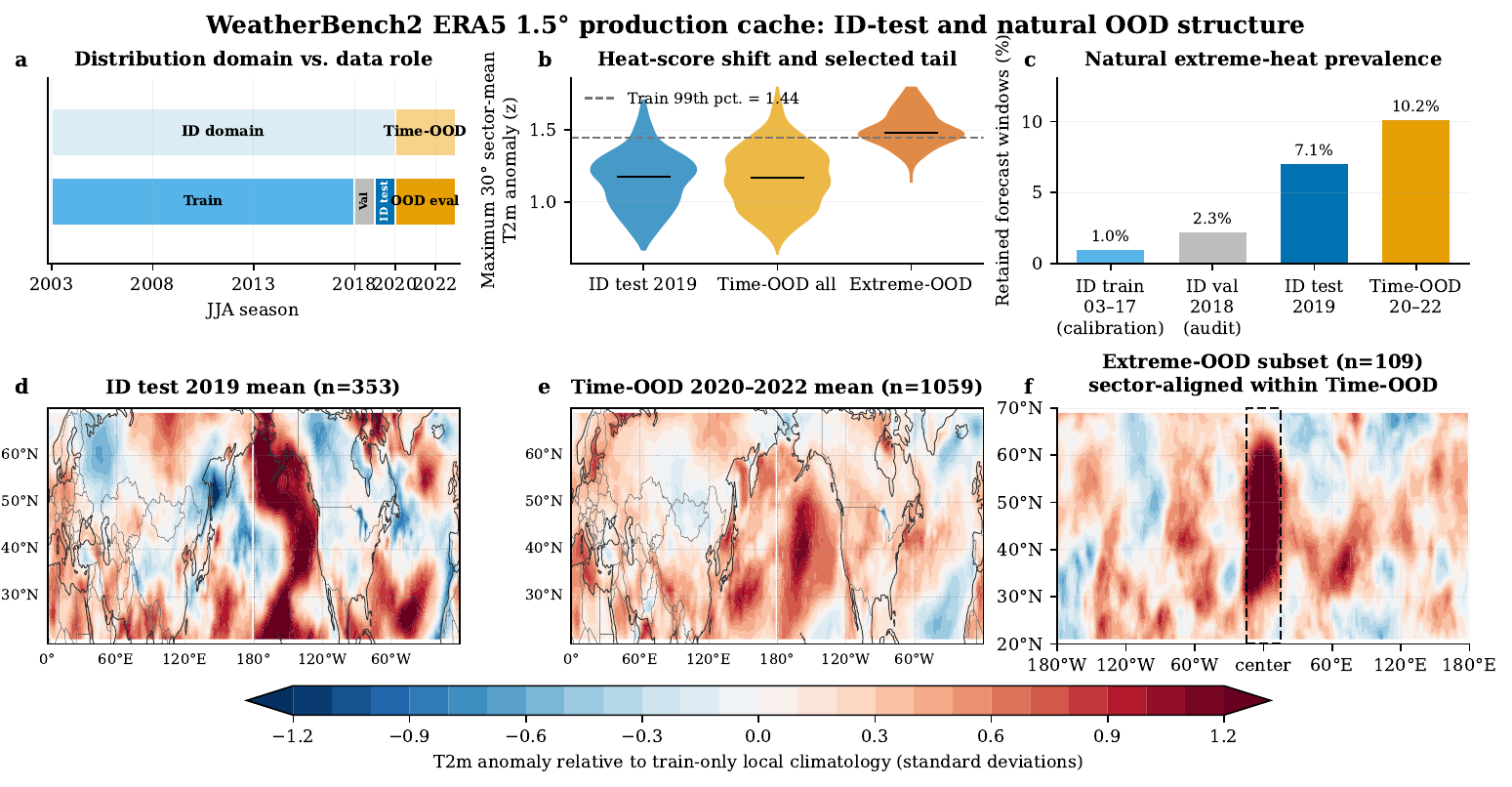}
        \caption{Data roles and natural-OOD audit.}
        \label{fig:weather-audit}
    \end{subfigure}\hfill
    \begin{subfigure}[t]{0.36\textwidth}
        \centering
        \includegraphics[width=\linewidth]{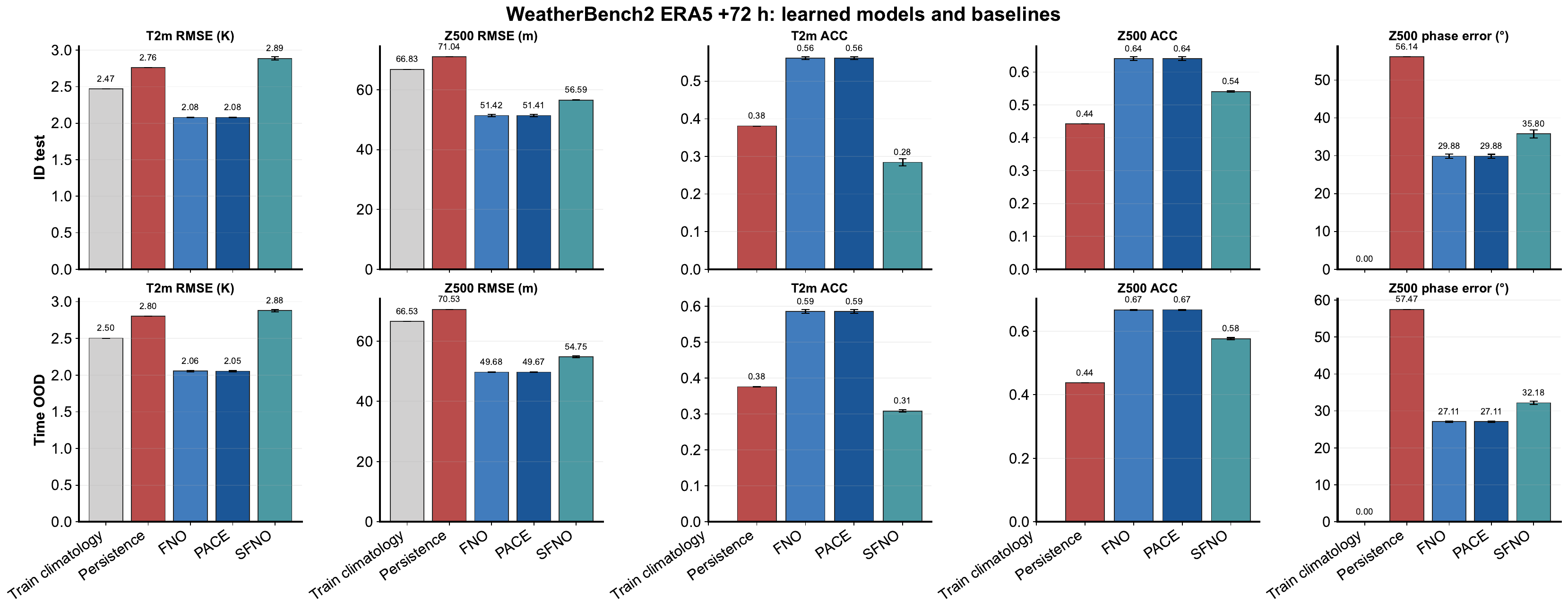}
        \caption{$+72$-h ID and time-OOD metrics.}
        \label{fig:weather-boundary}
    \end{subfigure}\hfill
    \begin{subfigure}[t]{0.28\textwidth}
        \centering
        \includegraphics[width=\linewidth]{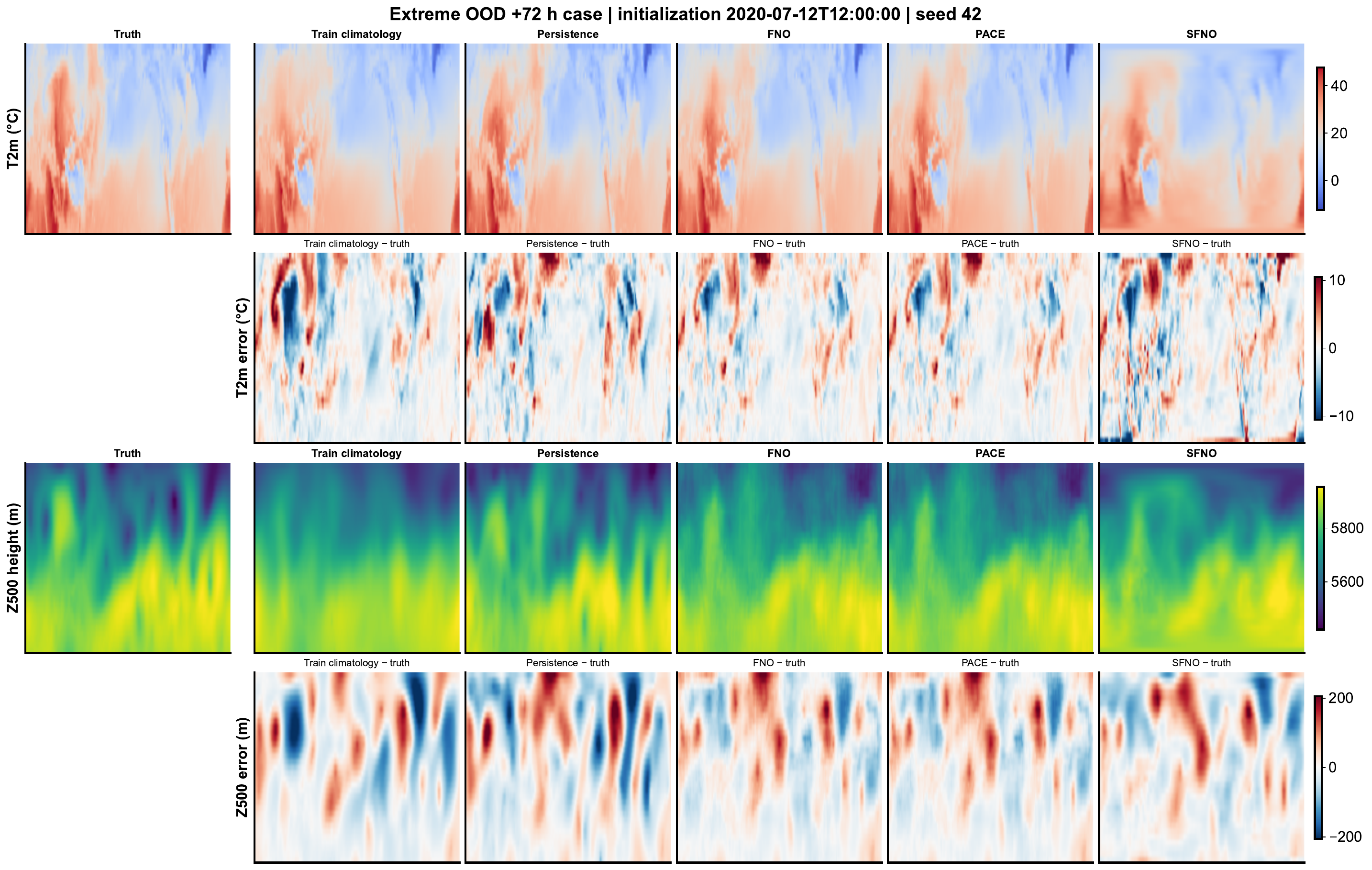}
        \caption{Fixed extreme-OOD fields.}
        \label{fig:weather-fields}
    \end{subfigure}
    \caption{WeatherBench2 diagnostics. The complete metrics and efficiency measurements are reported in Appendix~\ref{app:empirical}.}
\end{figure}

\FloatBarrier
\section{Conclusion}
\label{sec:conclusion}

PACE-FNO canonicalizes the input, predicts in the canonical frame, and
restores the requested output frame through known dynamics. The theory and
experiments support this design when the transformation action is explicit and
evaluable in field space.

\clearpage
\subsection*{Reproducibility statement}

The main text states the evaluation protocol and claim boundaries. The appendix
provides the governing equations, transformation supports, dataset splits,
three-seed convention, hyperparameters, hardware and runtime details, complete
tables, and proofs. Supplementary implementation and diagnostic materials are
not included in this arXiv source package.

\subsection*{AI use statement}

We used generative AI tools for editorial assistance, including revising the
manuscript structure and prose, figure selection, layout, and captions,
clarifying the scope of existing results, checking \LaTeX{} consistency, and
reviewing the mathematical arguments. This process identified a missing
multiplicative constant in one derived error bound; the bound was corrected and
checked against the theorem's assumptions. The tools did not run experiments or
create experimental data, numerical results, or citations. The authors
independently reviewed the final text, claims, equations, proofs, figures,
tables, citations, and cross-references and take responsibility for the
submission.

\bibliographystyle{plainnat}
\bibliography{references}
\clearpage
\appendix

\section{Additional Method Details}
\label{app:method_details}

This appendix defines the quantities used in the main text and lists the training details.

Let $G$ denote the problem-dependent ambient transformation group. Training and evaluation draw generators from bounded supports $S_{\mathrm{train}},S_{\mathrm{test}}\subset G$; these supports are not assumed to be subgroups. When both spatial and kinematic components are present, the corresponding Lie algebra admits the decomposition
\[
\mathfrak{g} = \mathfrak{g}_{\mathcal S} \oplus \mathfrak{g}_{\mathcal K}.
\]

\paragraph{Kinematic reduction.}
Let $\boldsymbol{v} \in L^1(D;\mathbb{R}^d)$ denote the velocity field. The kinematic generator $\boldsymbol{v}_{\mathrm{bg}} \in \mathfrak{g}_{\mathcal K}$ and the centering operator $\mathcal{S}_{\boldsymbol v}:\mathcal{A}\to\mathcal{A}$ are defined by
\begin{equation*}
    \boldsymbol{v}_{\mathrm{bg}}
    =
    \mathcal{P}_{\mathrm{phys}}(a)
    :=
    \frac{1}{|D|}\int_D \boldsymbol{v}(x)\,\mathrm{d}x,
    \qquad
    \mathcal{S}_{\boldsymbol v}(a)
    =
    a - \boldsymbol{v}_{\mathrm{bg}}.
\end{equation*}
This co-moving decomposition matches the usual interpretation of background transport and momentum removal \citep{pope2000turbulent,sommerfeld1952mechanics}.

\paragraph{Neural parameterization.}
To parameterize the geometric subalgebra $\mathfrak{g}_{\mathcal S}$ without singularities, let $k=\dim(\mathfrak{g}_{\mathcal S})$ and use a neural mapping
\[
\mathcal{M}_\phi : \mathcal{A} \to \mathbb{R}^{2k},
\]
followed by the smooth bounded retraction
\begin{equation*}
\mathrm{Ret}(z)=\sum_{j=1}^{k}\tau_j(z)B_j,
\qquad
\tau_j(z)=\tau_{j,\max}\frac{z_{2j}}{\sqrt{z_{2j-1}^{2}+z_{2j}^{2}+\varepsilon_{\mathrm{ret}}^{2}}},
\qquad \varepsilon_{\mathrm{ret}}=10^{-6},
\end{equation*}
where $\{B_j\}_{j=1}^{k}$ is the fixed Lie-algebra basis, $z=\mathcal M_\phi(a)\in\mathbb R^{2k}$, and $\tau_{j,\max}$ is the task-specific support bound. The retraction constrains each coordinate to its task-specific range. The positive regularizer $\varepsilon_{\mathrm{ret}}$ prevents a singularity at $z=0$. We use this parameterization for rotational variables \citep{falorsi2018explorations,shumaylov2024lie}. The spatial generator is
\begin{equation*}
    \xi_{\mathcal S}
    =
    \mathrm{Ret}\!\left(
    \mathcal{M}_\phi(\mathcal{S}_{\boldsymbol v}(a))
    \right).
\end{equation*}
The full Lie-algebra coordinate estimate is then
\begin{equation*}
    \mathcal{E}_\phi(a)
    =
    \xi_{\mathcal S} \oplus \boldsymbol{v}_{\mathrm{bg}}.
\end{equation*}

\paragraph{Spectral transformation implementation.}
Given the full Lie algebra coordinate $\hat{\xi}\in\mathfrak g$, we obtain the corresponding transformation parameter through $\hat{g}=\exp(\hat{\xi})$. To avoid dissipation induced by interpolation, the transformation is implemented spectrally. For translations, the continuous Fourier shift theorem gives \citep{bracewell1986fourier}
\begin{equation*}
    \mathcal{F}(\mathcal{T}_{\boldsymbol{\tau}_{\mathrm{trans}}}\circ a)(\mathbf{k})
    =
    \exp\!\big(-2\pi i \langle \mathbf{k}, \boldsymbol{\tau}_{\mathrm{trans}}\rangle\big)\hat{a}(\mathbf{k}).
\end{equation*}
Periodic translations implement an exact torus action up to floating-point and spectral truncation error. The three-shear map at an arbitrary angle is instead a numerical rotated-field transformation: except for lattice-compatible angles, it is not an exact automorphism of the square torus and is treated as an approximate action in the Navier--Stokes experiment.
For rotations, we use a three-shear factorization in the spectral domain \citep{paeth1990fast,unser1995convolution}. With $\mathcal H_x(a)(x,y)=(x+ay,y)$ and $\mathcal H_y(b)(x,y)=(x,y+bx)$,
\begin{equation*}
    R_\theta=H_x(\alpha)H_y(\beta)H_x(\alpha),
    \qquad \alpha=-\tan(\theta/2),\qquad \beta=\sin\theta,
\end{equation*}
and the transformed spectrum applies the corresponding three shear operators,
\begin{equation*}
    \mathcal{F}(\mathcal{R}_\theta\circ a)=
    \left(\widetilde{\mathcal H}_x(\alpha)\circ
    \widetilde{\mathcal H}_y(\beta)\circ
    \widetilde{\mathcal H}_x(\alpha)\right)\mathcal F(a).
\end{equation*}

\paragraph{Training-time perturbations.}
We expose the learned operator to controlled geometric shifts by sampling bounded perturbation generators $\xi_0$ from a restricted training support $\Omega_{\mathrm{train}}\subsetneq\Omega_{\mathrm{OOD}}$ and generating
\begin{equation*}
    a_{\mathrm{aug}}
    =
    \mathcal{C}_{\exp(\xi_0)}(a),
\end{equation*}
for canonical samples $a\sim\mu_c$, following the logic of restricted Lie-group augmentation \citep{brandstetter2022lie}. The same sampled generator defines the perturbed training target through the kinematic map $\xi_T=\Gamma_B(\xi_0;T)$. These sampled values are used only to construct training pairs and alignment losses; PACE-FNO does not receive them during in-distribution (ID) or out-of-distribution (OOD) evaluation.

Throughout this appendix, a parameter range attached to a ``training perturbation'' or $\Omega_{\mathrm{train}}$ describes only augmented training pairs. The OOD range is the independently generated support $\Omega_{\mathrm{OOD}}$ used for the reported evaluation tables; it is never reused as a training range. For SWE both supports are given numerically below. For the Burgers tasks, the displayed uniform boxes are the training supports ($\delta_x\in[-0.1,0.1]$, $v_{\mathrm{bg}}\in[-0.2,0.2]$ for 1-D; $t_x,t_y\in[-0.1,0.1]$, $U_{\mathrm{bg}}\in[-0.5,0.5]$ for 2-D), while the OOD sampler is a separate held-out distribution at the evaluation resolution and is not to be inferred from those training boxes.

\paragraph{Source of input and terminal generators.}
At inference, PACE-FNO receives only the observed field. The estimator predicts the spatial alignment coordinate, analytic reductions compute kinematic backgrounds from the physical channels, and the stated PDE relation determines the terminal transformation. The coordinate need not identify a unique physical frame. Table~\ref{tab:generator_sources} lists the source of each quantity.

\begin{table}[htbp]
\centering
\caption{Source of the generators used by PACE-FNO. Learned quantities come from $\mathcal{E}_\phi$ at inference; analytic quantities come from the observed field or the symmetry-induced transport rule. Sampled generators are used only for training perturbations and alignment losses.}
\label{tab:generator_sources}
\small
\begin{tabular}{
>{\raggedright\arraybackslash}p{0.13\textwidth}
>{\raggedright\arraybackslash}p{0.20\textwidth}
>{\raggedright\arraybackslash}p{0.31\textwidth}
>{\raggedright\arraybackslash}p{0.24\textwidth}}
\toprule
Test case & Learned at inference & Analytic at inference & Training-only sampled generator \\
\midrule
1-D Burgers
& Input translation phase $\hat{\delta}_x$
& Background offset $\hat v_{\mathrm{bg}}$ from the field mean; terminal shift $\hat{\delta}_{x,T}=\hat{\delta}_x+\hat v_{\mathrm{bg}}T$
& $\delta_x$, $v_{\mathrm{bg}}$ construct consistency perturbations and alignment losses \\
2-D Burgers
& Input translation $\hat{\mathbf t}=(\hat t_x,\hat t_y)$
& Scalar background $\hat U_{\mathrm{bg}}$ from the field mean; terminal translation $\hat{\mathbf t}_T=\hat{\mathbf t}+(\hat U_{\mathrm{bg}}T,\hat U_{\mathrm{bg}}T)$
& $\mathbf t$, $U_{\mathrm{bg}}$ construct consistency perturbations and alignment losses \\
2-D SWE
& Input translation $\hat{\Delta\mathbf x}$
& Background velocity $\hat{\mathbf V}_{\mathrm{bg}}$ from the spatial mean of the stored primitive velocity channels; terminal translation $\hat{\Delta\mathbf x}_T=\hat{\Delta\mathbf x}+\hat{\mathbf V}_{\mathrm{bg}}T$
& $\Delta\mathbf x$, $\mathrm{Fr}$, and direction angle construct velocity-boost consistency perturbations \\
(2+1)-D NS
& Input rotation--translation parameters $(\hat\theta,\hat t_x,\hat t_y)$
& The same numerical spatial transformation is restored, so $\hat{\xi}_T=\hat{\xi}_0$
& $(\theta,t_x,t_y)$ rotate/translate training trajectories and supervise alignment losses \\
\bottomrule
\end{tabular}
\end{table}

\paragraph{Detailed losses.}
The total training objective is
\begin{equation*}
    \mathcal{L}_{\mathrm{total}}(\theta,\phi)
    =
    \mathcal{L}_{\mathrm{PDE}}
    +
    \lambda_1 \mathcal{L}_{\mathrm{equiv}}
    +
    \lambda_2 \mathcal{L}_{\mathrm{anchor}}.
\end{equation*}
The implementation separates the supervised dynamics path from the geometric-consistency path. The physical term is evaluated on the clean training pair $(x,y)$ after one-shot canonicalization and frame restoration:
\begin{equation*}
\mathcal{L}_{\mathrm{PDE}}
=
\mathbb{E}
\left[
\frac{\|\hat{u}_{\mathrm{pred}}(x)-y\|_{L^2}}
{\|y\|_{L^2}}
+
\gamma
\frac{
\||\mathbf{K}|\odot(\mathcal{F}(\hat{u}_{\mathrm{pred}}(x))-\mathcal{F}(y))\|_{L^2}
}{
\||\mathbf{K}|\odot \mathcal{F}(y)\|_{L^2}
}
\right].
\end{equation*}
Independently, a bounded generator $\Delta_0$ produces a perturbed input $x^+$. The equivariance regularizer matches the relative estimator response to that known perturbation:
\begin{equation*}
    \mathcal{L}_{\mathrm{equiv}}
    =
    \mathbb{E}
    \big[
    \|\mathbf{M}(\hat{\xi}_0^+\ominus\hat{\xi}_0)-\mathbf{M}(\Delta_0)\|_F^2
    \big],
\end{equation*}
where $\mathbf M$ is the fixed coordinate embedding that scales each Lie-algebra component to its training support and returns a matrix representation for matrix-valued generators. For commuting periodic translations, $\ominus$ is componentwise wrapped subtraction on $\mathbb T^d$. For a non-commutative component such as the NS rotation--translation action, it means the local group difference
\[
\hat\xi_0^+\ominus_G\hat\xi_0:=\operatorname{Log}\!\left(\exp(\hat\xi_0^+)\exp(-\hat\xi_0)\right),
\]
whose Baker--Campbell--Hausdorff terms retain the rotation--translation ordering; we use $\ominus=\ominus_G$ for that block. The clean data are sampled near the canonical slice, so the anchor loss penalizes the spatial estimator output on the clean input:
\begin{equation*}
    \mathcal{L}_{\mathrm{anchor}}
    =
    \mathbb{E}
    \big[
    \|P_{\mathcal S}\hat{\xi}_0\|_{\mathfrak g}^2
    \big].
\end{equation*}

Here $\mathbf K$ is the grid wavenumber array: for a $d$-dimensional periodic grid with side lengths $L_1,\ldots,L_d$, $\mathbf K(\mathbf n)=(2\pi n_1/L_1,\ldots,2\pi n_d/L_d)$ and $|\mathbf K|=(\sum_jK_j^2)^{1/2}$; the zero mode is therefore unweighted. The task-specific channel weights, invariant-height term for SWE, and coefficients are listed in Table~\ref{tab:pace_repro_hparams}. The anchor coefficient is fixed for the Burgers and SWE runs; only the Navier--Stokes configuration uses the epoch-dependent decay shown in that table.

\section{Additional Implementation Details}
\label{app:implementation}
\label{app:details}
\label{app:benchmarks}

This section reports the problem construction and additional results. We train on canonical low-resolution data and evaluate on geometrically shifted OOD data at inference. The equations follow standard operator-learning benchmarks \citep{li2020fourier,kovachki2023neural,takamoto2022pdebench}; the symmetry-induced OOD splits are specific to this study.

\subsection{Implementation and reproducibility details}
\label{app:reproducibility}

All reported numbers use three random seeds, $\{42,1234,3407\}$. The Python, NumPy, and PyTorch random number generators are re-seeded before each data loader is constructed. Tables report means and standard deviations across seeds. Runtime entries are wall-clock time per training epoch and per full OOD evaluation pass, including TTA updates when enabled. All experiments ran on one NVIDIA GeForce RTX 4080 GPU with driver 580.105.08 and CUDA 13.0. The listed epochs, per-epoch time, and three seeds determine the total training compute. FNO+Aug uses the same FNO architecture and optimizer settings as FNO; only the training and evaluation data construction changes.

\paragraph{Runtime convention and limitation.}
For the 1-D runtime entries, the OOD tensors and data loader are constructed before the timed evaluation loop. The timed interval includes synchronized device transfers, model forward passes, metric computation, and prediction caching, but not OOD data construction or checkpoint loading. No historical runtime logs are available. We therefore treat the 1-D runtime values as descriptive full-loop measurements and do not interpret the atypical FNO+Aug value as evidence that augmentation intrinsically increases inference latency.

Tables~\ref{tab:repro_core_hparams} and \ref{tab:pace_repro_hparams} record the shared training configuration and the PACE-FNO-specific losses and TTA settings.
\begin{table}[htbp]
\centering
\caption{Core training hyperparameters used by the experimental scripts. ``Batch'' denotes train/OOD-evaluation batch size. GFNO uses the square-grid $p4$ symmetry group (integer translations and rotations by multiples of $90^\circ$), which does not cover the continuous sub-pixel translations or Galilean boosts in these OOD tests; for GFNO on NS we use time modes $10$ in addition to spatial modes $10$. For SWE, the FNO-model weight decay is $10^{-4}$ for FNO+Aug/PACE-FNO and $10^{-3}$ for FNO without augmentation and GFNO, matching the experiment configuration. This unequal decay is a protocol limitation and means those SWE rows are not a perfectly regularization-matched architectural comparison.}
\label{tab:repro_core_hparams}
\small
\resizebox{\textwidth}{!}{%
\begin{tabular}{llcccccc}
\toprule
Test case & Models & Train/Test & ID Grid & OOD Grid & Batch & Epochs & Optimizer, Scheduler, Architecture \\
\midrule
1-D Burgers
& FNO, FNO+Aug, GFNO, PACE-FNO
& 1000/200
& $1024$
& $2048$
& $20/1$
& 500
& Adam, lr $10^{-3}$, wd $10^{-4}$; StepLR $(50,0.5)$; modes $20$, width $64$ \\
2-D Burgers
& FNO, FNO+Aug, GFNO, PACE-FNO
& 1000/200
& $64^2$
& $128^2$
& $20/10$
& 500
& Adam, lr $10^{-3}$, wd $10^{-4}$; StepLR $(100,0.5)$; modes $12$, width $32$ \\
2-D SWE
& FNO, FNO+Aug, GFNO, PACE-FNO
& 1000/200
& $64^2$
& $128^2$
& $20/10$
& 500
& Adam, lr $10^{-3}$; StepLR $(100,0.2)$; modes $16$, width $32$ \\
(2+1)-D NS
& FNO, FNO+Aug, GFNO, PACE-FNO
& 1000/200
& $64^2$
& $128^2$
& $10/10$
& 100
& Adam, lr $2.5{\times}10^{-3}$, wd $10^{-4}$; StepLR $(100,0.5)$; spatial modes $10$, width $20$ \\
\bottomrule
\end{tabular}
}
\end{table}

\begin{table}[htbp]
\centering
\caption{PACE-FNO specific losses and TTA settings. One-shot PACE-FNO uses the same trained checkpoint and sets the number of TTA steps to zero. The Lie-algebra coordinate estimator parameter group uses the same learning rate as the FNO predictor and no additional weight decay unless explicitly listed.}
\label{tab:pace_repro_hparams}
\scriptsize
\setlength{\tabcolsep}{4pt}
\renewcommand{\arraystretch}{1.08}
\resizebox{\textwidth}{!}{%
\begin{tabular}{llll}
\toprule
Test case & PACE-FNO Training Loss & TTA Variables and Optimizer & TTA Steps / LR \\
\midrule
1-D Burgers
& Split $L^2$/$H^1$ loss with EMA $H^1$ weight ($\alpha=0.9$) plus $10\,\mathcal L_{\mathrm{equiv}}$
& Translation phase, Adam
& 20 / 0.05 \\
2-D Burgers
& $\mathcal L_{2}+5\mathcal L_{H^1}+1\mathcal L_{\mathrm{equiv}}+5\mathcal L_{\mathrm{anchor}}$
& $(t_x,t_y)$, Adam
& 50 / 0.01 \\
2-D SWE
& Channel-weighted PDE loss plus scheduled $H^1$, $10\mathcal L_{\mathrm{equiv}}$, $5\mathcal L_{\mathrm{anchor}}$, and $0.5$ invariant-height loss
& $(t_x,t_y)$, AdamW with wd $10^{-4}$
& 80 / 0.05 \\
(2+1)-D NS
& $\mathcal L_{2}+5\mathcal L_{H^1}+1\mathcal L_{\mathrm{equiv}}+\lambda_{\mathrm{anc}}(e)\mathcal L_{\mathrm{anchor}}$, $\lambda_{\mathrm{anc}}(e)=20(0.85)^e+0.1$
& $(\theta,t_x,t_y)$, Adam
& 5 / 0.05 \\
\bottomrule
\end{tabular}
}
\end{table}

For the 1-D Burgers split, the exponential moving average in the loss label ``EMA $H^1$ weight'' tracks the normalized batch ratio
\[
r_e=\frac{\mathcal L_{H^1,e}}{\mathcal L_{2,e}+\mathcal L_{H^1,e}+10^{-8}},
\qquad
\bar r_e=0.9\,\bar r_{e-1}+0.1\,r_e,
\qquad
\lambda_{H^1,e}=\operatorname{clip}(\bar r_e,0.05,0.95),
\]
with $\bar r_0=r_0$; the split loss is $(1-\lambda_{H^1,e})\mathcal L_2+\lambda_{H^1,e}\mathcal L_{H^1}$. Thus the EMA smooths the relative contribution of the two normalized physical terms, not model weights or predictions.

\subsection{1-D Burgers equation}
\label{app:burgers_1d}

The benchmark solves
\[
\partial_tu+\tfrac12\partial_x(u^2)=\nu\partial_{xx}u,
\qquad
u(x,0)=u_0(x),
\qquad x\in\mathbb T,
\]
with periodic boundary conditions. Canonical ID initial fields are sampled from a Gaussian random field, while OOD samples are generated by
\[
u_{0,\mathrm{OOD}}(x)=u_0(x-\Delta x)+c,
\]
which combines spatial translation $\Delta x$ and a Galilean velocity offset $c$. This follows the restricted Lie-group augmentation framework \citep{brandstetter2022lie}.

\paragraph{Data augmentation.} We sample measure-preserving perturbations from the Galilean group $G = \mathbb{T}^1 \rtimes \mathrm{Gal}$. For a canonical initial field $u_0(x)$, the augmented state at $t=0$ is generated by the pushforward action of $g_0 = \exp(\xi_0)$:
\begin{equation*}
    u_{\mathrm{aug}}(x, 0) = u_0(x - \delta_x) + v_{\mathrm{bg}}
\end{equation*}
where the initial Lie algebra generator $\xi_0 \in \mathfrak{g}$ is defined by the superposition of the spatial translation and velocity boost differential operators:
\begin{equation*}
    \xi_0 = -\delta_x \frac{\partial}{\partial x} + v_{\mathrm{bg}} \mathcal{I}_u
\end{equation*}
Here, $\delta_x \sim \mathcal{U}(-0.1, 0.1)$ parameterizes the spatial translation generator within the geometric subalgebra $\mathfrak{g}_{\mathcal{S}}$, $\mathcal{I}_u$ denotes the identity shift operator on the codomain, and $v_{\mathrm{bg}} \sim \mathcal{U}(-0.2, 0.2)$ represents the Galilean kinematic velocity boost within $\mathfrak{g}_{\mathcal{K}}$.

To maintain Galilean invariance during physical evolution, the target state at time $T$ is shifted by the induced kinematic convection. The target group element evolves analytically as $g_T = \exp(\xi_0 + \Delta x \cdot \frac{\partial}{\partial x})$, where the spatial translation increment is $\Delta x = v_{\mathrm{bg}} T$.

\paragraph{Self-consistency TTA.} PACE-FNO keeps all learned weights fixed and adapts only the coordinate, unlike parameter-updating test-time training (TTT) and test-time entropy minimization (TENT) \citep{sun2020test,wang2021tent}. After canonicalization, the trained estimator $\mathcal{E}_\phi$ should produce a near-zero residual. We define the residual phase shift as $\epsilon_{\mathrm{res}} = \mathcal{P}_{\mathcal{S}} \circ \mathcal{E}_\phi \big( \mathcal{C}_{\exp(\hat{\xi}_{\mathcal{S}})}^{-1}(a_{\mathrm{OOD}}) \big)$, where $\mathcal{P}_{\mathcal{S}}$ projects onto the spatial translation subspace. The TTA loss is the batch mean squared residual:
\begin{equation*}
    \mathcal{J}_{\mathrm{TTA}}(\hat{\xi}_{\mathcal{S}}) = \frac{1}{B} \sum_{i=1}^B \left( \epsilon_{\mathrm{res}}^{(i)} \right)^2 = \frac{1}{B} \sum_{i=1}^B \left( \mathcal{P}_{\mathcal{S}} \circ \mathcal{E}_\phi \Big( \mathcal{C}_{\exp(\hat{\xi}_{\mathcal{S}}^{(i)})}^{-1}(a_{\mathrm{OOD}}^{(i)}) \Big) \right)^2
\end{equation*}
where $B$ denotes the batch size.

Figure~\ref{fig:burgers_1d_id_canonical} shows the pullback on a representative translated and boosted 1-D Burgers state.
\begin{figure}[htbp]
        \centering
        \includegraphics[width=0.74\textwidth]{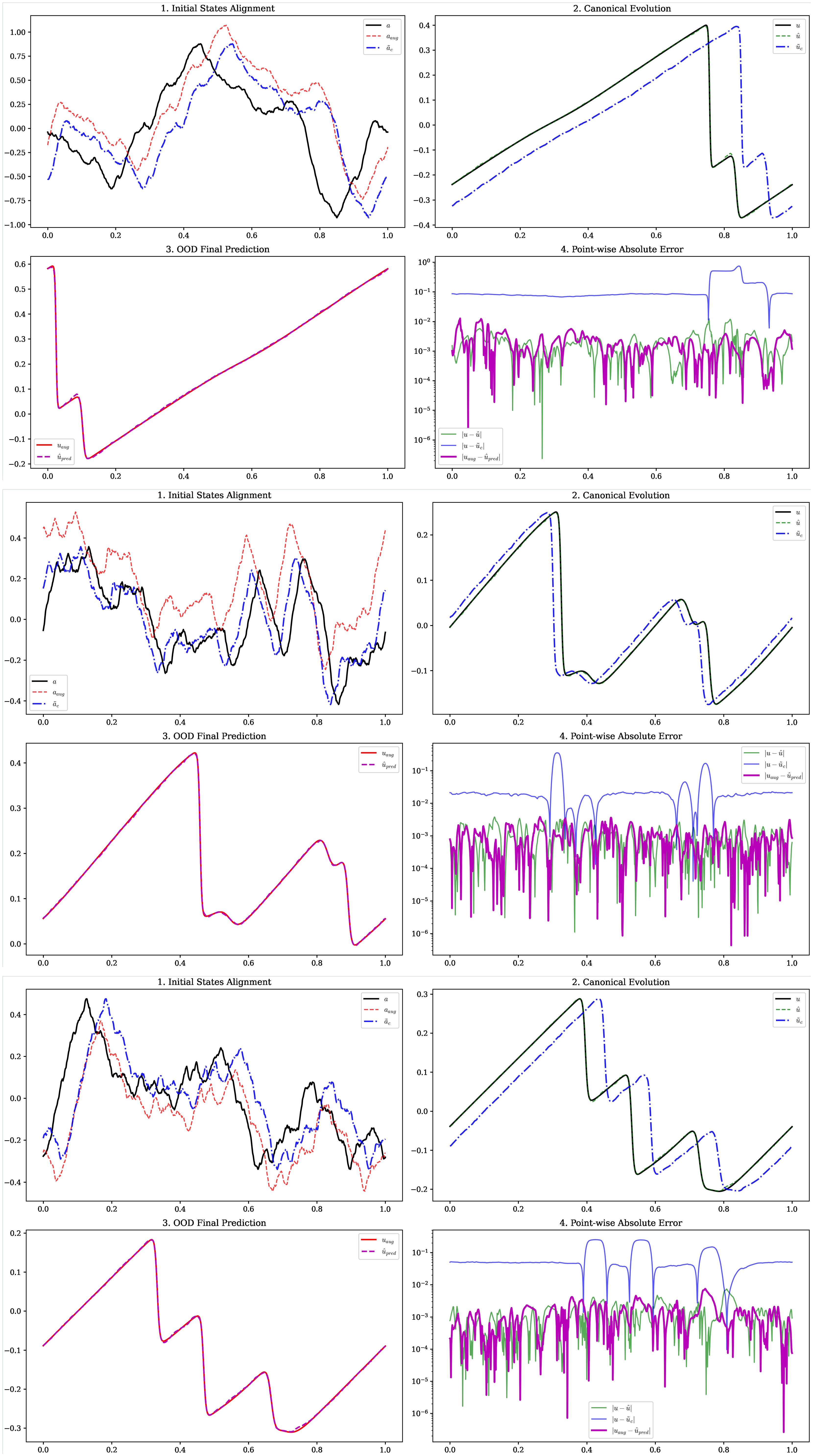}
        \caption{Canonicalization on 1-D Burgers. The physical panels show the translated and boosted input being pulled back toward the canonical frame; the latent panels show that the estimator residual is small after alignment.}
        \label{fig:burgers_1d_id_canonical}
    \end{figure}

Figure~\ref{fig:burgers_1d_ood_comparison} compares the resulting OOD predictions under the same shifted input.

\begin{figure}[htbp]
        \centering
        \begin{subfigure}[t]{0.74\textwidth}
            \centering
            \includegraphics[width=\linewidth]{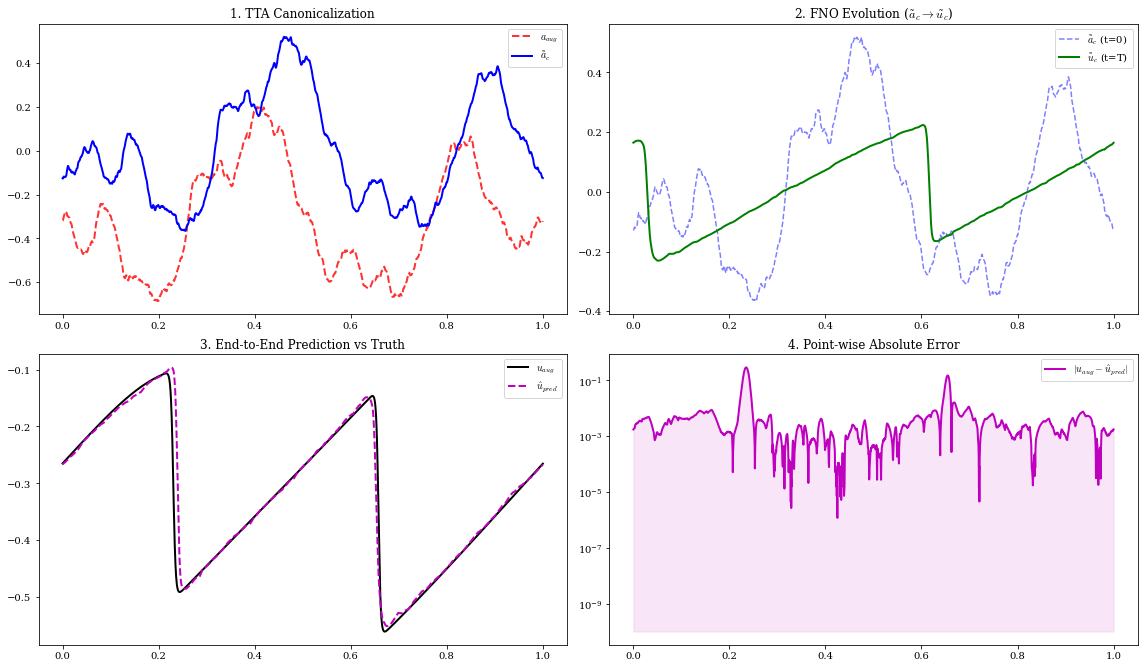}
            \subcaption{PACE-FNO.}
        \end{subfigure}

        \vspace{0.35em}

        \begin{subfigure}[t]{0.74\textwidth}
            \centering
            \includegraphics[width=\linewidth]{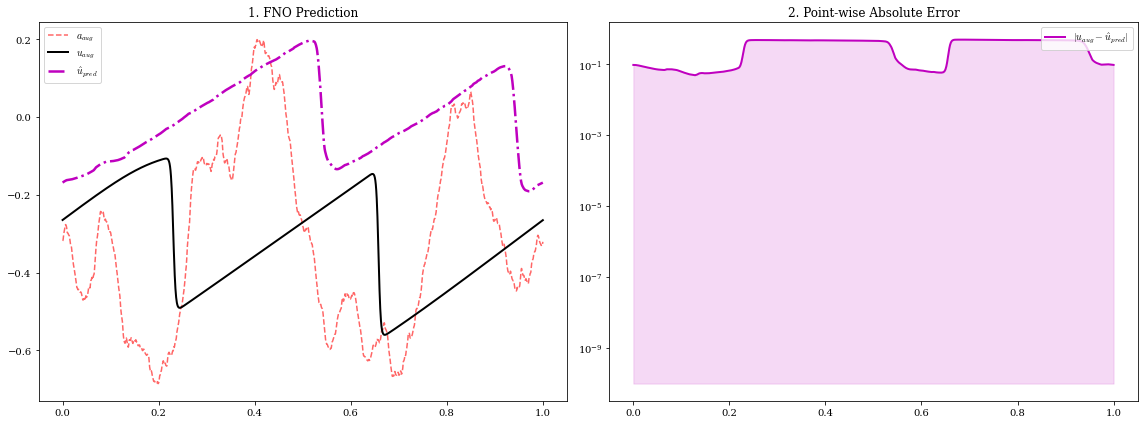}
            \subcaption{FNO.}
        \end{subfigure}

        \vspace{0.35em}

        \begin{subfigure}[t]{0.74\textwidth}
            \centering
            \includegraphics[width=\linewidth]{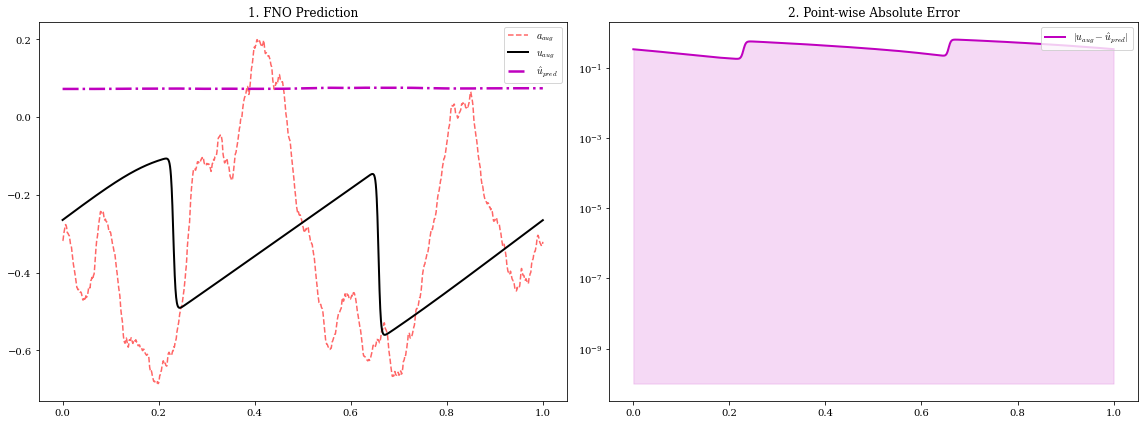}
            \subcaption{FNO+Aug.}
        \end{subfigure}

        \vspace{0.35em}

        \begin{subfigure}[t]{0.74\textwidth}
            \centering
            \includegraphics[width=\linewidth]{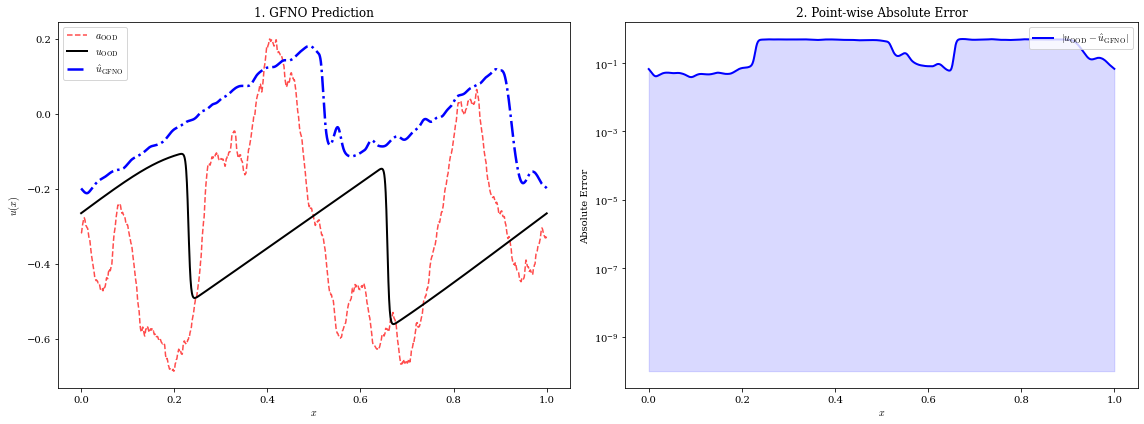}
            \subcaption{GFNO.}
        \end{subfigure}
        \caption{OOD prediction on 1-D Burgers under translation and Galilean drift. PACE-FNO aligns the wave location and amplitude with the target, while the baselines show frame mismatch or phase error under the same shifted input.}
        \label{fig:burgers_1d_ood_comparison}
    \end{figure}

\subsection{2-D Burgers equation}
\label{app:burgers_2d}

The two-dimensional benchmark solves
\[
\partial_tu+\tfrac12\partial_x(u^2)+\tfrac12\partial_y(u^2)
=\nu(\partial_{xx}u+\partial_{yy}u),
\qquad
(x,y)\in\mathbb T^2.
\]
We use the scalar 2-D Burgers variant rather than the coupled vector-valued system to isolate the effect of geometric canonicalization from vector-valued coupling.
Canonical states are centered and aligned, while OOD samples follow
\[
u_{0,\mathrm{OOD}}(x,y)=u_0(x-t_x,y-t_y)+U_{\mathrm{bg}},
\]
with translations $(t_x,t_y)$ and background flow $U_{\mathrm{bg}}$.

\paragraph{Data augmentation.}
For the 2-D setting, we sample measure-preserving perturbations from the Galilean group $G = \mathbb{T}^2 \rtimes \mathrm{Gal}$. The corresponding initial Lie algebra generator $\xi_0 \in \mathfrak{g}$ is defined by the superposition of the two-dimensional spatial translation and velocity boost differential operators:
\begin{equation*}
    \xi_0 = -t_x \frac{\partial}{\partial x} - t_y \frac{\partial}{\partial y} + U_{\mathrm{bg}} \mathcal{I}_u
\end{equation*}
Here, the translation vector $\mathbf{t} = (t_x, t_y) \sim \mathcal{U}(-0.1, 0.1)^2$ parameterizes the spatial generators within the geometric subalgebra $\mathfrak{g}_{\mathcal{S}}$, $\mathcal{I}_u$ denotes the identity shift operator on the codomain, and $U_{\mathrm{bg}} \sim \mathcal{U}(-0.5, 0.5)$ represents the Galilean velocity boost within $\mathfrak{g}_{\mathcal{K}}$. These shifts violate the mass and amplitude constraints of the in-distribution dataset.

To preserve Galilean invariance during spatiotemporal evolution, the target state at time $T$ undergoes kinematic convection. The target group element evolves analytically as $g_T = \exp\big(\xi_0 + \Delta \mathbf{x} \cdot \nabla \big)$, where the translation increment vector is driven by the background flow such that $\Delta \mathbf{x} = (U_{\mathrm{bg}} T, U_{\mathrm{bg}} T)$.

\paragraph{Self-consistency TTA.} The estimator applied to the canonicalized state should return near-zero residual shifts when alignment succeeds. We define the residual two-dimensional phase shift vector as $\boldsymbol{\epsilon}_{\mathrm{res}} = (\epsilon_{\mathrm{res}, x}, \epsilon_{\mathrm{res}, y})$. The adaptation objective during evaluation is the empirical mean squared error of these two-dimensional phase residuals over a given batch:
\begin{equation*}
    \mathcal{J}_{\mathrm{TTA}}(\hat{\xi}_{\mathcal{S}}) = \frac{1}{B} \sum_{i=1}^B \left( (\epsilon_{\mathrm{res}, x}^{(i)})^2 + (\epsilon_{\mathrm{res}, y}^{(i)})^2 \right)
\end{equation*}
where $B$ denotes the batch size.

Figure~\ref{fig:burgers_2d_id_canonical} shows the learned pullback before the 2-D Burgers rollout.
\begin{figure}[htbp]
    \centering
    \includegraphics[width=0.98\textwidth]{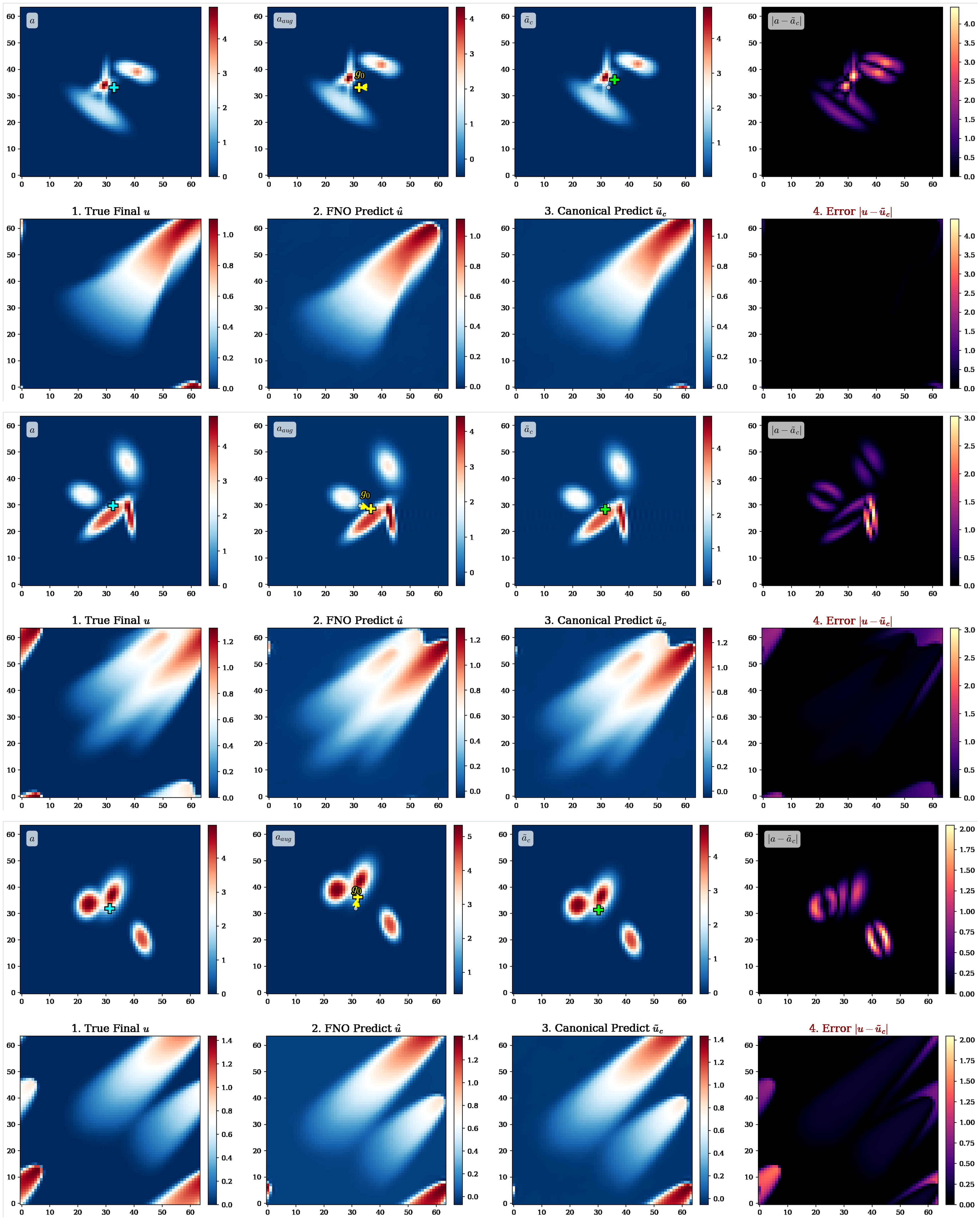}
    \caption{Canonicalization on 2-D Burgers. The learned pullback removes the dominant spatial shift and background-flow component before the FNO rollout, leaving a canonical field whose latent residual is close to the ID frame.}
    \label{fig:burgers_2d_id_canonical}
\end{figure}

Figure~\ref{fig:burgers_2d_ood_comparison} gives the corresponding baseline OOD predictions.
\begin{figure}[!ht]
    \centering
    \begin{subfigure}[t]{0.94\textwidth}
        \centering
        \includegraphics[width=\linewidth]{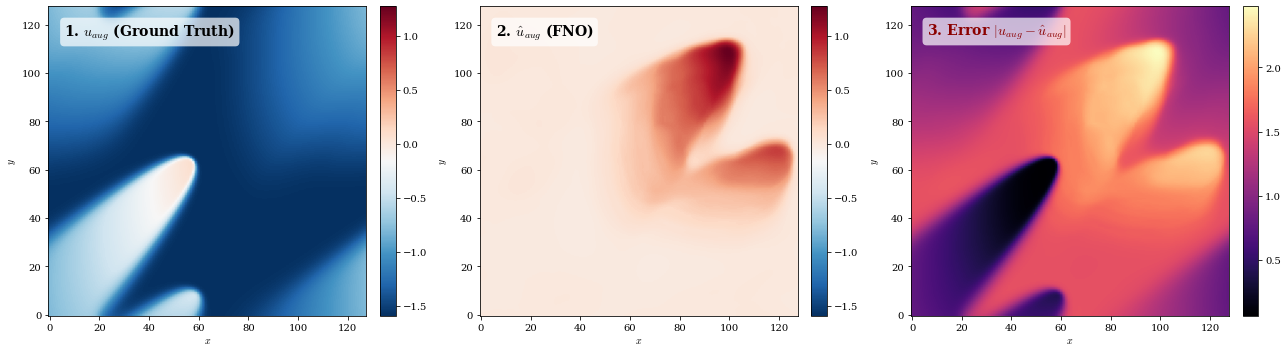}
        \subcaption{FNO.}
    \end{subfigure}

    \vspace{0.45em}

    \begin{subfigure}[t]{0.94\textwidth}
        \centering
        \includegraphics[width=\linewidth]{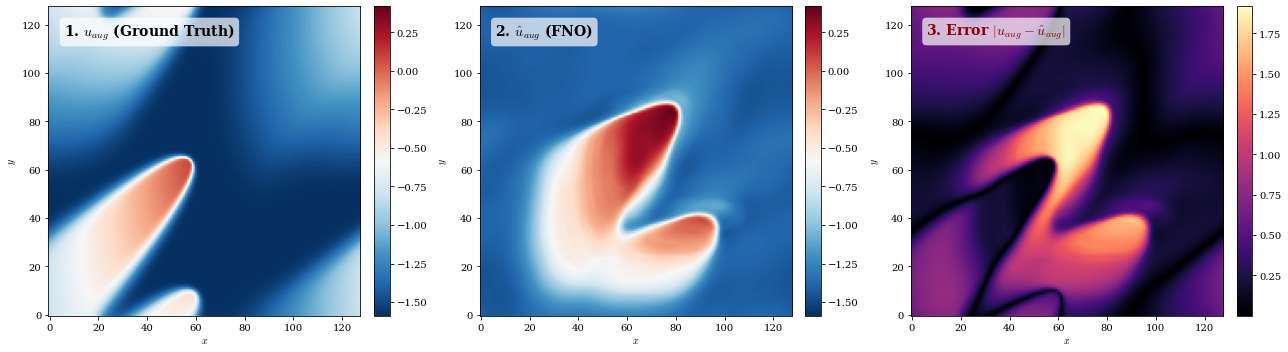}
        \subcaption{FNO+Aug.}
    \end{subfigure}

    \vspace{0.45em}

    \begin{subfigure}[t]{0.94\textwidth}
        \centering
        \includegraphics[width=\linewidth]{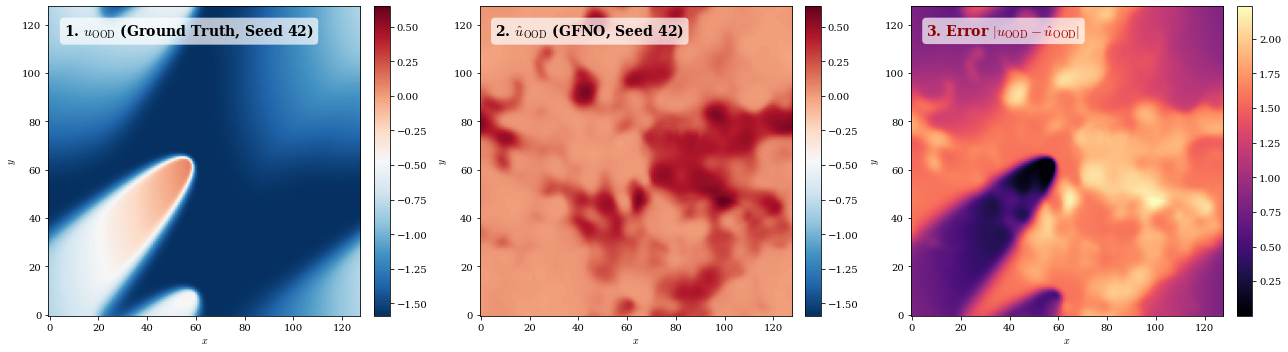}
        \subcaption{GFNO.}
    \end{subfigure}
    \caption{Baseline OOD predictions on 2-D Burgers under translation and Galilean drift. The corresponding PACE-FNO prediction is shown in the main paper; these additional panels expose the frame errors retained by FNO, FNO+Aug, and GFNO under the same shifted input.}
    \label{fig:burgers_2d_ood_comparison}
\end{figure}

\begin{figure}[htbp]
    \centering
    \includegraphics[width=0.74\textwidth]{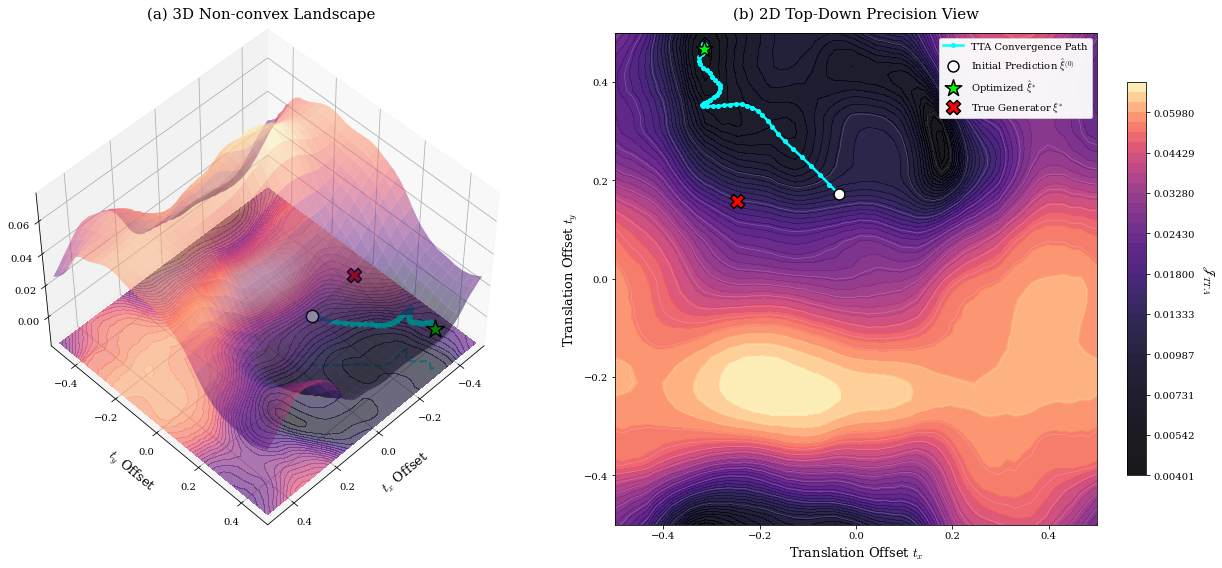}
    \caption{The 2-D Burgers TTA objective has symmetry-induced low-energy basins. TTA changes only the estimated coordinate; it does not update the FNO or estimator weights.}
    \label{fig:tta-landscape}
\end{figure}
\FloatBarrier

\subsection{2-D shallow-water equations}
\label{app:swe_2d}

The numerical solver evolves the conservative variables $\mathbf U=(h,\mathbf m)$ according to
\[
\partial_th+\nabla\!\cdot\!\mathbf m=\nu\nabla^2h,
\qquad
\partial_t\mathbf m+
\nabla\!\cdot\!\left(\frac{\mathbf m\otimes\mathbf m}{h}\right)
+\nabla\!\left(\tfrac12gh^2\right)
=\nu\nabla^2\mathbf m
\]
on $\mathbb T^2$. The released operator-learning arrays, however, store the primitive representation $(h,\mathbf v)$ with $\mathbf v=\mathbf m/h$; the solver converts back to velocity before saving every frame. Canonical conditions enforce zero macroscopic background velocity and centered mass distribution. OOD states are generated by
\[
\mathbf{U}_{0,\mathrm{OOD}}(\mathbf{x})
=
\mathcal{B}_{\mathbf{V}_{\mathrm{bg}}}\mathbf{U}_0(\mathbf{x}-\Delta\mathbf{x}),
\]
which combines sub-pixel translation and a Galilean background flow. In the stored primitive representation, the boost is $\mathcal B_{\mathbf V}[h,\mathbf v]^\top=[h,\mathbf v+\mathbf V]^\top$; equivalently, in the solver's conservative variables, $\mathcal B_{\mathbf V}[h,\mathbf m]^\top=[h,\mathbf m+h\mathbf V]^\top$.

\paragraph{Data augmentation.}
For shallow-water dynamics, we use the same Galilean group $G = \mathbb{T}^2 \rtimes \mathrm{Gal}$ to perturb the initial vector field. The pushforward action of $g_0 = \exp(\xi_0)$ is the boost-and-translation map above, with initial Lie algebra generator
\begin{equation*}
    \xi_0 = -t_x \frac{\partial}{\partial x} - t_y \frac{\partial}{\partial y} + \mathbf{V}_{\mathrm{bg}} \mathcal{I}_{\mathbf{v}}
\end{equation*}
Here, the translation vector $\Delta \mathbf{x} = (t_x, t_y) \sim \mathcal{U}(-0.15, 0.15)^2$ parameterizes the spatial generators within the geometric subalgebra $\mathfrak{g}_{\mathcal{S}}$. To reduce interpolation-induced dissipation of high-frequency surface ripples, these continuous sub-pixel translations are executed analytically in the Fourier domain. The boost direction $\mathcal I_{\mathbf v}$ acts additively on the stored velocity channels. For each clean sample, let $A$ denote its inferred surface-height amplitude and define the reference wave speed
\[
    c_{\mathrm{ref}}=\sqrt{g(1+A/2)}.
\]
Training perturbations sample $\mathrm{Fr}\sim\mathcal U(0,0.2)$ and $\alpha\sim\mathcal U(0,2\pi)$, then set
\[
    \mathbf V_{\mathrm{bg}}
    =
    c_{\mathrm{ref}}\,\mathrm{Fr}\,(\cos\alpha,\sin\alpha).
\]
Thus the implemented training support is radial in Froude number rather than a component-wise uniform box in $(u_{\mathrm{bg}},v_{\mathrm{bg}})$. The independently solved OOD set uses $\mathrm{Fr}\sim\mathcal U(0.3,0.8)$ and samples the wave center uniformly over the periodic grid, giving $t_x,t_y\in[-\tfrac12+\tfrac{1}{2s},\tfrac12-\tfrac{1}{2s}]$ at the generation resolution $s=256$. These shifts lie outside the training support and violate the mass-centering and amplitude constraints of the in-distribution dataset.

To maintain Galilean invariance during spatiotemporal evolution, the target state at time $T$ undergoes kinematic convection. The target group element evolves analytically as $g_T = \exp\big(\xi_0 + \Delta \mathbf{x}_T \cdot \nabla \big)$, where the translation increment vector is dynamically driven by the background flow such that $\Delta \mathbf{x}_T = \mathbf{V}_{\mathrm{bg}} T$.

\paragraph{Self-consistency TTA.} After canonicalization, the estimator should produce near-zero residual shifts. We define the residual two-dimensional phase shift vector as $\boldsymbol{\epsilon}_{\mathrm{res}} = (\epsilon_{\mathrm{res}, x}, \epsilon_{\mathrm{res}, y})$. The TTA objective is the empirical mean squared error over a batch:
\begin{equation*}
    \mathcal{J}_{\mathrm{TTA}}(\hat{\xi}_{\mathcal{S}}) = \frac{1}{B} \sum_{i=1}^B \left( (\epsilon_{\mathrm{res}, x}^{(i)})^2 + (\epsilon_{\mathrm{res}, y}^{(i)})^2 \right)
\end{equation*}
where $B$ denotes the batch size.

Figure~\ref{fig:pace_fno_2d_swe_multi_canonicalization} shows the aligned shallow-water states before prediction.
\begin{figure*}[htbp]
    \centering
    \includegraphics[width=0.98\textwidth]{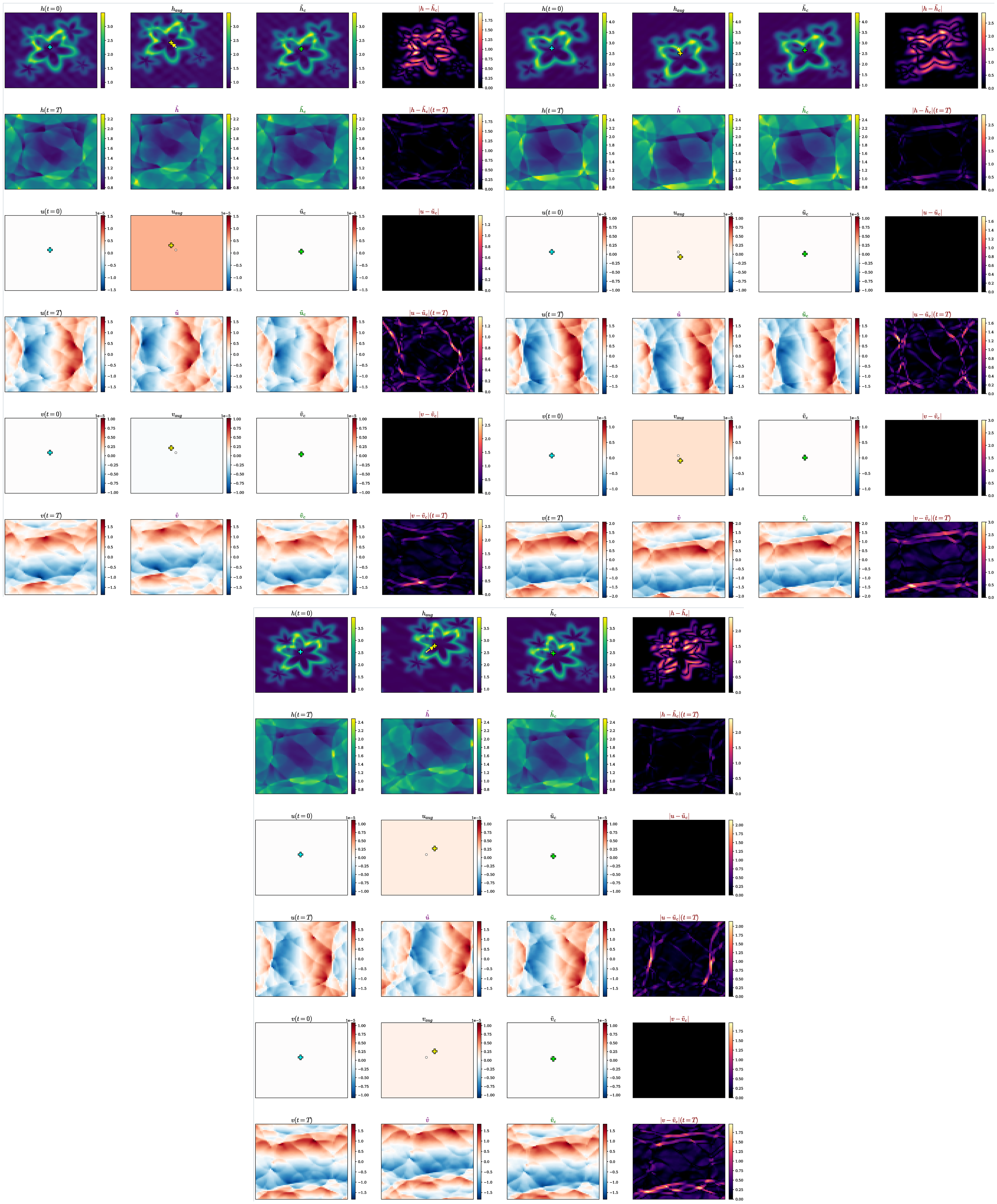}

    \caption{
    Canonicalization behavior of PACE-FNO under Galilean perturbations for the 2-D shallow-water system. The aligned states remove the dominant translation and background transport before prediction, which explains the smaller OOD error in Table~\ref{tab:swe_2d_main}.
    }
    \label{fig:pace_fno_2d_swe_multi_canonicalization}
\end{figure*}

Figure~\ref{fig:swe_ood_comparison} shows the baseline predictions for the same shallow-water shift.
\begin{figure}[htbp]
    \centering
    \begin{subfigure}[t]{0.58\linewidth}
        \centering
        \includegraphics[width=\linewidth]{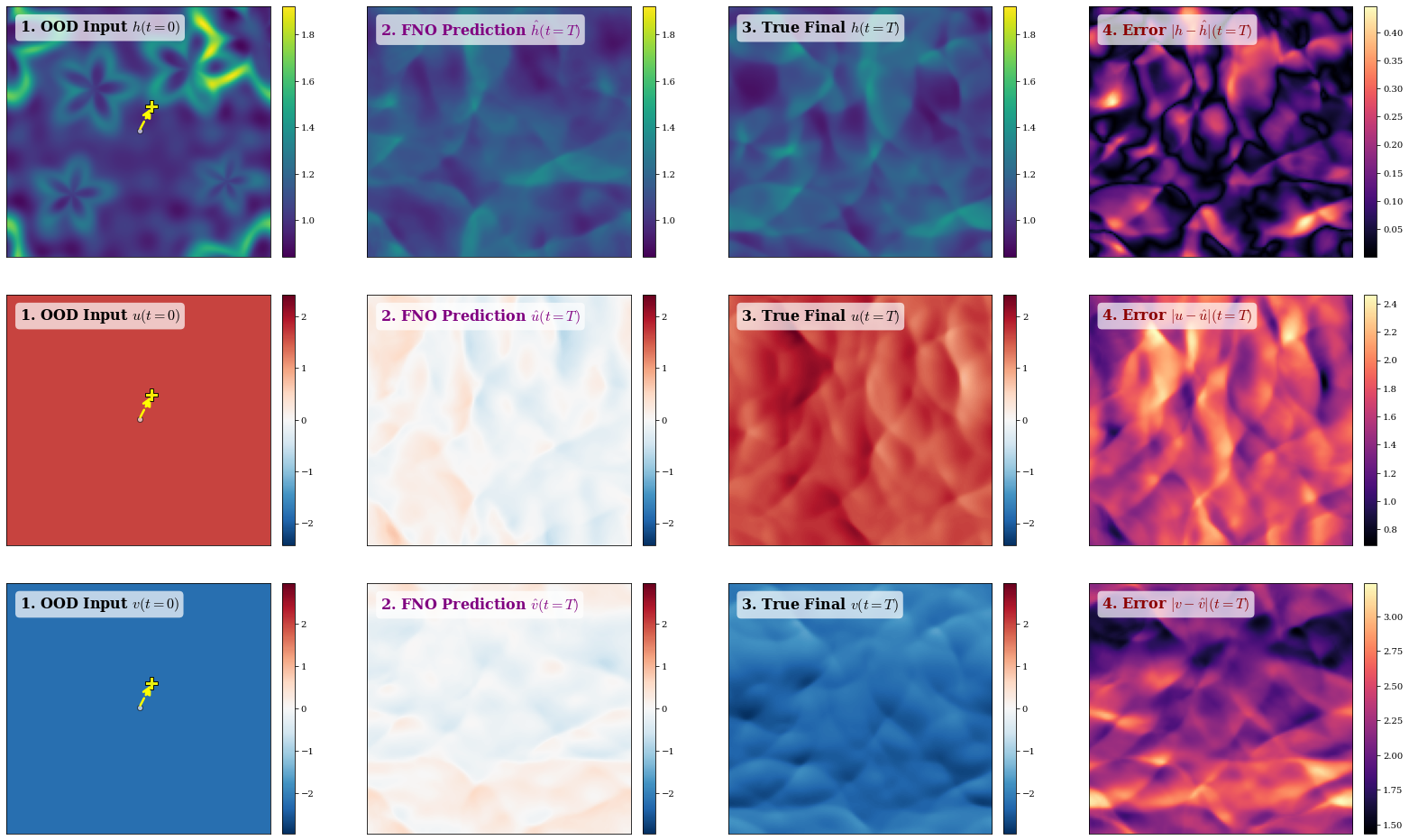}
        \subcaption{FNO.}
    \end{subfigure}

    \vspace{0.35em}

    \begin{subfigure}[t]{0.58\linewidth}
        \centering
        \includegraphics[width=\linewidth]{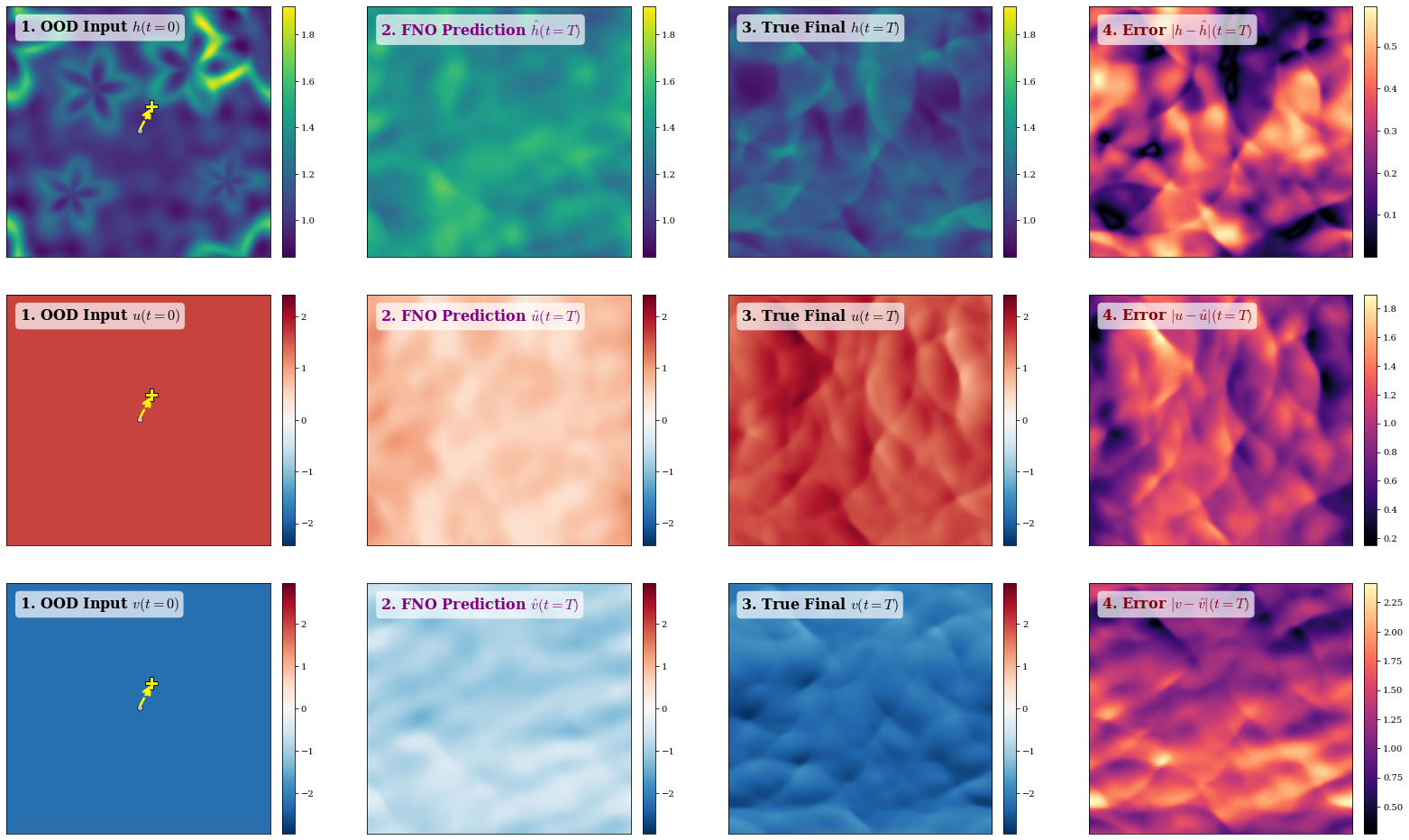}
        \subcaption{FNO+Aug.}
    \end{subfigure}

    \vspace{0.35em}

    \begin{subfigure}[t]{0.58\linewidth}
        \centering
        \includegraphics[width=\linewidth]{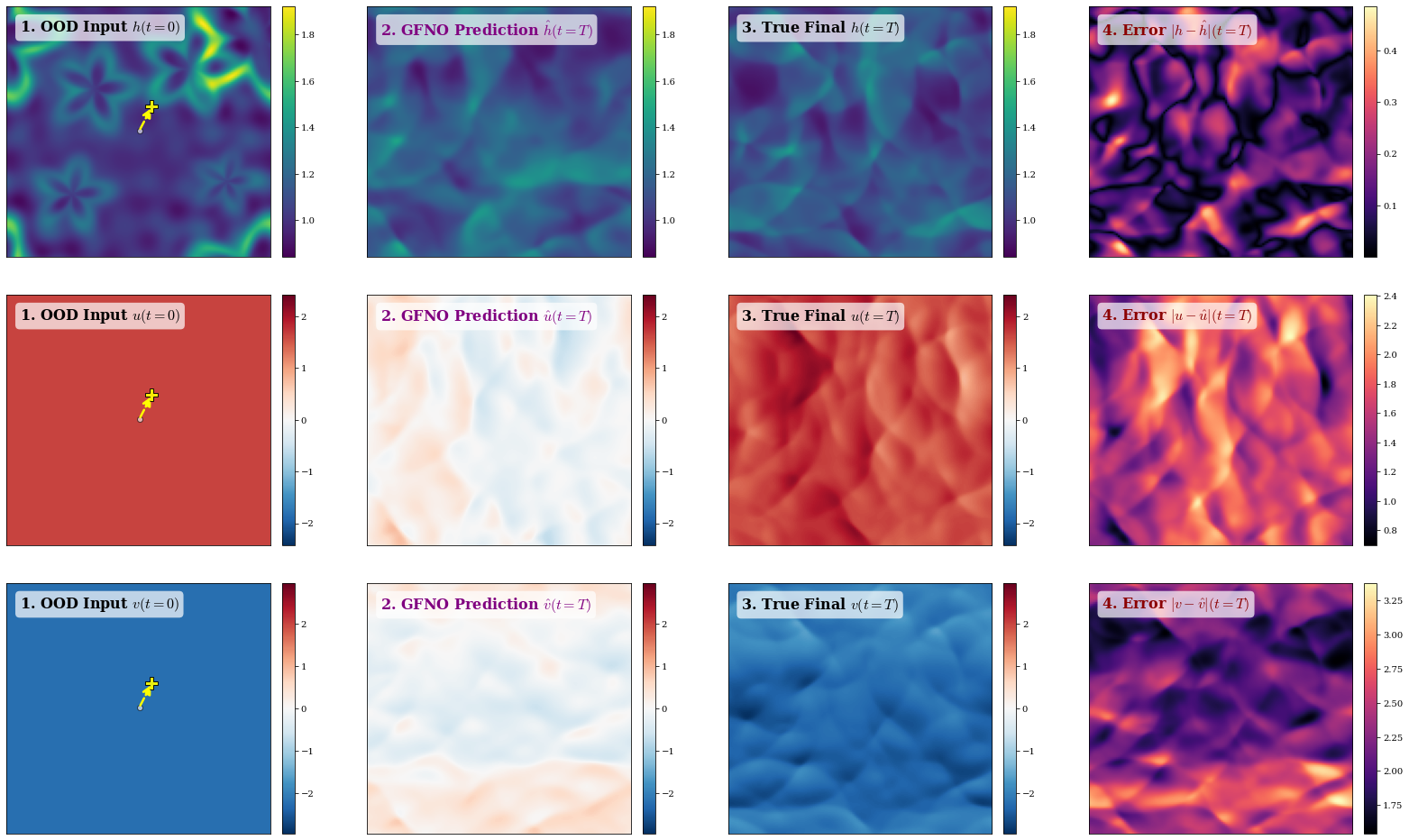}
        \subcaption{GFNO.}
    \end{subfigure}
    \caption{Baseline OOD predictions on the 2-D shallow-water system under Galilean perturbation. The corresponding PACE-FNO prediction is shown in the main paper; the additional baseline panels show the larger displacement and amplitude errors.}
    \label{fig:swe_ood_comparison}
\end{figure}

\begin{figure}[htbp]
    \centering
    \includegraphics[width=0.74\textwidth]{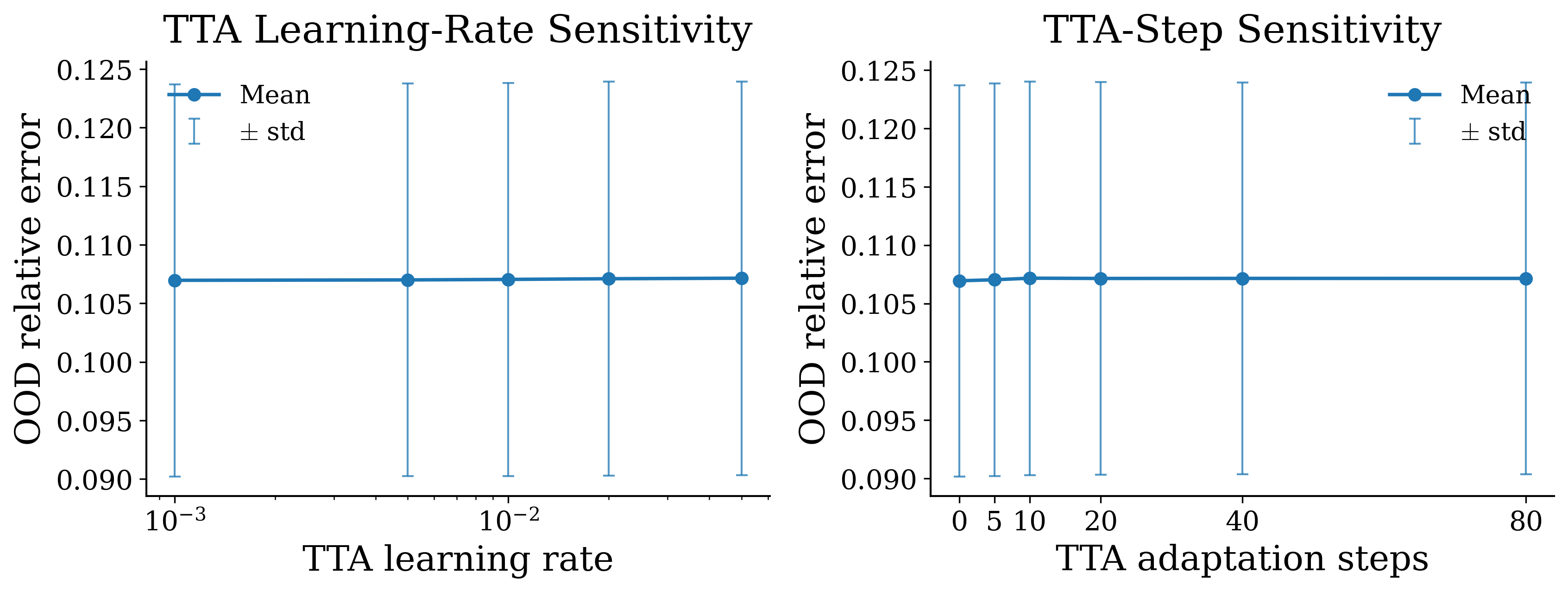}
    \caption{SWE TTA sensitivity to the tested learning rates and adaptation-step budgets. The final OOD relative error changes little over the reported ranges.}
    \label{fig:tta-sensitivity-swe-main}
\end{figure}
\FloatBarrier

\subsection{(2+1)-D Navier-Stokes equations}
\label{app:ns_2d}

The vorticity benchmark solves
\[
\partial_t\omega+\mathbf u\!\cdot\!\nabla\omega=\nu\nabla^2\omega,
\qquad
\mathbf u=(\partial_y\psi,-\partial_x\psi),
\qquad
\nabla^2\psi=-\omega
\]
on $\mathbb T^2$. It evaluates spatiotemporal rollouts under a controlled coupled translation--rotation transformation. Models observe $T_{\mathrm{in}}=10$ steps and predict $T=20$ steps. Each OOD trajectory is solved from an initial condition constructed with one parameter tuple; at $t=0$ this gives
\[
\omega_{0,\mathrm{OOD}}(\mathbf{x})
=
\omega_0\big(\mathbf{R}_{-\theta}(\mathbf{x}-\Delta\mathbf{x})\big),
\]
which applies a translation and a spatial rotation to the canonical initial condition.

\paragraph{Rotation--translation augmentation.}
For the Navier--Stokes task, we parameterize the planar transformation by an ambient $\mathfrak{se}(2)$ generator, formed by superposing two-dimensional translation and rigid rotation differential operators:
\begin{equation*}
    \xi_0 = -t_x \frac{\partial}{\partial x} - t_y \frac{\partial}{\partial y} - \theta \left( x \frac{\partial}{\partial y} - y \frac{\partial}{\partial x} \right)
\end{equation*}
Here, $\Delta \mathbf{x} = (t_x, t_y) \sim \mathcal{U}(-0.05, 0.05)^2$ is an exact periodic translation, while $\theta \sim \mathcal{U}(-\pi/4, \pi/4)$ controls the three-shear rotated-field map. The ambient $\mathfrak{se}(2)$ notation parameterizes this implementation; it does not make the arbitrary-angle map a representation of the full $SE(2)$ group on the square torus. Unlike kinematic background flows, these spatial perturbations do not accumulate a macroscopic advection increment over time.

\paragraph{Self-consistency TTA.} After canonicalization, the estimator should return little residual translation or rotation. We write the residual translation as $\mathbf{t}_{\mathrm{res}} = (t_{x,\mathrm{res}}, t_{y,\mathrm{res}})$ and the residual rotation as $\theta_{\mathrm{res}}$. To avoid phase-wrapping discontinuities and singular gradients near the periodic parameter identifications, the TTA loss uses a cosine distance rather than mean squared error:
\begin{equation*}
    \mathcal{J}_{\mathrm{TTA}}(\hat{\xi}_{\mathcal{S}}) = \frac{1}{B} \sum_{k=1}^B \left( \big(1 - \cos(\theta_{\mathrm{res}}^{(k)})\big) + \big(1 - \cos(2\pi t_{x,\mathrm{res}}^{(k)})\big) + \big(1 - \cos(2\pi t_{y,\mathrm{res}}^{(k)})\big) \right)
\end{equation*}
This formulation removes discontinuities induced by the periodic identification of angles and translations.

Figure~\ref{fig:ns_id_canonical} shows the frame correction on the Navier--Stokes rollout input.
\begin{figure}[htbp]
        \centering
        \includegraphics[width=0.98\textwidth]{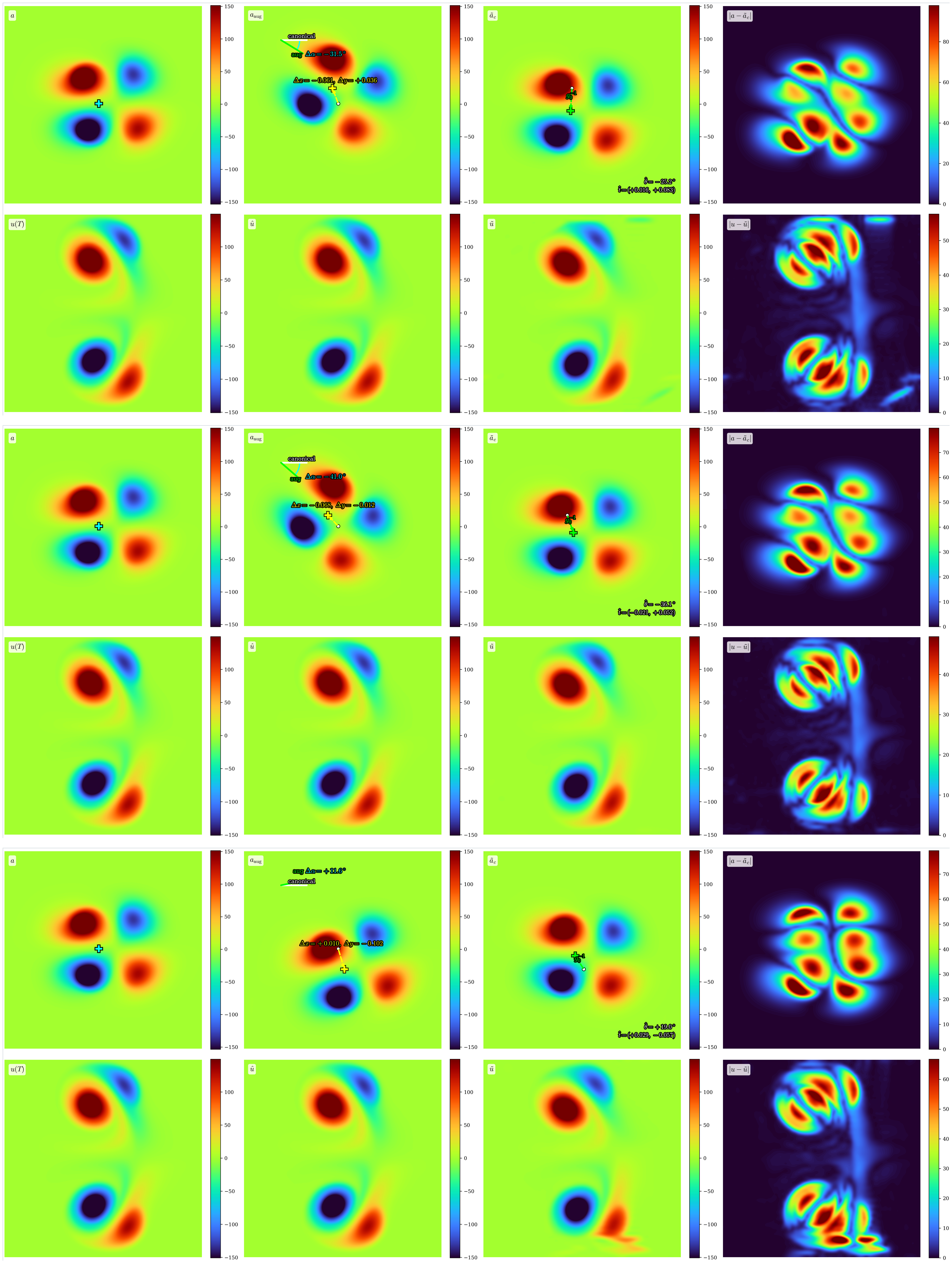}
        \caption{Canonicalization on the $(2+1)$-dimensional Navier-Stokes setting. The learned frame correction is applied to the imposed translation-rotation shift before rollout.}
        \label{fig:ns_id_canonical}
\end{figure}

Figure~\ref{fig:ns_3d_ood_visualization} compares the baseline rollouts under the controlled diagnostic evaluation.
\begin{figure}[htbp]
    \centering
    \begin{subfigure}[t]{0.98\textwidth}
        \centering
        \includegraphics[width=\linewidth]{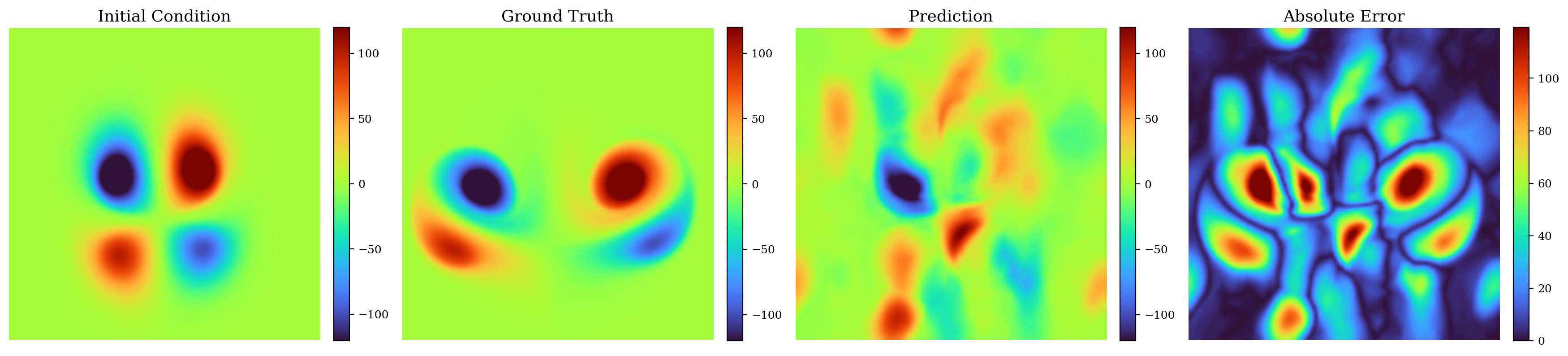}
        \subcaption{FNO.}
    \end{subfigure}

    \vspace{0.4em}

    \begin{subfigure}[t]{0.98\textwidth}
        \centering
        \includegraphics[width=\linewidth]{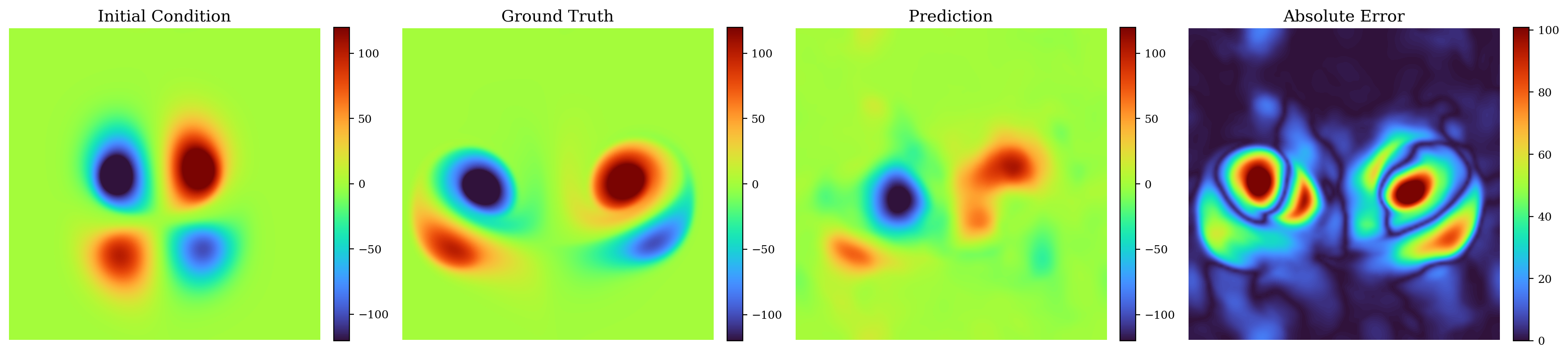}
        \subcaption{FNO+Aug.}
    \end{subfigure}

    \vspace{0.4em}

    \begin{subfigure}[t]{0.98\textwidth}
        \centering
        \includegraphics[width=\linewidth]{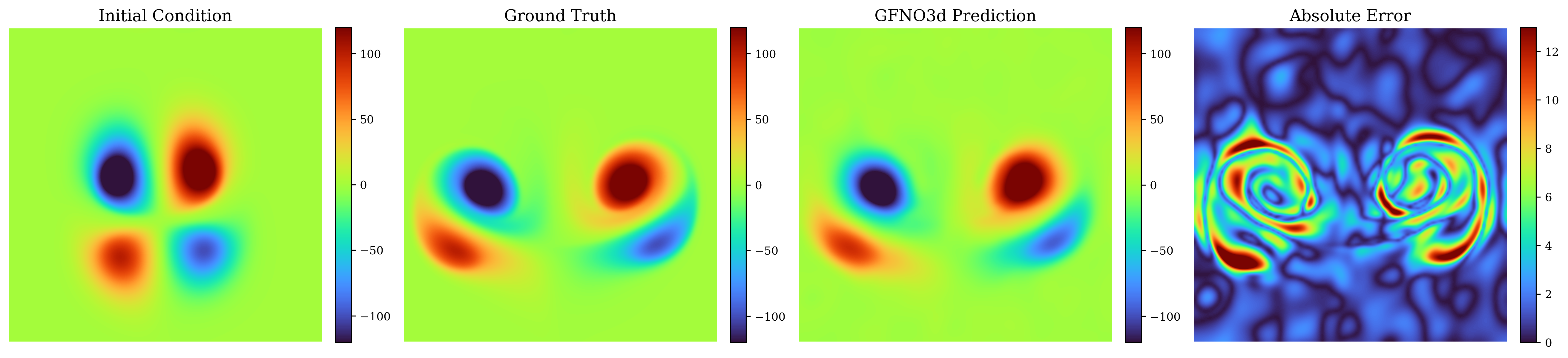}
        \subcaption{GFNO.}
    \end{subfigure}
    \caption{
    Baseline OOD predictions under the controlled rotation--translation diagnostic for the $(2+1)$-dimensional Navier--Stokes equations. The corresponding PACE-FNO prediction is shown in the main paper; these panels isolate the baseline failures in vortex placement and rollout geometry.
    }
    \label{fig:ns_3d_ood_visualization}
\end{figure}
\FloatBarrier
\subsection{Mechanism diagnostics}
\label{app:ablation_tta}

Table~\ref{tab:ablation_pacefno_all} collects the ablation and observation-diagnostic results used to separate the wrapper from auxiliary estimator capacity.
\begin{table}[htbp]
\centering
\caption{Unified ablation study of PACE-FNO. The \textit{Without canonicalization} variant disables learned pullback/pushforward during inference. For SWE, \textit{Without the full Lie generator} is numerically unstable (unclipped error reported). Dashes indicate entries with no independent training cost or evaluation.}
\label{tab:ablation_pacefno_all}
\setlength{\tabcolsep}{2.5pt}
\renewcommand{\arraystretch}{0.98}
\scriptsize
\resizebox{\textwidth}{!}{%
\begin{tabular}{llcccc}
\toprule
\multirow{2}{*}{Dataset} & \multirow{2}{*}{Model}
& \multicolumn{2}{c}{Relative Error}
& \multicolumn{2}{c}{Runtime (s)} \\
\cmidrule(lr){3-4} \cmidrule(lr){5-6}
& & ID relative error $\downarrow$ & OOD relative error $\downarrow$ & Training (s/epoch) & OOD evaluation (s) \\
\midrule

\multirow{4}{*}{1D Burgers}
& PACE-FNO + TTA
& $0.0621 \pm 0.0028$
& $0.0260 \pm 0.0012$
& $0.4864 \pm 0.0037$
& $5.9590 \pm 0.0963$ \\

& Without the full Lie generator
& $0.0471 \pm 0.0027$
& $0.0193 \pm 0.0007$
& $0.4902 \pm 0.0034$
& $6.0546 \pm 0.2355$ \\

& PACE-FNO (one-shot)
& $-$
& $0.0267 \pm 0.0014$
& $-$
& $0.5287 \pm 0.0171$ \\

& Without canonicalization
& $-$
& $1.0384 \pm 0.0003$
& $-$
& $0.4016 \pm 0.0058$ \\
\midrule

\multirow{5}{*}{2-D Burgers}
& PACE-FNO + TTA
&$0.0393 \pm 0.0012$
& $0.0640 \pm 0.0012$
& $1.5770 \pm 0.0066$
& $4.4345 \pm 0.0409$\\

& Without the structured spatial generator
& $0.0392 \pm 0.0008$
& $0.0666 \pm 0.0076$
& $1.0761 \pm 0.0268$
& $4.3969 \pm 0.0336$ \\

& Without the full Lie generator
& $0.0303 \pm 0.0018$
& $0.5478 \pm 0.0303$
& $1.1030 \pm 0.0083$
& $4.3551 \pm 0.0131$ \\

& PACE-FNO (one-shot)
& $-$
& $0.0755 \pm 0.0139$
& $-$
& $0.1399 \pm 0.0023$ \\

& Without canonicalization
& $-$
& $1.0085 \pm 0.0035$
& $-$
& $0.1417 \pm 0.0311$ \\
\midrule

\multirow{5}{*}{2-D SWE}
& PACE-FNO + TTA
& $0.0491 \pm 0.0002$
& $0.1071 \pm 0.0168$
& $2.6139 \pm 0.0046$
& $8.1388 \pm 0.0855$ \\

& Without the structured spatial generator
& $0.0501 \pm 0.0011$
& $0.1008 \pm 0.0142$
& $2.5280 \pm 0.0059$
& $7.8222 \pm 0.0129$ \\

& Without the full Lie generator
& $0.0586 \pm 0.0058$
& $2.11{\times}10^{4} \pm 2.28{\times}10^{4}$
& $2.0817 \pm 0.0222$
& $7.5391 \pm 0.0296$ \\

& PACE-FNO (one-shot)
& $-$
& $0.1069 \pm 0.0168$
& $-$
& $0.2952 \pm 0.0349$ \\

& Without canonicalization
& $-$
& $0.7232 \pm 0.0008$
& $-$
& $0.1426 \pm 0.0260$ \\
\midrule

\multirow{4}{*}{(2+1)D NS}
& PACE-FNO + TTA
& $0.0129 \pm 0.0006$
& $0.2593 \pm 0.0522$
& $16.5571 \pm 0.3631$
& $10.0376 \pm 0.0418$ \\

& Without canonicalization
& $0.0133 \pm 0.0006$
& $0.6280 \pm 0.0076$
& $10.6182 \pm 0.0660$
& $2.6558 \pm 0.0342$ \\

& Without the full Lie generator
& $0.0133 \pm 0.0006$
& $0.6280 \pm 0.0076$
& $11.0809 \pm 0.0649$
& $2.7494 \pm 0.0169$ \\

& PACE-FNO (one-shot)
& $-$
& $0.2766 \pm 0.0353$
& $-$
& $3.3728 \pm 0.1381$ \\
\bottomrule
\end{tabular}
}
\end{table}

The ablation above separates the one-shot canonicalizer from iterative TTA. Figure~\ref{fig:tta-sensitivity-swe-main} visualizes the tested SWE sensitivity range alongside the complete ablation table.

\FloatBarrier
\subsection{Vision transfer protocols}
\label{app:vision}

This appendix lists the data partitions, inference budgets, measurement conventions, and diagnostic-figure rules for the vision experiments in Section~\ref{sec:vision-experiments}. The reported model selection remains unchanged.

\subsection{PACE-FNO--LieLAC comparison scope}

LieLAC \citep{shumaylov2024lie} is a neural-operator method evaluated on PDE solvers. We include a faithful LieLAC+FNO comparison on the 1-D Burgers OOD split in Table~\ref{tab:burgers_1d_main}, and use Tables~\ref{tab:camelyon-main} and \ref{tab:digits-main} for the additional image comparisons. The remaining PDE backbone experiments use different field actions, data construction, and error objectives; they test wrapper transfer across predictors and are not LieLAC reproductions.

\subsection{Camelyon17-WILDS protocol and leakage audit}

We use the official Camelyon17-WILDS v1.0 split. Training and ID validation contain hospitals $\{0,3,4\}$, OOD validation contains hospital $1$, and OOD test contains hospital $2$. The implementation audits duplicate sample identifiers and slide overlap: training and ID validation share the official source-slide pool but have disjoint patches, while each OOD hospital is slide-disjoint from every source split and from the other OOD split. Hyperparameters and checkpoint selection use ID validation only; neither OOD validation nor OOD test is used for fitting or selection.

The Camelyon canonicalizer acts in log-optical-density color space: its six parameters are three channel scales and three offsets, so the hematoxylin and eosin (H\&E)-correlated synthetic perturbation is a stain/color transformation and does not spatially move the tissue. It uses a six-dimensional Huber inverse-parameter loss and an $\ell_1$ reconstruction term with coefficient $2$. For the TTA labels, ``moment'' initializes the coordinate from low-order source-image color statistics, whereas ``learned'' initializes it from the amortized encoder; both then optimize only the six-dimensional coordinate with frozen network weights. The convolutional neural network (CNN) and FNO columns in Table~\ref{tab:camelyon-main} are two classifier-backbone variants of the same canonicalizer, and are reported separately to avoid implying that one method has two simultaneous architectures. Full LieLAC uses 30 gradient steps and 8 hypotheses per patch; PACE-FNO TTA uses the PACE-FNO estimate as its only initialization and performs 5 steps. The reported runtime measures only canonicalizer execution time per patch and excludes data loading and backbone inference, so it measures alignment cost rather than end-to-end clinical throughput.

\begin{table}[htbp]
\centering
\caption{Camelyon17 TTA ablation on hospital-2 OOD test (mean $\pm$ standard deviation across three seeds). Additional time is relative to one-shot PACE-FNO.}
\label{tab:camelyon-ablation}
\small
\setlength{\tabcolsep}{4pt}
\begin{tabular}{lrrr}
\toprule
Variant & $\Delta$ CNN accuracy (pp) & $\Delta$ FNO accuracy (pp) & Additional time (ms/patch) \\
\midrule
PACE-FNO & 0.00 & 0.00 & 0.000 \\
PACE-FNO + TTA (moment) & +1.75 & +2.33 & +0.344 \\
PACE-FNO + TTA (learned) & +3.75 & +8.05 & +0.575 \\
\bottomrule
\end{tabular}
\end{table}

Figure~\ref{fig:camelyon-id-ood-appendix} separates source-hospital ID validation from the two held-out hospitals. The observed accuracy degradation motivates evaluating canonicalization, but hospital 1 is not interchangeable with hospital 2; the primary number throughout the paper remains hospital-2 OOD test. Figure~\ref{fig:camelyon-ablation-figure} visualizes the three-seed TTA ablation in Table~\ref{tab:camelyon-ablation}. The large CNN standard deviations preclude a statistical-significance claim.

\begin{figure}[htbp]
\centering
\includegraphics[width=0.94\linewidth]{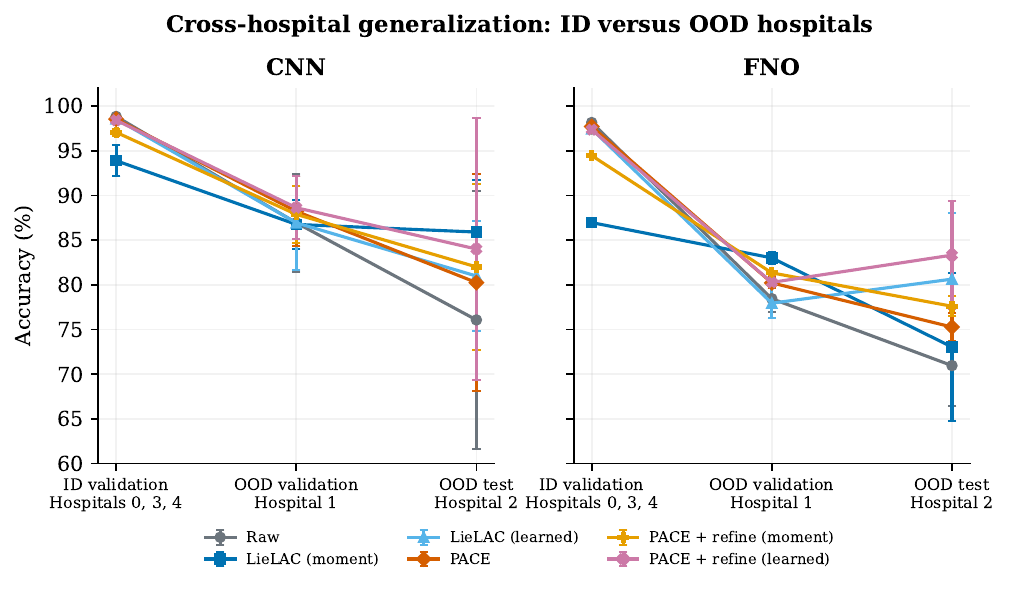}
\caption{Camelyon17 ID-to-OOD accuracy across source hospitals (0, 3, 4) and held-out hospitals (1, 2).}
\label{fig:camelyon-id-ood-appendix}
\end{figure}

\begin{figure}[htbp]
\centering
\includegraphics[width=0.78\linewidth]{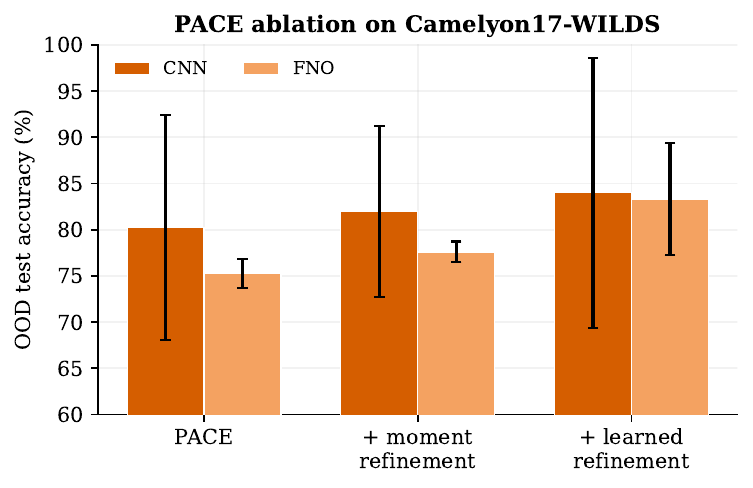}
\caption{PACE-FNO TTA ablation on Camelyon17 hospital-2 OOD test. Error bars are standard deviations over three seeds.}
\label{fig:camelyon-ablation-figure}
\end{figure}
\FloatBarrier
\subsection{Digit protocol, fairness, and timing}

The digit experiment retains clean MNIST as an ID control, uses affNIST for affine transfer, and treats homNIST as an evaluation-only projective-shift benchmark. affNIST train and validation sets are partitioned by original digit identity; homNIST is never used for training or model selection. The learned PACE-FNO encoder receives images only at inference. Affine matrices or homographies supervise source-side frame learning when available and are used separately for oracle evaluation, never supplied to the learned test-time encoder.

Training uses ID/source images and finite, label-preserving synthetic nuisance transformations. The synthetic transformations provide alignment supervision; they are not OOD examples and do not cover the full support of an external OOD distribution. Camelyon17 evaluates natural hospital variation, while affNIST and homNIST retain their official independently generated geometric shifts.

The comparison uses 10,000 samples per benchmark with subset seed 42. The digit classifier follows the convolutional MNIST setting \citep{lecun1998gradient}; affNIST uses an affine action and homNIST uses a projective action. Full LieLAC uses 100 optimization steps and 64 hypotheses. PACE-FNO with TTA starts from its one-shot estimate and takes the listed additional steps. Latency is end-to-end, excludes data loading, and uses synchronized accelerator timing boundaries. PACE-FNO amortizes frame estimation, whereas LieLAC optimizes each sample.

Table~\ref{tab:digit-protocol} records the split roles and inference budgets used in this comparison.
\begin{table}[htbp]
\centering
\caption{Digit evaluation and inference budgets. ``Oracle frame'' is used only for supervised source-side training targets and separate oracle diagnostics, never as test-time input to PACE-FNO.}
\label{tab:digit-protocol}
\small
\setlength{\tabcolsep}{4pt}
\begin{tabular}{p{0.20\linewidth}p{0.25\linewidth}p{0.39\linewidth}}
\toprule
Benchmark & Split role & Test-time protocol \\
\midrule
MNIST & ID retention control & Image only; identity geometry \\
affNIST & Affine OOD transfer & Image only; $\mathrm{Aff}(2)$ PACE-FNO; 20-step PACE-FNO with TTA or LieLAC \\
homNIST & Evaluation-only projective OOD & Image only; $\mathrm{PGL}(3)$ PACE-FNO; 20-step PACE-FNO with TTA or LieLAC \\
\bottomrule
\end{tabular}
\end{table}

\FloatBarrier
\subsection{Diagnostic visualizations}

The Camelyon17 montage in Figure~\ref{fig:camelyon-predictions} selects patches with large cross-method or cross-backbone score disagreement while preferring distinct slides. The digit montage in Figure~\ref{fig:digits-predictions} follows a fixed recovery-oriented rule. These montages are qualitative diagnostics, not random samples, and they do not enter the accuracy calculations. Tables~\ref{tab:camelyon-main} and \ref{tab:digits-main} contain the complete-test-set metrics.

\subsection{WeatherBench2 protocol}
\label{sec:weatherbench2_results}

Table~\ref{tab:weatherbench2_main} gives the frozen regional WeatherBench2 comparison at a direct $+72$-h lead. It includes the deterministic references, FNO, PACE-FNO, and SFNO. Learned-model entries are means and standard deviations over seeds $\{42,1234,3407\}$; model selection uses validation only. The data-role audit, aggregate metrics, and fixed extreme-case fields appear in the main paper.

Table~\ref{tab:weatherbench2_efficiency} records the selection epoch and runtime. SFNO has the lowest batch-24 memory use and the highest throughput in this benchmark. PACE-FNO is slightly slower and uses more memory than FNO. The memory column comes from the batch-24 CUDA microbenchmark, not from the 30-epoch training logs.

\section{Additional Empirical Results}
\label{app:empirical}
\label{exper}
This section gives the complete three-seed measurements, runtimes, parameter counts, and diagnostic evaluations behind the main-text comparisons. Benchmark definitions and transformation supports are given in Appendix~\ref{app:implementation}.

\subsection{Periodic PDE benchmarks}

Tables~\ref{tab:burgers_1d_main}--\ref{tab:ns_2d_main} give the complete three-seed results for the four core PDE benchmarks.
\begin{table}[htbp]
\centering
\caption{1-D Burgers under cross-resolution group shift (resolution $1024\!\to\!2048$). For the FNO and PACE-FNO rows, OOD evaluation time is the timed evaluation loop after the OOD tensors are loaded; it includes device transfers, forward passes, metric computation, and prediction caching, but not data construction or checkpoint loading. It is not a pure backbone-forward latency measurement. The LieLAC+FNO row is the dedicated same-split comparison; its ID error and training time were not recorded in the comparison summary, and its inference time follows that run's timing protocol.}
\label{tab:burgers_1d_main}
\small
\resizebox{\textwidth}{!}{%
\begin{tabular}{lcccc}
\toprule
Model & ID relative error $\downarrow$ & OOD relative error $\downarrow$ & Training (s/epoch) & OOD evaluation (s) \\
\midrule
FNO (no aug.)         & $0.0662 \pm 0.0028$ & $1.0352 \pm 0.0001$ & $0.4212 \pm 0.0095$ & $0.3943 \pm 0.0031$ \\
FNO+Aug               & $0.0293 \pm 0.0005$ & $0.3222 \pm 0.0077$ & $0.7698 \pm 0.0062$ & $1.2826 \pm 0.0340$ \\
PACE-FNO (one-shot)   & $0.0623 \pm 0.0015$ & $0.0267 \pm 0.0014$ & $0.4560 \pm 0.0030$ & $0.5287 \pm 0.0171$ \\
PACE-FNO + TTA        & $0.0623 \pm 0.0015$ & $\mathbf{0.0259 \pm 0.0007}$ & $0.4560 \pm 0.0030$ & $5.4611 \pm 0.1030$ \\
GFNO                  & $0.1406 \pm 0.0064$ & $1.0244 \pm 0.0027$ & $1.0365 \pm 0.0411$ & $1.0148 \pm 0.0091$ \\
LieLAC+FNO            & --- & $0.0342 \pm 0.0119$ & --- & $8.73$ \\
\bottomrule
\end{tabular}%
}
\end{table}

\begin{table}[htbp]
\centering
\caption{2-D Burgers under cross-resolution group shift (resolution $64^2\!\to\!128^2$).}
\label{tab:burgers_2d_main}
\scriptsize
\resizebox{\textwidth}{!}{%
\begin{tabular}{lcccc}
\toprule
Model & ID relative error $\downarrow$ & OOD relative error $\downarrow$ & Training (s/epoch) & OOD evaluation (s) \\
\midrule
FNO (no aug.)         & $0.0274 \pm 0.0010$ & $0.6137 \pm 0.0362$ & $0.6524 \pm 0.0157$ & $0.1244 \pm 0.0118$ \\
FNO+Aug               & $0.0261 \pm 0.0009$ & $0.2035 \pm 0.0423$ & $0.7559 \pm 0.0177$ & $0.1045 \pm 0.0038$ \\
PACE-FNO (one-shot)   & $0.0393 \pm 0.0012$ & $0.0755 \pm 0.0139$ & $1.5770 \pm 0.0066$ & $0.1399 \pm 0.0023$ \\
PACE-FNO + TTA        & $0.0393 \pm 0.0012$ & $\mathbf{0.0640 \pm 0.0012}$ & $1.5770 \pm 0.0066$ & $4.4345 \pm 0.0409$ \\
GFNO                  & $0.0967 \pm 0.0054$ & $1.0131 \pm 0.0119$ & $4.8613 \pm 0.3527$ & $0.7462 \pm 0.0080$ \\
\bottomrule
\end{tabular}%
}
\end{table}

\begin{table}[htbp]
\centering
\caption{2-D shallow water under cross-resolution group shift (resolution $64^2\!\to\!128^2$).}
\label{tab:swe_2d_main}
\small
\resizebox{\textwidth}{!}{%
\begin{tabular}{lcccc}
\toprule
Model & ID relative error $\downarrow$ & OOD relative error $\downarrow$ & Training (s/epoch) & OOD evaluation (s) \\
\midrule
FNO (no aug.)         & $0.0827 \pm 0.0009$ & $0.7260 \pm 0.0003$ & $0.7024 \pm 0.0474$ & $0.1501 \pm 0.0068$ \\
FNO+Aug               & $0.0769 \pm 0.0010$ & $0.5136 \pm 0.0010$ & $0.8502 \pm 0.0082$ & $0.1503 \pm 0.0012$ \\
PACE-FNO (one-shot)   & $0.0491 \pm 0.0002$ & $\mathbf{0.1069 \pm 0.0168}$ & $2.6139 \pm 0.0046$ & $0.2952 \pm 0.0349$ \\
PACE-FNO + TTA        & $0.0491 \pm 0.0002$ & $0.1071 \pm 0.0168$ & $2.6139 \pm 0.0046$ & $8.1388 \pm 0.0855$ \\
GFNO                  & $0.0307 \pm 0.0007$ & $0.7593 \pm 0.0071$ & $6.7509 \pm 0.0355$ & $0.7161 \pm 0.0401$ \\
\bottomrule
\end{tabular}%
}
\end{table}

\begin{table}[htbp]
\centering
\caption{Spatiotemporal Navier-Stokes under cross-resolution group shift (resolution $64^2\!\to\!128^2$).}
\label{tab:ns_2d_main}
\scriptsize
\resizebox{\textwidth}{!}{%
\begin{tabular}{lcccc}
\toprule
Model & ID relative error $\downarrow$ & OOD relative error $\downarrow$ & Training (s/epoch) & OOD evaluation (s) \\
\midrule
FNO (no aug.)         & $0.0377 \pm 0.0043$ & $0.6422 \pm 0.0302$ & $10.6448 \pm 0.0392$ & $1.8308 \pm 0.6626$ \\
FNO+Aug               & $0.0266 \pm 0.0059$ & $0.2840 \pm 0.2001$ & $10.7371 \pm 0.0281$ & $1.9761 \pm 0.5997$ \\
PACE-FNO (one-shot)   & $0.0129 \pm 0.0006$ & $0.2766 \pm 0.0353$ & $16.5571 \pm 0.3631$ & $3.3728 \pm 0.1381$ \\
PACE-FNO + TTA        & $0.0129 \pm 0.0006$ & $\mathbf{0.2593 \pm 0.0522}$ & $16.5571 \pm 0.3631$ & $10.0376 \pm 0.0418$ \\
GFNO                  & $0.0459 \pm 0.0016$ & $0.4091 \pm 0.0118$ & $55.9359 \pm 0.0496$ & $9.2956 \pm 0.0585$ \\
\bottomrule
\end{tabular}%
}
\end{table}

\FloatBarrier
\subsection{Backbone transfer}
\label{sec:backbone_portability}

We apply the same alignment wrapper to the Mamba Neural Operator, Continuous Vision Transformer (CViT), and Transolver predictors \citep{cheng2026mamba,wang2025cvit,wu2024transolver}. The ``PACE-'' names denote our wrapped variants, not new versions of the cited backbones. MambaGalerkin is compared with its base and augmented versions; CViT+Aug and Transolver+Aug are not included, so those comparisons are not causal comparisons with matched data augmentation.

Table~\ref{tab:backbone_transfer} reports the matched backbone-transfer comparisons, including runtime measurements. All listed predictors use one-shot input-to-output inference; Burgers is trained at $64^2$ and evaluated at $128^2$, and SWE uses the same resolution pair. The timed OOD column is milliseconds per sample. Runtime measurements were made on the same RTX 4080 setup described in Appendix~\ref{app:reproducibility}.
\begin{table}[htbp]
\centering
\caption{Same-PDE backbone comparisons (three seeds), with one-shot protocol and runtime. Each PACE variant retains the listed predictor and adds only the group-coordinate estimator and prescribed frame actions. Training time is seconds per epoch; OOD latency is milliseconds per sample. Dashes indicate that the corresponding runtime was not reported.}
\label{tab:backbone_transfer}
\scriptsize
\resizebox{\textwidth}{!}{%
\begin{tabular}{llrrrrrrr}
\toprule
PDE / protocol & Method & ID error $\downarrow$ & OOD error $\downarrow$ & Backbone params & Extra params & Total params & Train s/epoch & OOD ms/sample \\
\midrule
2-D Burgers, one-shot $64^2\!\to128^2$ & MambaGalerkin & $0.05006\pm0.00167$ & $0.73061\pm0.05598$ & $4{,}854{,}177$ & $0$ & $4{,}854{,}177$ & $20.757\pm0.062$ & $22.772$ \\
& MambaGalerkin+Aug & $\mathbf{0.03449\pm0.00091}$ & $0.32505\pm0.11714$ & $4{,}854{,}177$ & $0$ & $4{,}854{,}177$ & $20.862\pm0.112$ & $23.494$ \\
& PACE-MambaGalerkin & $0.04632\pm0.00203$ & $\mathbf{0.04246\pm0.00048}$ & $4{,}854{,}177$ & $297{,}156$ & $5{,}151{,}333$ & $23.130\pm0.248$ & $22.876$ \\
& CViT & $\mathbf{0.08819\pm0.00499}$ & $0.77986\pm0.02018$ & $1{,}834{,}177$ & $0$ & $1{,}834{,}177$ & --- & --- \\
& PACE-CViT & $0.18956\pm0.17901$ & $\mathbf{0.15340\pm0.05810}$ & $1{,}834{,}177$ & $297{,}156$ & $2{,}131{,}333$ & --- & --- \\
\midrule
2-D SWE, one-shot $64^2\!\to128^2$ & Transolver & $\mathbf{0.12479\pm0.00302}$ & $0.78036\pm0.06026$ & $11{,}229{,}251$ & $0$ & $11{,}229{,}251$ & $20.991\pm0.171$ & $20.064\pm0.575$ \\
& PACE-Transolver & $0.13924\pm0.01335$ & $\mathbf{0.33906\pm0.00456}$ & $11{,}229{,}251$ & $297{,}732$ & $11{,}526{,}983$ & $32.194\pm0.568$ & $18.061\pm0.007$ \\
\bottomrule
\end{tabular}%
}
\end{table}

\FloatBarrier
\subsection{Irregular-domain and rollout tests}
\label{sec:large_scale_scope}

Table~\ref{tab:large_scope} reports the irregular-domain and rollout scope tests.
\begin{table}[htbp]
\centering
\caption{Irregular-domain and rollout results (three seeds). Poisson uses relative $L_2$; KS OOD values are mean-centered relative $L_2$ over 300 closed-loop steps.}
\label{tab:large_scope}
\scriptsize
\resizebox{\textwidth}{!}{%
\begin{tabular}{llrrrr}
\toprule
Problem & Method & ID & Parameter/ID OOD & Geometric/shift OOD & Additional diagnostic \\
\midrule
3-D Poisson-ball & Transolver & $0.0212\pm0.0007$ & $0.0281\pm0.0008$ & $0.3000\pm0.0470$ ($90^\circ$) & $0.2344\pm0.0044$ (forcing) \\
& PACE-Transolver & $\mathbf{0.0180\pm0.0002}$ & $\mathbf{0.0262\pm0.0007}$ & $\mathbf{0.1815\pm0.0178}$ ($90^\circ$) & $0.2330\pm0.0044$ (forcing) \\
\midrule
1-D KS, 300 steps & FNO & $0.38783\pm0.05581$ & --- & $0.98605\pm0.06473$ & $1.37071\pm0.11634$ (final step) \\
& FNO+Aug & $0.45730\pm0.09516$ & --- & $0.46978\pm0.02451$ & $1.02729\pm0.07122$ (final step) \\
& PACE-FNO & $\mathbf{0.35860\pm0.07267}$ & --- & $\mathbf{0.35919\pm0.07196}$ & $\mathbf{0.82025\pm0.15761}$ (final step) \\
\bottomrule
\end{tabular}%
}
\end{table}

\FloatBarrier
\subsection{Ablation and observation diagnostics}

Table~\ref{tab:matched_control} reports the causal control, noise, and sparse-observation checks.
\begin{table}[htbp]
\centering
\caption{Causal control and observation diagnostics on 2-D Burgers (three seeds; standard-resolution OOD, distinct from the cross-resolution results in Table~\ref{tab:burgers_2d_main}). Subsequent rows reuse the top block's checkpoints.}
\label{tab:matched_control}
\scriptsize
\resizebox{\textwidth}{!}{%
\begin{tabular}{llrrrr}
\toprule
Setting & Metric & FNO & FNO+Aug & FNO+AuxEstimator & PACE-FNO \\
\midrule
Standard OOD & Rel. $L_2$ & $0.63950\pm0.03559$ & $0.18080\pm0.01653$ & $0.66021\pm0.03644$ & $\mathbf{0.05863\pm0.00287}$ \\
Standard OOD & Centered Rel. $L_2$ & $0.97141\pm0.05360$ & $0.75069\pm0.03450$ & $0.97814\pm0.01541$ & $\mathbf{0.17144\pm0.00974}$ \\
Noise $\alpha=0.10$ & Rel. $L_2$ & $0.63947\pm0.03567$ & $0.18090\pm0.01655$ & $0.66016\pm0.03650$ & $\mathbf{0.05867\pm0.00291}$ \\
Random sparse, $50\%$ & Rel. $L_2$ & $0.6369\pm0.0344$ & $0.3025\pm0.0129$ & $0.6622\pm0.0413$ & $\mathbf{0.2397\pm0.0031}$ \\
Structured missing, $50\%$ & Rel. $L_2$ & $0.6808\pm0.0259$ & $0.3560\pm0.0119$ & $0.7016\pm0.0316$ & $\mathbf{0.3244\pm0.0016}$ \\
\bottomrule
\end{tabular}%
}
\end{table}
\FloatBarrier

\subsection{Complete Accuracy and Latency Results}
\label{sec:vision-experiments}

\subsubsection{Camelyon17: cross-hospital stain shift}

\begin{table}[htbp]
\centering
\caption{Camelyon17 held-out hospital-2 OOD accuracy (mean $\pm$ standard deviation across three seeds). ``CNN accuracy'' and ``FNO accuracy'' denote separate classifier-backbone variants; canonicalization time excludes backbone inference.}
\label{tab:camelyon-main}
\small
\setlength{\tabcolsep}{3.2pt}
\begin{tabular}{lrrr}
\toprule
Method & CNN accuracy (\%) & FNO accuracy (\%) & Time (ms/patch) \\
\midrule
Raw & 76.06\pmstd14.46 & 70.95\pmstd4.56 & 0.000 \\
LieLAC (moment) & \best{85.90\pmstd5.86} & 73.04\pmstd8.23 & 4.823 \\
LieLAC (learned) & 80.98\pmstd6.17 & 80.63\pmstd7.42 & 21.854 \\
PACE-FNO & 80.24\pmstd12.14 & 75.27\pmstd1.57 & \best{0.022} \\
PACE-FNO + TTA (learned) & 84.00\pmstd14.62 & \best{83.31\pmstd6.07} & 0.597 \\
\bottomrule
\end{tabular}
\end{table}

Figure~\ref{fig:camelyon-latency} places the Camelyon17 accuracy and canonicalization latency measurements on the same axes.
\begin{figure}[htbp]
\centering
\includegraphics[width=0.94\linewidth]{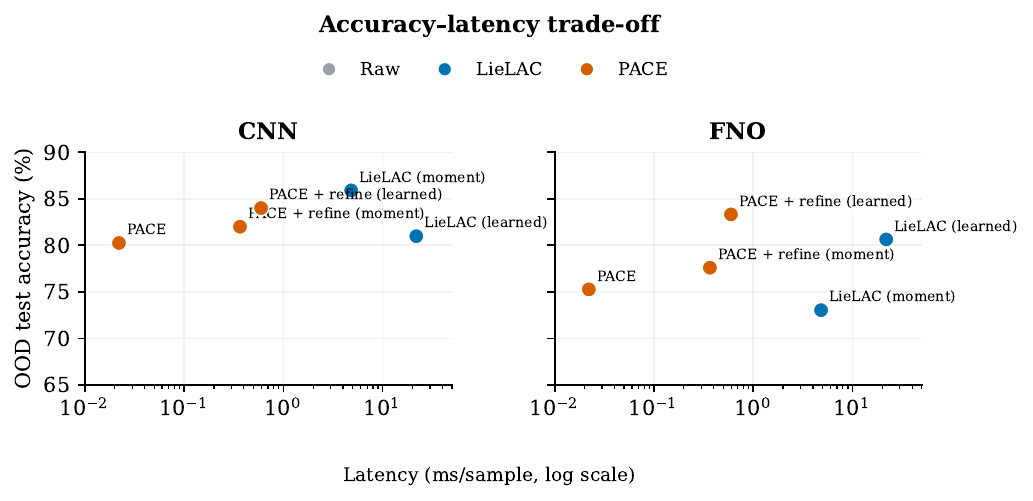}
\caption{Camelyon17 OOD accuracy--latency trade-off for canonicalization on the held-out hospital.}
\label{fig:camelyon-latency}
\end{figure}

\FloatBarrier
\subsubsection{Digits: affine and projective shifts}

\begin{table}[htbp]
\centering
\caption{PACE-FNO results on digit benchmarks (mean $\pm$ standard deviation across three seeds); latency is end-to-end.}
\label{tab:digits-main}
\small
\setlength{\tabcolsep}{3.2pt}
\begin{tabular}{llrr}
\toprule
Dataset & Method & Accuracy (\%) & Latency (ms/sample) \\
\midrule
MNIST & Raw & 99.050\pmstd0.062 & 0.066 \\
 & PACE-FNO & 99.200\pmstd0.030 & 0.200 \\
 & PACE-FNO + 10 steps & \best{99.270\pmstd0.017} & 3.023 \\
 & LieLAC & 99.130\pmstd0.040 & 286.234 \\
\midrule
affNIST & Raw & 37.743\pmstd0.768 & 0.066 \\
 & PACE-FNO & 89.953\pmstd0.254 & 0.202 \\
 & PACE-FNO + 20 steps & 92.340\pmstd0.300 & 5.361 \\
 & LieLAC & \best{95.503\pmstd0.370} & 288.592 \\
\midrule
homNIST & Raw & 89.400\pmstd1.109 & 0.066 \\
 & PACE-FNO & 97.180\pmstd0.141 & 0.200 \\
 & PACE-FNO + 20 steps & 97.673\pmstd0.153 & 5.435 \\
 & LieLAC & \best{98.433\pmstd0.142} & 287.818 \\
\bottomrule
\end{tabular}
\end{table}

Figure~\ref{fig:digits-latency} gives the corresponding accuracy--latency trade-off for the digit benchmarks.
\begin{figure}[htbp]
\centering
\includegraphics[width=0.94\linewidth]{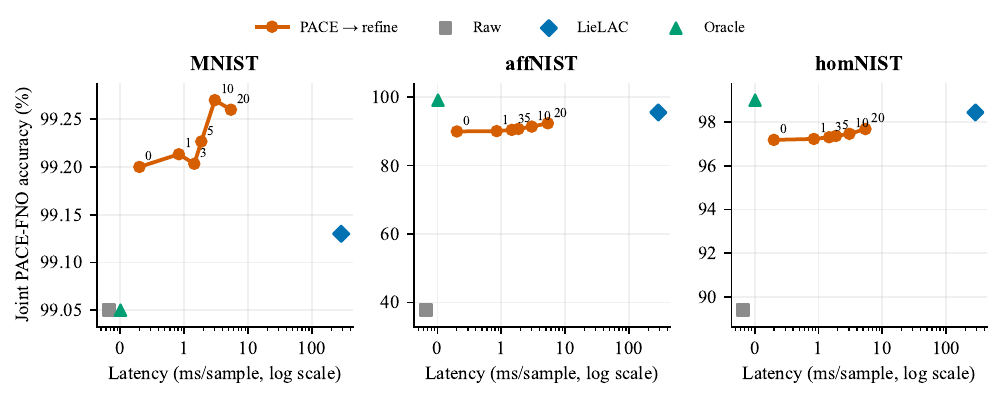}
\caption{PACE-FNO accuracy--latency trade-off on digit benchmarks.}
\label{fig:digits-latency}
\end{figure}
\FloatBarrier

\subsection{WeatherBench2 empirical metrics}

\begin{table}[!ht]
\centering
\caption{WeatherBench2 ERA5 $+72$-h results. Learned models report mean $\pm$ standard deviation over three seeds; deterministic baselines have zero seed variance. An em-dash denotes that the anomaly correlation coefficient (ACC) is undefined for the static train-climatology baseline.}
\label{tab:weatherbench2_main}
\small
\resizebox{\linewidth}{!}{%
\begin{tabular}{llccccc}
\toprule
Split & Metric & Train climatology & Persistence & FNO & PACE-FNO & SFNO \\
\midrule
ID test & T2m RMSE (K) & $2.471\,\pm\,0.000$ & $2.761\,\pm\,0.000$ & $2.078\,\pm\,0.005$ & $2.078\,\pm\,0.005$ & $2.886\,\pm\,0.022$ \\
ID test & Z500 RMSE (m) & $66.83\,\pm\,0.00$ & $71.04\,\pm\,0.00$ & $51.42\,\pm\,0.34$ & $51.41\,\pm\,0.35$ & $56.59\,\pm\,0.08$ \\
ID test & T2m ACC & --- & $0.381\,\pm\,0.000$ & $0.561\,\pm\,0.003$ & $0.561\,\pm\,0.003$ & $0.284\,\pm\,0.009$ \\
ID test & Z500 ACC & --- & $0.443\,\pm\,0.000$ & $0.641\,\pm\,0.005$ & $0.641\,\pm\,0.005$ & $0.541\,\pm\,0.002$ \\
ID test & Z500 phase error (deg) & $0.00\,\pm\,0.00$ & $56.14\,\pm\,0.00$ & $29.88\,\pm\,0.53$ & $29.88\,\pm\,0.51$ & $35.80\,\pm\,1.07$ \\
\midrule
Time OOD & T2m RMSE (K) & $2.504\,\pm\,0.000$ & $2.804\,\pm\,0.000$ & $2.055\,\pm\,0.008$ & $2.055\,\pm\,0.008$ & $2.880\,\pm\,0.017$ \\
Time OOD & Z500 RMSE (m) & $66.53\,\pm\,0.00$ & $70.53\,\pm\,0.00$ & $49.68\,\pm\,0.06$ & $49.67\,\pm\,0.05$ & $54.75\,\pm\,0.29$ \\
Time OOD & T2m ACC & --- & $0.376\,\pm\,0.000$ & $0.585\,\pm\,0.006$ & $0.586\,\pm\,0.005$ & $0.309\,\pm\,0.003$ \\
Time OOD & Z500 ACC & --- & $0.437\,\pm\,0.000$ & $0.667\,\pm\,0.002$ & $0.667\,\pm\,0.001$ & $0.576\,\pm\,0.003$ \\
Time OOD & Z500 phase error (deg) & $0.00\,\pm\,0.00$ & $57.47\,\pm\,0.00$ & $27.11\,\pm\,0.18$ & $27.11\,\pm\,0.17$ & $32.18\,\pm\,0.41$ \\
\midrule
Extreme OOD & T2m RMSE (K) & $2.468\,\pm\,0.000$ & $2.715\,\pm\,0.000$ & $1.994\,\pm\,0.001$ & $1.993\,\pm\,0.002$ & $2.833\,\pm\,0.016$ \\
Extreme OOD & Z500 RMSE (m) & $63.57\,\pm\,0.00$ & $68.25\,\pm\,0.00$ & $47.80\,\pm\,0.07$ & $47.77\,\pm\,0.10$ & $52.88\,\pm\,0.45$ \\
Extreme OOD & T2m ACC & --- & $0.403\,\pm\,0.000$ & $0.601\,\pm\,0.002$ & $0.602\,\pm\,0.001$ & $0.326\,\pm\,0.003$ \\
Extreme OOD & Z500 ACC & --- & $0.424\,\pm\,0.000$ & $0.663\,\pm\,0.002$ & $0.664\,\pm\,0.002$ & $0.566\,\pm\,0.009$ \\
Extreme OOD & Z500 phase error (deg) & $0.00\,\pm\,0.00$ & $55.66\,\pm\,0.00$ & $26.14\,\pm\,0.18$ & $26.06\,\pm\,0.20$ & $30.66\,\pm\,1.37$ \\
\bottomrule
\end{tabular}
}%
\end{table}

\subsection{WeatherBench2 efficiency}

\begin{table}[!ht]
\centering
\caption{Training selection and batch-24 runtime benchmark. Best epoch and wall time are mean $\pm$ standard deviation over three seeds. Throughput and memory are from the recorded batch-24 CUDA microbenchmark.}
\label{tab:weatherbench2_efficiency}
\small
\begin{tabular}{lccc ccc}
\toprule
Model & Best epoch & Wall time (min) & Step (s) & Samples/s & Peak GPU (GiB) \\
\midrule
FNO & 17.7 $\pm$ 1.2 & 58.9 $\pm$ 0.7 & 0.036 & 669.0 & 1.39 \\
PACE-FNO & 17.7 $\pm$ 1.2 & 62.9 $\pm$ 2.7 & 0.040 & 600.8 & 1.59 \\
SFNO & 21.7 $\pm$ 0.9 & 62.0 $\pm$ 0.1 & 0.026 & 919.1 & 0.62 \\
\bottomrule
\end{tabular}
\end{table}

\section{Theoretical Details and Proofs}
\label{app:theory}
\label{app:theory_details}

\paragraph{Proof roadmap.} We first relate exact compact-subgroup averaging to covering numbers. We then establish the input--terminal relation, decompose prediction error into canonical approximation and field-space alignment terms, and state local TTA stability. Each result is conditional on its displayed assumptions.

\subsection{Compact-subgroup averaging and covering numbers}
\label{app:generalization}

The following result concerns exact compact-subgroup averaging, whereas the experimental supports are bounded sets and need not be groups. A general bounded sampling support therefore has no normalized Haar measure.

Let $\rho$ denote the data distribution over $\mathcal{A}\times\mathcal{U}$, and let $\ell:\mathcal{U}\times\mathcal{U}\to[0,B]$ be a bounded loss function. For a learned operator $F\in\mathcal{F}$, the population risk and empirical risk are
\begin{equation*}
    \mathcal{R}(F) = \mathbb{E}_{(a,u)\sim\rho}\,\ell(F(a),u),
    \qquad
    \widehat{\mathcal{R}}_N(F) = \frac{1}{N}\sum_{i=1}^{N}\ell(F(a_i),u_i).
\end{equation*}

\begin{definition}[Generalization Error]
\label{def:gen_error}
The generalization error of an operator $F\in\mathcal{F}$ with respect to $N$ i.i.d.\ samples is
\begin{equation*}
    \mathrm{Gen}_N(F) := \mathcal{R}(F) - \widehat{\mathcal{R}}_N(F).
\end{equation*}
The uniform generalization error over the class is $\sup_{F\in\mathcal{F}}|\mathrm{Gen}_N(F)|$.
\end{definition}

\begin{definition}[Covering Number]
\label{def:covering_number}
Let $(\mathcal{F},d)$ be a metric space. For $\epsilon>0$, the $\epsilon$-covering number $\mathcal{N}(\mathcal{F},\epsilon,d)$ is the minimum cardinality of a set $\{F_1,\dots,F_m\}\subseteq\mathcal{F}$ such that for every $F\in\mathcal{F}$ there exists $j\in[m]$ with $d(F,F_j)\le\epsilon$.
\end{definition}

For operator learning, the natural metric is the uniform deviation over the input domain $\mathcal X$ supporting the data:
\begin{equation*}
    d_\infty(F_1,F_2) := \sup_{a\in\mathcal X}\|F_1(a)-F_2(a)\|_{\mathcal{U}}.
\end{equation*}

\begin{theorem}[Compact-Subgroup Orbit Averaging]
\label{thm:symmetry_complexity_reduction}
Let $\mathcal{A}$ and $\mathcal{U}$ be Banach spaces, let $H\le G$ be a compact subgroup with normalized Haar probability measure $\nu_H$, and let $\mathcal{X}_H\subseteq\mathcal{A}$ be $H$-invariant. Assume that $\mathcal{C}^{\mathcal A}:H\to\mathrm{Iso}(\mathcal A)$ is a continuous action by bijective isometries and that $\mathcal{C}^{\mathcal U}:H\to\mathcal{L}(\mathcal U)$ is a strongly continuous linear isometric representation. For a class $\mathcal{F}$ of continuous operators $\mathcal{X}_H\to\mathcal{U}$, define
\[
(\Pi_HF)(a)
:=
\int_H
\mathcal{C}^{\mathcal U}_{h^{-1}}
F(\mathcal{C}^{\mathcal A}_h a)
\,d\nu_H(h),
\qquad
\mathcal{F}_H:=\Pi_H(\mathcal{F}).
\]
Then $\Pi_HF$ is $H$-equivariant and $\Pi_H$ is non-expansive under $d_\infty$. Consequently, for any $r>0$,
\[
\mathcal{N}(\mathcal{F}_{H}, r, d_\infty)
\le
\mathcal{N}(\mathcal{F}, r, d_\infty).
\]
If $\eta_H(F):=\sup_{h\in H,\,a\in\mathcal{X}_H}\|(F\circ\mathcal{C}^{\mathcal A}_h)(a)-(\mathcal{C}^{\mathcal U}_h\circ F)(a)\|_{\mathcal{U}}\le\eta$, then $d_\infty(F,\Pi_HF)\le\eta$. For an $L_\ell$-Lipschitz loss evaluated on the same sample,
\[
\left|\mathrm{Gen}_N(F)\right|
\le
\left|\mathrm{Gen}_N(\Pi_HF)\right|+2L_\ell\eta.
\]
\end{theorem}

\begin{lemma}[Covering Number to Generalization Error]
\label{lemma:covering_to_gen}
Let $\mathcal{X}\subseteq\mathcal{A}$, let $\rho$ be supported on $\mathcal{X}\times\mathcal{U}$, and let $\mathcal{F}$ be a class of operators from $\mathcal{X}$ to $\mathcal{U}$. Define $d_\infty(F_1,F_2)=\sup_{a\in\mathcal X}\|F_1(a)-F_2(a)\|_{\mathcal U}$. Assume that the loss $\ell$ is $L_\ell$-Lipschitz in its first argument and bounded in $[0,B]$, and that $\mathcal{N}(\mathcal F,\epsilon,d_\infty)<\infty$. For any $\epsilon>0$ and $\delta\in(0,1)$, with probability at least $1-\delta$ over the draw of $N$ i.i.d.\ samples,
\begin{equation*}
    \sup_{F\in\mathcal{F}}\big|\mathcal{R}(F)-\widehat{\mathcal{R}}_N(F)\big|
    \le
    2L_\ell\epsilon + B\sqrt{\frac{\log\!\big(2\,\mathcal{N}(\mathcal{F},\epsilon,d_\infty)/\delta\big)}{2N}}.
\end{equation*}
\end{lemma}

\begin{proof}
Take an $\epsilon$-cover $\{F_1,\dots,F_m\}$ of $\mathcal{F}$ under $d_\infty$ with $m=\mathcal{N}(\mathcal{F},\epsilon,d_\infty)$. For any $F\in\mathcal{F}$, let $F_j$ be the cover element with $d_\infty(F,F_j)\le\epsilon$. The $L_\ell$-Lipschitz property of $\ell$ gives
\[
|\ell(F(a),u)-\ell(F_j(a),u)|\le L_\ell\|F(a)-F_j(a)\|_{\mathcal{U}}\le L_\ell\epsilon,
\]
so both population and empirical risks differ by at most $L_\ell\epsilon$ from those of $F_j$. Hence
\[
|\mathrm{Gen}_N(F)|\le 2L_\ell\epsilon + |\mathrm{Gen}_N(F_j)|.
\]
For a fixed $F_j$, Hoeffding's inequality for losses bounded in $[0,B]$ yields
\[
\Pr\!\big(|\mathrm{Gen}_N(F_j)|\ge t\big)\le 2\exp\!\left(-\frac{2Nt^2}{B^2}\right).
\]
A union bound over the $m$ cover elements and the choice $t=B\sqrt{\frac{\log(2m/\delta)}{2N}}$ complete the proof.
\end{proof}

\begin{remark}[Scope of the compact-subgroup bound]
\label{rem:improved_generalization}
Theorem~\ref{thm:symmetry_complexity_reduction} establishes that the exact compact-subgroup averaging operator $\Pi_H$ is non-expansive under $d_\infty$, which yields
\[
\mathcal{N}(\mathcal{F}_H,\epsilon,d_\infty)
\le
\mathcal{N}(\mathcal{F},\epsilon,d_\infty).
\]
Substituting this inequality into Lemma~\ref{lemma:covering_to_gen}, the resulting \emph{upper bound} for the explicitly projected class $\mathcal F_H=\Pi_H(\mathcal F)$ is no larger than the analogous upper bound for $\mathcal F$ at the same sample size. This comparison is between bounds; it does not order the realized generalization gaps of two trained models. If $F$ has compact-subgroup commutation residual $\eta_H(F)\le\eta$, the proof gives $d_\infty(F,\Pi_H F)\le\eta$, and the $L_\ell$-Lipschitz property of $\ell$ gives
\begin{equation}
    |\mathcal{R}(F)-\mathcal{R}(\Pi_H F)|\le L_\ell\eta,
    \qquad
    |\widehat{\mathcal{R}}_N(F)-\widehat{\mathcal{R}}_N(\Pi_H F)|\le L_\ell\eta.
    \label{eq:risk_deviation_eta}
\end{equation}
Combining~\eqref{eq:risk_deviation_eta} with Lemma~\ref{lemma:covering_to_gen} applied to $\Pi_H F\in\mathcal{F}_H$ and using the triangle inequality yields
\begin{equation*}
    |\mathrm{Gen}_N(F)|
    \le
    \underbrace{2L_\ell\epsilon + B\sqrt{\frac{\log\!\big(2\,\mathcal{N}(\mathcal{F}_H,\epsilon,d_\infty)/\delta\big)}{2N}}}_{\text{generalization bound for }\mathcal{F}_H}
    \;+\;
    2L_\ell\eta.
\end{equation*}
PACE-FNO is not constructed by applying $\Pi_H$ to an arbitrary operator, and this theorem does not prove that its learned hypothesis class equals $\mathcal F_H$. Nor does it identify $\eta_H(F)$ with coordinate-estimation error without an additional argument. For PACE-FNO, the directly applicable result is instead Theorem~\ref{thm:equiv_approx}, whose field-space alignment residuals remain explicit and must be assessed empirically.
\end{remark}

\subsection{Proofs}
\label{proof}
The standard FNO approximation result requires a compact input set and a continuous target operator. We restate the needed assumptions in the notation used here, then prove the two bounds used in the main text.

\paragraph{Induced terminal group element.}
For an exact transformation of the physical solution operator, an admissible initial element $g_0$ from the compact support $S$ defined below and a terminal element $g_T$ satisfy the intertwining relation
\[
\mathcal{G}^\dagger\circ\mathcal{C}_{g_0} = \mathcal{C}_{g_T}\circ\mathcal{G}^\dagger.
\]
For exact periodic translations, $g_T=g_0$. For the Galilean transformations used in the Burgers and SWE settings, the background velocity transports the field by $\mathbf{v}_{\mathrm{bg}}T$ over the time horizon $T$, inducing $g_T = \exp(\xi_0 + \mathbf{v}_{\mathrm{bg}}T\cdot\nabla)$. Arbitrary-angle rotations on the discretized square torus are excluded from this exact statement. The relation is applied in Theorem~\ref{thm:equiv_approx} to separate canonical approximation and field-space alignment residuals.

\begin{lemma}[Compactness and Orbit Saturation]
\label{lemma:compactness}
Let $D=\mathbb{T}^d$ be a periodic domain, and let the input and output spaces be $\mathcal{A}=L^2(D; \mathbb{R}^{d_a})$ and $\mathcal{U}=L^2(D; \mathbb{R}^{d_u})$. Assume that the canonical training measure $\mu$ is supported on a compact subset $\mathcal{A}_c \subset \mathcal{A}$. Let $S\subset G$ be a compact set of transformation elements used to define the bounded OOD family; $S$ is not assumed to be a subgroup. Define
\begin{equation*}
    \mathcal{A}_{\mathrm{OOD}} := \{ \mathcal{C}_g(a) : a \in \mathcal{A}_c, g \in S \} \subset \mathcal{A}.
\end{equation*}
If $(g,a)\mapsto\mathcal{C}_g(a)$ is continuous, then $\mathcal{A}_{\mathrm{OOD}}$ is compact as the continuous image of the compact product $S\times\mathcal A_c$. For Theorem~\ref{thm:equiv_approx}, also assume that $\mathcal{G}^\dagger$ is continuous on $\mathcal A_c$ and satisfies the exact input--terminal intertwining relation for every transformation in $S$.
\end{lemma}

Under these assumptions, the standard Fourier Neural Operator approximation result applies on the canonical compact set.

\begin{theorem}
    \label{theorem:fno_universal}
    If $\mathcal G^\dagger|_{\mathcal A_c}$ is continuous, then for any $\epsilon > 0$ there exists a Fourier Neural Operator $\mathcal{G}_\theta$ of sufficient depth or width such that
    \begin{equation*}
        \sup_{a \in \mathcal{A}_c} \|\mathcal{G}_\theta(a) - \mathcal{G}^\dagger(a)\|_{\mathcal{U}} \le \epsilon.
    \end{equation*}
\end{theorem}

\begin{remark}
    Theorem \ref{theorem:fno_universal} summarizes the standard compact-set approximation setting for neural operators. Its relevance to OOD prediction is limited: when an input field is transformed by an element $g \in S$, the transformed data $a_{\mathrm{OOD}} = \mathcal{C}_g(a)$ may fall outside the original canonical compact set $\mathcal{A}_c$. The theorem then no longer applies directly to that transformed input. PACE-FNO applies an estimated inverse transformation to map the observation toward $\mathcal{A}_c$. The resulting guarantee is conditional on the field-space alignment residual; it neither proves recovery of the generating coordinate nor requires that coordinate to be identifiable.
\end{remark}

For every $a\in\mathcal A_{\mathrm{OOD}}$, fix an admissible witness $(a_c(a),g_0(a),g_T(a))$ as in Theorem~\ref{thm:equiv_approx}. The witness need not be unique. The input field-space alignment residual is
\[
\epsilon_{\mathrm{geo},0}
:=
\sup_{a \in \mathcal{A}_{\mathrm{OOD}}}
\|\mathcal{C}_{\hat{g}_0(a)}^{-1}(a) - a_c(a)\|_{\mathcal{A}},
\]
and the terminal alignment error is
\[
\epsilon_{\mathrm{geo},T}
:=
\sup_{a \in \mathcal{A}_{\mathrm{OOD}}}
\|\mathcal{C}_{\hat{g}_T(a)}(\mathcal{G}^\dagger(a_c(a))) - \mathcal{C}_{g_T(a)}(\mathcal{G}^\dagger(a_c(a)))\|_{\mathcal{U}}.
\]

\begin{lemma}[Lipschitz Continuity of the Output Action]
\label{lemma:Lipschitz_1}
    Assume that each transformation $\mathcal{C}_g$ used in the exact input--terminal relation is Lipschitz on the output space $\mathcal{U}$. Exact continuum translations are isometries in $L^2$. More generally, assume that there exists a uniform constant $K_C \ge 1$ such that, for any $g$ arising as an estimated or admissible terminal element and any $u_1,u_2\in\mathcal{U}$,
    \[
    \|\mathcal{C}_g(u_1) - \mathcal{C}_g(u_2)\|_{\mathcal{U}} \le K_C \|u_1 - u_2\|_{\mathcal{U}}.
    \]
\end{lemma}

\begin{lemma}[Lipschitz Continuity of the Learned Operator]
\label{lemma:Lipschitz_2}
    Assume that the trained Fourier Neural Operator $\mathcal{G}_\theta$ is Lipschitz on a neighborhood of $\mathcal{A}_c$ that contains the pulled-back estimates considered in the proof. That is, for some $L_\theta \ge 1$ and all $a_1,a_2$ in this neighborhood,
    \[
    \|\mathcal{G}_\theta(a_1) - \mathcal{G}_\theta(a_2)\|_{\mathcal{U}} \le L_\theta \|a_1 - a_2\|_{\mathcal{A}}.
    \]
\end{lemma}

\begin{theorem}[Error Decomposition under Symmetry-Induced Shifts]
\label{thm:equiv_approx}
Assume the hypotheses of Lemma~\ref{lemma:compactness} (compact canonical support and continuous action), Lemmas~\ref{lemma:Lipschitz_1}--\ref{lemma:Lipschitz_2} (isometric/Lipschitz output action and Lipschitz learned operator), and that the \emph{chosen} FNO satisfies $\sup_{a_c\in\mathcal A_c}\|\mathcal G_\theta(a_c)-\mathcal G^\dagger(a_c)\|_{\mathcal U}\le\epsilon$. The universal-approximation result in Theorem~\ref{theorem:fno_universal} only motivates the existence of such an FNO; it does not assert that optimization finds it. For each $a\in\mathcal{A}_{\mathrm{OOD}}$, fix any admissible witness $(a_c(a),g_0(a),g_T(a))$ satisfying
\[
a=\mathcal{C}_{g_0(a)}(a_c(a)),
\qquad
\mathcal{G}^{\dagger}(a)
=
\mathcal{C}_{g_T(a)}(\mathcal{G}^{\dagger}(a_c(a))).
\]
No uniqueness of this witness is assumed. Then the prediction error of PACE-FNO over the bounded OOD family satisfies
\[
\sup_{a \in \mathcal{A}_{\mathrm{OOD}}}
\left\|
\mathcal{C}_{\hat{g}_T} \circ \mathcal{G}_{\theta} \circ \mathcal{C}^{-1}_{\hat{g}_0}(a)
-
\mathcal{G}^{\dagger}(a)
\right\|_{\mathcal{U}}
\le
K_C\epsilon + M_0 \epsilon_{\mathrm{geo},0} + \epsilon_{\mathrm{geo},T},
\]
where $\hat{g}_0=\exp(\hat{\xi}_0)$, $\hat{g}_T=\exp(\hat{\xi}_T)$, and $M_0=K_C L_\theta$ with $K_C$ the output-action Lipschitz constant and $L_\theta$ the FNO Lipschitz constant. Here $\epsilon$ is the canonical approximation error, while $\epsilon_{\mathrm{geo},0}$ and $\epsilon_{\mathrm{geo},T}$ are the field-space input and terminal alignment residuals defined above. For an isometric output action, $K_C=1$.
\end{theorem}

\begin{proof}[Proof of Theorem \ref{thm:equiv_approx}]
    Let $a \in \mathcal{A}_{\mathrm{OOD}}$ be arbitrary, and use the fixed admissible witness $(a_c(a),g_0(a),g_T(a))$. To simplify notation within this proof, write it as $(a_c,g_0,g_T)$. By definition,
    \[
    a = \mathcal{C}_{g_0}(a_c),
    \qquad
    \mathcal{G}^\dagger(a)=\mathcal{C}_{g_T}(\mathcal{G}^\dagger(a_c)).
    \]
    The triangle inequality separates the total prediction error into input-alignment, canonical-approximation, and terminal-alignment terms:
    \begin{align*}
    \| \mathcal{C}_{\hat{g}_T}(\mathcal{G}_\theta(\mathcal{C}_{\hat{g}_0}^{-1}(a))) - \mathcal{C}_{g_T}(\mathcal{G}^\dagger(a_c)) \|_{\mathcal{U}}
    &= \Big\| \mathcal{C}_{\hat{g}_T}(\mathcal{G}_\theta(\mathcal{C}_{\hat{g}_0}^{-1}(a))) - \mathcal{C}_{\hat{g}_T}(\mathcal{G}_\theta(a_c)) \notag\\
    &\quad + \mathcal{C}_{\hat{g}_T}(\mathcal{G}_\theta(a_c)) - \mathcal{C}_{\hat{g}_T}(\mathcal{G}^\dagger(a_c)) \notag\\
    &\quad + \mathcal{C}_{\hat{g}_T}(\mathcal{G}^\dagger(a_c)) - \mathcal{C}_{g_T}(\mathcal{G}^\dagger(a_c)) \Big\|_{\mathcal{U}}\\
    &\le \underbrace{\| \mathcal{C}_{\hat{g}_T}(\mathcal{G}_\theta(\mathcal{C}_{\hat{g}_0}^{-1}(a))) - \mathcal{C}_{\hat{g}_T}(\mathcal{G}_\theta(a_c)) \|_{\mathcal{U}}}_{\mathrm{(I)}\ \text{Initial alignment error}} \notag\\
    &\quad + \underbrace{\| \mathcal{C}_{\hat{g}_T}(\mathcal{G}_\theta(a_c)) - \mathcal{C}_{\hat{g}_T}(\mathcal{G}^\dagger(a_c)) \|_{\mathcal{U}}}_{\mathrm{(II)}\ \text{Canonical-domain approximation error}} \notag\\
    &\quad + \underbrace{\| \mathcal{C}_{\hat{g}_T}(\mathcal{G}^\dagger(a_c)) - \mathcal{C}_{g_T}(\mathcal{G}^\dagger(a_c)) \|_{\mathcal{U}}}_{\mathrm{(III)}\ \text{Terminal group-action mismatch}}.
\end{align*}
    For the first term, Lemmas \ref{lemma:Lipschitz_1} and \ref{lemma:Lipschitz_2} give
    \[
    \mathrm{(I)}
\le
K_C \|\mathcal{G}_\theta(\mathcal{C}_{\hat{g}_0}^{-1}(a)) - \mathcal{G}_\theta(a_c)\|_{\mathcal{U}}
\le
K_C \cdot L_\theta \|\mathcal{C}_{\hat{g}_0}^{-1}(a) - a_c\|_{\mathcal{A}}.
    \]
    By the definition of $\epsilon_{\mathrm{geo},0}$, this term is bounded by $M_0\epsilon_{\mathrm{geo},0}$, where $M_0:=K_C L_\theta$. For the canonical-domain approximation term, Lemma \ref{lemma:Lipschitz_1} and the theorem's assumed canonical approximation bound give
    \[
    \mathrm{(II)}
\le
K_C\|\mathcal{G}_\theta(a_c)-\mathcal{G}^\dagger(a_c)\|_{\mathcal{U}} \le K_C\epsilon.
    \]
    The third term is exactly bounded by the definition of the terminal alignment error:
    \[
    \mathrm{(III)} \le \epsilon_{\mathrm{geo},T}.
    \]
    Combining the three bounds and taking the supremum over $a \in \mathcal{A}_{\mathrm{OOD}}$ gives
    \begin{equation*}
        \sup_{a \in \mathcal{A}_{\mathrm{OOD}}}\| \mathcal{C}_{\hat{g}_T} \circ \mathcal{G}_\theta \circ \mathcal{C}_{\hat{g}_0}^{-1}(a) - \mathcal{G}^\dagger(a) \|_{\mathcal{U}} \le M_0 \epsilon_{\mathrm{geo},0} + K_C\epsilon + \epsilon_{\mathrm{geo},T}.
    \end{equation*}
    If the witness satisfies the intertwining relation only up to a uniform residual $\epsilon_{\mathrm{sym}}$, insert $\mathcal{C}_{g_T}\mathcal{G}^\dagger(a_c)$ between the prediction and $\mathcal G^\dagger(a)$ and apply the triangle inequality once more. This adds exactly $\epsilon_{\mathrm{sym}}$ to the right-hand side.
\end{proof}

\begin{proof}[Proof of Theorem \ref{thm:symmetry_complexity_reduction}]
Fix $F\in\mathcal F$ and $a\in\mathcal X_H$. Because $\mathcal X_H$ is $H$-invariant, the integrand is defined for every $h\in H$. Strong continuity, continuity of $F$, compactness of $H$, and completeness of $\mathcal U$ ensure that the displayed orbit average is Bochner integrable.

For any $q\in H$, the representation laws, the change of variables $k=hq$, and right invariance of normalized Haar measure give
\begin{align*}
(\Pi_HF)(\mathcal C_q^{\mathcal A}a)
&=
\int_H
\mathcal C_{h^{-1}}^{\mathcal U}
F(\mathcal C_{hq}^{\mathcal A}a)
\,d\nu_H(h)\\
&=
\int_H
\mathcal C_{qk^{-1}}^{\mathcal U}
F(\mathcal C_k^{\mathcal A}a)
\,d\nu_H(k)\\
&=
\mathcal C_q^{\mathcal U}
\int_H
\mathcal C_{k^{-1}}^{\mathcal U}
F(\mathcal C_k^{\mathcal A}a)
\,d\nu_H(k)\\
&=
\mathcal C_q^{\mathcal U}(\Pi_HF)(a).
\end{align*}
The third equality uses that $\mathcal C_q^{\mathcal U}$ is bounded and linear, so it commutes with the Bochner integral. Thus $\Pi_HF$ is $H$-equivariant. Averaging an already equivariant operator returns the same operator, so $\Pi_H^2=\Pi_H$.

For any $F_1,F_2\in\mathcal F$, linearity and isometry of the output representation imply
\begin{align*}
d_\infty(\Pi_HF_1,\Pi_HF_2)
&\le
\sup_{a\in\mathcal X_H}
\int_H
\left\|
\mathcal C_{h^{-1}}^{\mathcal U}
\big(
F_1(\mathcal C_h^{\mathcal A}a)
-
F_2(\mathcal C_h^{\mathcal A}a)
\big)
\right\|_{\mathcal U}
\,d\nu_H(h)\\
&=
\sup_{a\in\mathcal X_H}
\int_H
\left\|
F_1(\mathcal C_h^{\mathcal A}a)
-
F_2(\mathcal C_h^{\mathcal A}a)
\right\|_{\mathcal U}
\,d\nu_H(h)\\
&\le
d_\infty(F_1,F_2),
\end{align*}
where the last inequality again uses $H$-invariance of $\mathcal X_H$. Hence $\Pi_H$ is non-expansive. If $\{F_j\}_{j=1}^{m}$ is an $r$-cover of $\mathcal F$, then $\{\Pi_HF_j\}_{j=1}^{m}\subset\mathcal F_H$ is an $r$-cover of $\mathcal F_H$. Therefore,
\[
\mathcal N(\mathcal F_H,r,d_\infty)
\le
\mathcal N(\mathcal F,r,d_\infty).
\]
Lemma~\ref{lemma:covering_to_gen} applied on $\mathcal X_H$ now gives the stated comparison between the two covering-number upper bounds whenever the relevant covering numbers are finite.

Finally, suppose $\eta_H(F)\le\eta$. For every $h\in H$ and $a\in\mathcal X_H$, output isometry gives
\[
\left\|
\mathcal C_{h^{-1}}^{\mathcal U}
F(\mathcal C_h^{\mathcal A}a)
-
F(a)
\right\|_{\mathcal U}
=
\left\|
F(\mathcal C_h^{\mathcal A}a)
-
\mathcal C_h^{\mathcal U}F(a)
\right\|_{\mathcal U}
\le\eta.
\]
The norm inequality for the Bochner integral therefore yields
\[
d_\infty(F,\Pi_HF)\le\eta.
\]
The $L_\ell$-Lipschitz property of the loss implies
\[
|\mathcal R(F)-\mathcal R(\Pi_HF)|\le L_\ell\eta,
\qquad
|\widehat{\mathcal R}_N(F)-\widehat{\mathcal R}_N(\Pi_HF)|
\le L_\ell\eta.
\]
Applying the triangle inequality to
$\mathrm{Gen}_N(F)-\mathrm{Gen}_N(\Pi_HF)$
gives
\[
|\mathrm{Gen}_N(F)|
\le
|\mathrm{Gen}_N(\Pi_HF)|
+
2L_\ell\eta,
\]
which completes the proof.
\end{proof}

\begin{lemma}[Local Smoothness and Quadratic Growth Around the Minima Manifold]
\label{lemma:quadratic_growth_manifold}
Assume that the zero set $\Xi^*$ is nonempty, and let
\[
\mathcal{J}_{\mathrm{TTA}} : \mathfrak{g} \to \mathbb{R}_{\ge 0}
\]
be continuously differentiable on an open neighborhood $\mathcal{N}(\Xi^*) \subset \mathfrak{g}$ of
\[
\Xi^* := \{ \xi \in \mathfrak{g} \mid \mathcal{J}_{\mathrm{TTA}}(\xi)=0 \}.
\]
Assume further that there exists $L>0$ such that, for all $\hat{\xi}_1,\hat{\xi}_2 \in \mathcal{N}(\Xi^*)$,
\begin{equation*}
    \|\nabla \mathcal{J}_{\mathrm{TTA}}(\hat{\xi}_1)-\nabla \mathcal{J}_{\mathrm{TTA}}(\hat{\xi}_2)\|_{\mathfrak{g}}
    \le
    L \|\hat{\xi}_1-\hat{\xi}_2\|_{\mathfrak{g}},
\end{equation*}
and that there exists $\alpha>0$ such that, for all $\hat{\xi} \in \mathcal{N}(\Xi^*)$,
\begin{equation*}
    \mathcal{J}_{\mathrm{TTA}}(\hat{\xi})
    \ge
    \alpha\,\mathrm{dist}(\hat{\xi},\Xi^*)^2,
\end{equation*}
where
\[
\mathrm{dist}(\hat{\xi},\Xi^*)
:=
\inf_{\xi' \in \Xi^*}\|\hat{\xi}-\xi'\|_{\mathfrak{g}}.
\]
Since $\Xi^*$ is the zero set of $\mathcal{J}_{\mathrm{TTA}}$, we have
\[
\min \mathcal{J}_{\mathrm{TTA}} = 0.
\]
The quadratic-growth condition immediately yields the local error bound \citep{karimi2016linear}
\begin{equation*}
    \mathrm{dist}(\hat{\xi},\Xi^*)
    \le
    \alpha^{-1/2}\,\mathcal{J}_{\mathrm{TTA}}(\hat{\xi})^{1/2}.
\end{equation*}
\end{lemma}

\begin{theorem}[Monotone Energy Decay for Test-Time Adaptation Near the Minima Manifold]
\label{thm:tta_local_stability_qg}
Assume the hypotheses of Lemma~\ref{lemma:quadratic_growth_manifold}. Consider a global solution of the continuous-time gradient flow
\[
\frac{\mathrm{d}\hat{\xi}}{\mathrm{d}t}
=
-\nabla \mathcal{J}_{\mathrm{TTA}}(\hat{\xi}(t)).
\]
Let the initial point $\hat{\xi}^{(0)}$ satisfy
\[
\widetilde{\mathcal S}_0
:=
\Big\{
\xi \in \mathfrak{g}
\;\Big|\;
\mathcal{J}_{\mathrm{TTA}}(\xi)
\le
\mathcal{J}_{\mathrm{TTA}}(\hat{\xi}^{(0)})
\Big\},
\qquad
\mathcal S_0:=\overline{\operatorname{Conn}_{\hat\xi^{(0)}}(\widetilde{\mathcal S}_0)}
\Subset \mathcal{N}(\Xi^*).
\]
Here $\operatorname{Conn}_{\hat\xi^{(0)}}$ denotes the connected component containing the initial point. This component-wise assumption is needed for periodic coordinates: the full sublevel set in the vector-space chart $\mathfrak g\cong\mathbb R^k$ can contain infinitely many translated copies and need not be compact.
Then, for all $t\ge 0$, the energy is nonincreasing along the flow:
\begin{equation*}
    \frac{\mathrm{d}}{\mathrm{d}t}\mathcal{J}_{\mathrm{TTA}}(\hat{\xi}(t))
    =
    -\|\nabla \mathcal{J}_{\mathrm{TTA}}(\hat{\xi}(t))\|_{\mathfrak{g}}^2
    \le 0.
\end{equation*}
Consequently,
\begin{equation*}
    \mathcal{J}_{\mathrm{TTA}}(\hat{\xi}(t))
    +
    \int_0^t
    \|\nabla \mathcal{J}_{\mathrm{TTA}}(\hat{\xi}(s))\|_{\mathfrak{g}}^2\,\mathrm{d}s
    =
    \mathcal{J}_{\mathrm{TTA}}(\hat{\xi}^{(0)}),
\end{equation*}
and the trajectory remains in the connected component $\mathcal S_0\subset\mathcal{N}(\Xi^*)$. There exists
\[
\mathcal{J}_\infty \in [0,\mathcal{J}_{\mathrm{TTA}}(\hat{\xi}^{(0)})]
\]
such that
\begin{equation*}
    \lim_{t\to\infty}\mathcal{J}_{\mathrm{TTA}}(\hat{\xi}(t))
    =
    \mathcal{J}_\infty.
\end{equation*}
In addition,
\begin{equation*}
    \int_0^\infty
    \|\nabla \mathcal{J}_{\mathrm{TTA}}(\hat{\xi}(t))\|_{\mathfrak{g}}^2\,\mathrm{d}t
    =
    \mathcal{J}_{\mathrm{TTA}}(\hat{\xi}^{(0)})-\mathcal{J}_\infty
    < \infty,
\end{equation*}
and
\begin{equation*}
    \lim_{t\to\infty}
    \|\nabla \mathcal{J}_{\mathrm{TTA}}(\hat{\xi}(t))\|_{\mathfrak{g}}
    = 0.
\end{equation*}
The distance to the minima manifold is controlled by the energy:
\begin{equation*}
    \mathrm{dist}(\hat{\xi}(t),\Xi^*)
    \le
    \alpha^{-1/2}\,\mathcal{J}_{\mathrm{TTA}}(\hat{\xi}(t))^{1/2}
    \le
    \alpha^{-1/2}\,\mathcal{J}_{\mathrm{TTA}}(\hat{\xi}^{(0)})^{1/2}.
\end{equation*}
If, in addition, $\mathcal{J}_\infty = 0$, then
\begin{equation*}
    \lim_{t\to\infty}\mathrm{dist}(\hat{\xi}(t),\Xi^*) = 0.
\end{equation*}
The trajectory $\hat\xi(t)$ has accumulation points in the compact component $\mathcal S_0$, and every trajectory accumulation point is a critical point of $\mathcal{J}_{\mathrm{TTA}}$.
\end{theorem}

\begin{proof}[Proof of Theorem \ref{thm:tta_local_stability_qg}]
Since
\[
\frac{\mathrm{d}\hat{\xi}}{\mathrm{d}t}
=
-\nabla \mathcal{J}_{\mathrm{TTA}}(\hat{\xi}(t)),
\]
the chain rule gives
\[
\frac{\mathrm{d}}{\mathrm{d}t}\mathcal{J}_{\mathrm{TTA}}(\hat{\xi}(t))
=
\left\langle
\nabla \mathcal{J}_{\mathrm{TTA}}(\hat{\xi}(t)),
\frac{\mathrm{d}\hat{\xi}}{\mathrm{d}t}
\right\rangle_{\mathfrak{g}}
=
-
\|\nabla \mathcal{J}_{\mathrm{TTA}}(\hat{\xi}(t))\|_{\mathfrak{g}}^2
\le 0.
\]
This proves monotonicity, and integrating over $[0,t]$ yields
\[
\mathcal{J}_{\mathrm{TTA}}(\hat{\xi}(t))
+
\int_0^t
\|\nabla \mathcal{J}_{\mathrm{TTA}}(\hat{\xi}(s))\|_{\mathfrak{g}}^2\,\mathrm{d}s
=
\mathcal{J}_{\mathrm{TTA}}(\hat{\xi}^{(0)}).
\]
Hence the trajectory remains in the initial sublevel set. Since the trajectory is continuous and starts in the component $\mathcal S_0$, it cannot move to a disjoint component without leaving the sublevel set; therefore it remains in $\mathcal S_0$.

Because $\mathcal{J}_{\mathrm{TTA}}(\hat{\xi}(t))$ is nonincreasing and bounded below by $0$, it converges to a limit $\mathcal{J}_\infty \ge 0$. Passing to the limit $t\to\infty$ in the identity above gives
\[
\int_0^\infty
\|\nabla \mathcal{J}_{\mathrm{TTA}}(\hat{\xi}(s))\|_{\mathfrak{g}}^2\,\mathrm{d}s
=
\mathcal{J}_{\mathrm{TTA}}(\hat{\xi}^{(0)})-\mathcal{J}_\infty
< \infty.
\]

Because $\mathcal S_0$ is compact and $\nabla\mathcal J_{\mathrm{TTA}}$ is continuous, there is an $M<\infty$ such that $\|\nabla\mathcal J_{\mathrm{TTA}}(\hat\xi(t))\|_{\mathfrak g}\le M$ along the trajectory. Hence $\dot{\hat\xi}=-\nabla\mathcal J_{\mathrm{TTA}}(\hat\xi)$ implies that $\hat\xi(t)$ is $M$-Lipschitz in time. Composing this trajectory with the $L$-Lipschitz gradient makes the map
\[
t \mapsto \nabla \mathcal{J}_{\mathrm{TTA}}(\hat{\xi}(t))
\]
$LM$-Lipschitz, and therefore uniformly continuous. Together with the square-integrability established above, Barbalat's lemma implies
\[
\|\nabla \mathcal{J}_{\mathrm{TTA}}(\hat{\xi}(t))\|_{\mathfrak{g}} \to 0
\qquad
\text{as } t\to\infty.
\]

The quadratic-growth assumption in Lemma~\ref{lemma:quadratic_growth_manifold} gives
\[
\mathcal{J}_{\mathrm{TTA}}(\hat{\xi}(t))
\ge
\alpha\,\mathrm{dist}(\hat{\xi}(t),\Xi^*)^2,
\]
and therefore
\[
\mathrm{dist}(\hat{\xi}(t),\Xi^*)
\le
\alpha^{-1/2}\,\mathcal{J}_{\mathrm{TTA}}(\hat{\xi}(t))^{1/2}.
\]
Since the energy is nonincreasing,
\[
\mathrm{dist}(\hat{\xi}(t),\Xi^*)
\le
\alpha^{-1/2}\,\mathcal{J}_{\mathrm{TTA}}(\hat{\xi}^{(0)})^{1/2}.
\]
Finally, if $\mathcal{J}_{\mathrm{TTA}}(\hat{\xi}(t)) \to 0$, the same inequality implies
\[
\mathrm{dist}(\hat{\xi}(t),\Xi^*) \to 0.
\]

Compactness of $\mathcal S_0$ gives accumulation points of the trajectory. Since the gradient vanishes asymptotically and the gradient is continuous, every trajectory accumulation point is a critical point of $\mathcal{J}_{\mathrm{TTA}}$.
\end{proof}

\begin{lemma}[Lipschitz Continuity of Inverse Group Action]
\label{lemma:lipschitz_action}
Let $G$ be a finite-dimensional Lie group with Lie algebra $\mathfrak{g}$, fix $a_{\mathrm{OOD}}\in\mathcal A_{\mathrm{OOD}}$, and let $V\subset\mathfrak g$ be open. Assume that the parameter map $\xi\mapsto\mathcal C_{\exp(\xi)}^{-1}(a_{\mathrm{OOD}})$ is continuously differentiable on $V$. For any compact convex set $K\Subset V$, there exists a constant $L_{a,K}>0$ such that for all $\hat\xi_1,\hat\xi_2\in K$,
\begin{equation*}
    \big\|
    \mathcal{C}_{\exp(\hat{\xi}_1)}^{-1}(a_{\mathrm{OOD}})
    -
    \mathcal{C}_{\exp(\hat{\xi}_2)}^{-1}(a_{\mathrm{OOD}})
    \big\|_{\mathcal{A}}
    \le
    L_{a,K} \|\hat{\xi}_1 - \hat{\xi}_2\|_{\mathfrak{g}}.
\end{equation*}
This is a local statement on $K$, not a global Lipschitz claim on an arbitrary exponential chart. For continuum translations, the differentiability assumption is satisfied, for example, on sufficiently regular (e.g., $H^1$) fields; for the finite-grid implementation it is a property of the bounded resampling map.
\end{lemma}

\begin{corollary}[Prediction Error Bound Under Symmetry-Induced Equivalent Minima]
\label{cor:equivalent_attractors_qg}
Assume the hypotheses of Theorems~\ref{thm:tta_local_stability_qg} and~\ref{thm:equiv_approx}. Let $a_{\mathrm{OOD}} \in \mathcal{A}_{\mathrm{OOD}}$ be an arbitrary out-of-distribution input, and suppose the adaptation trajectory and the zero set $\Xi^*$ lie in a common compact convex coordinate set $K_\xi$ on which the input action is Lipschitz. Assume that the kinematic map $\Gamma_B(\cdot;T)$ is Lipschitz on $K_\xi$, so that
\[
\|\Gamma_B(\xi;T)-\Gamma_B(\xi';T)\|_{\mathfrak g}\le L_\Gamma\|\xi-\xi'\|_{\mathfrak g}.
\]
Assume, in addition, that every $\xi'\in\Xi^*$ is \emph{calibrated}: its pullback and terminal action are symmetry-equivalent to an admissible witness in Theorem~\ref{thm:equiv_approx}. Let $\hat{\xi}$ denote any finite-time estimator on this trajectory. Then the geometric perturbation argument in Theorem~\ref{thm:equiv_approx} yields
\begin{equation*}
    \| \hat{u}_{\mathrm{pred}} - \mathcal{G}^{\dagger}(a_{\mathrm{OOD}}) \|_{\mathcal{U}}
    \le
    K_C\epsilon + C_{\mathrm{geom}}\,\mathrm{dist}(\hat{\xi},\Xi^*),
\end{equation*}
for some $C_{\mathrm{geom}}>0$. Then
\begin{equation*}
    \| \hat{u}_{\mathrm{pred}} - \mathcal{G}^{\dagger}(a_{\mathrm{OOD}}) \|_{\mathcal{U}}
    \le
    K_C\epsilon + C_{\mathrm{geom}}\,\alpha^{-1/2}\,\mathcal{J}_{\mathrm{TTA}}(\hat{\xi})^{1/2}.
\end{equation*}
In particular, along any adaptation trajectory $\hat{\xi}(t)$ such that
\[
\mathcal{J}_{\mathrm{TTA}}(\hat{\xi}(t)) \to 0,
\]
we have
\begin{equation*}
    \limsup_{t\to\infty}
    \| \hat{u}_{\mathrm{pred}} - \mathcal{G}^{\dagger}(a_{\mathrm{OOD}}) \|_{\mathcal{U}}
    \le
    K_C\epsilon.
\end{equation*}
\end{corollary}

\begin{proof}[Proof of Corollary \ref{cor:equivalent_attractors_qg}]
Fix an arbitrary estimator $\hat{\xi}$ produced by the test-time adaptation dynamics. Since $\mathfrak{g}$ is finite-dimensional and $\mathcal{J}_{\mathrm{TTA}}$ is continuous, the nonempty zero set
\[
\Xi^*=\{\xi\in\mathfrak{g}\mid \mathcal{J}_{\mathrm{TTA}}(\xi)=0\}
\]
is closed. Hence the distance from $\hat{\xi}$ to $\Xi^*$ is attained by some $\xi' \in \Xi^*$, so that
\[
\|\hat{\xi}-\xi'\|_{\mathfrak{g}}
=
\mathrm{dist}(\hat{\xi},\Xi^*).
\]

By Lemma~\ref{lemma:lipschitz_action} on the compact coordinate set $K_\xi$, the inverse group action on the input space is Lipschitz with respect to the Lie algebra coordinates. Hence,
\begin{equation*}
    \left\|
    \mathcal{C}_{\exp(\hat{\xi})}^{-1}(a_{\mathrm{OOD}})
    -
    \mathcal{C}_{\exp(\xi')}^{-1}(a_{\mathrm{OOD}})
    \right\|_{\mathcal{A}}
    \le
    L_a\,\|\hat{\xi}-\xi'\|_{\mathfrak{g}}.
\end{equation*}
The calibration assumption makes $\xi'$ an admissible symmetry-equivalent witness. The Lipschitz assumption on $\Gamma_B$ together with the local Lipschitz property of the terminal output action gives, for some $L_u>0$,
\begin{equation*}
    \epsilon_{\mathrm{geo},T}
    \le
    L_u\,\|\hat{\xi}-\xi'\|_{\mathfrak{g}}.
\end{equation*}

Applying Theorem~\ref{thm:equiv_approx} with the symmetrically equivalent minimizer $\xi' \in \Xi^*$ yields
\begin{align*}
    \| \hat{u}_{\mathrm{pred}} - \mathcal{G}^{\dagger}(a_{\mathrm{OOD}}) \|_{\mathcal{U}}
    &\le
    K_C\epsilon
    + M_0
    \left\|
    \mathcal{C}_{\exp(\hat{\xi})}^{-1}(a_{\mathrm{OOD}})
    -
    \mathcal{C}_{\exp(\xi')}^{-1}(a_{\mathrm{OOD}})
    \right\|_{\mathcal{A}}
    + \epsilon_{\mathrm{geo},T}
    \notag\\
    &\le
    K_C\epsilon
    + (M_0L_a + L_u)\,\|\hat{\xi}-\xi'\|_{\mathfrak{g}}.
\end{align*}
Defining
\[
 C_{\mathrm{geom}} := M_0L_a + L_u,
\]
we obtain
\[
\| \hat{u}_{\mathrm{pred}} - \mathcal{G}^{\dagger}(a_{\mathrm{OOD}}) \|_{\mathcal{U}}
\le
 K_C\epsilon + C_{\mathrm{geom}}\,\mathrm{dist}(\hat{\xi},\Xi^*).
\]

Finally, applying the quadratic-growth error bound from Lemma~\ref{lemma:quadratic_growth_manifold}, we obtain
\[
\mathrm{dist}(\hat{\xi},\Xi^*)
\le
\alpha^{-1/2}\,\mathcal{J}_{\mathrm{TTA}}(\hat{\xi})^{1/2},
\]
and therefore
\[
\| \hat{u}_{\mathrm{pred}} - \mathcal{G}^{\dagger}(a_{\mathrm{OOD}}) \|_{\mathcal{U}}
\le
    K_C\epsilon + C_{\mathrm{geom}}\,\alpha^{-1/2}\,\mathcal{J}_{\mathrm{TTA}}(\hat{\xi})^{1/2}.
\]

If $\mathcal{J}_{\mathrm{TTA}}(\hat{\xi}(t)) \to 0$, then the right-hand side has limit $K_C\epsilon$, which yields the stated $\limsup$ bound.
\end{proof}

This is a local stability result for TTA, not a global convergence result for arbitrary OOD inputs. If initialization lies near the symmetry-equivalent minima manifold, reducing the residual energy controls the remaining geometric error. On the experimental tori, shifts by a full spatial period produce the same canonical field, so the bound does not depend on which equivalent minimizer attains the distance.

\end{document}